\documentclass{article}

\pdfoutput=1
\usepackage{ijcai21}

\usepackage{times}
\usepackage{soul}
\usepackage{url}
\usepackage[hidelinks]{hyperref}
\usepackage[utf8]{inputenc}
\usepackage[small]{caption}
\usepackage{graphicx}

\usepackage{amsfonts}
\usepackage{amssymb}
\usepackage{amsmath}
\usepackage{mathtools}
\usepackage{amsthm}
\usepackage{amstext}
\usepackage{array}
\usepackage{boxedminipage}
\usepackage{enumitem}
\usepackage{misc-ijcai}
\usepackage{xspace}
\usepackage{pgf}
\usepackage{tikz}
\usetikzlibrary{fit,calc, shapes, positioning}

\usepackage{thm-restate}

\usepackage{algorithm}
\usepackage[noend]{algpseudocode}

\newtheorem{theorem}{Theorem}
\newtheorem{lemma}{Lemma}

\newtheorem{proposition}{Proposition}

\title{Actively Learning Concepts and Conjunctive Queries under \ELdr$\!\!$-Ontologies}

\author{
Subject Area: Description Logics and Ontologies}

\author{
Maurice Funk$^1$\and
Jean Christoph Jung$^2$\And
Carsten Lutz$^{1}$
\affiliations
$^1$University of Bremen\\
$^2$University of Hildesheim\\
\emails
mfunk@uni-bremen.de,
jungj@uni-hildesheim.de,
clu@uni-bremen.de
}

\begin{document}

\maketitle

\begin{abstract}
  We consider the problem to learn a concept or a query in the
  presence of an ontology formulated in the description logic \ELdr,
 in Angluin's framework of active learning that allows the learning
  algorithm to interactively query an oracle (such as a domain
  expert). We show that the following can be learned in polynomial
  time: (1)~\EL-concepts, (2)~symmetry-free \ELI-concepts, and
  (3)~conjunctive queries (CQs) that are chordal, symmetry-free, and
  of bounded arity.  In all cases, the learner can pose to the oracle
  membership queries based on ABoxes and equivalence queries that ask
  whether a given concept/query from the considered class is
  equivalent to the target.  The restriction to bounded arity in (3)
  can be removed when we admit unrestricted CQs in equivalence
  queries. 
 We also show that \EL-concepts are not 
  polynomial query learnable in the presence of \ELI-ontologies.
\end{abstract}

\section{Introduction}

In logic based knowledge representation, a significant bottleneck is
the construction of logical formulas such as description logic
(DL) concepts, queries, and ontologies, as it is laborious and
expensive. This is particularly true if the construction
involves multiple parties because logic expertise and domain knowledge
are not in the same hands. Angluin's model of exact learning, a form
of active learning, is able to support the construction of logical formulas 
in terms of a game-like
collaboration between a learner and an oracle
\cite{DBLP:journals/ml/Angluin87,DBLP:journals/iandc/Angluin87}.
Applied in knowledge representation, the learner can be a logic expert
and the oracle a domain expert that is interactively queried by the
learner. Alternatively, the oracle can take other forms such as a set
of labeled data examples that in some way represents the formula to be
learned. The aim is to find an algorithm that, when executed by the
learner, constructs the desired formula in polynomial time even when
the oracle is not able to provide most informative answers. Landmark
results from active learning state that such algorithms exist for
learning propositional Horn formulas and finite automata
\cite{DBLP:journals/ml/AngluinFP92,DBLP:journals/iandc/Angluin87}.

The aim of this paper is to study active learning of \emph{DL
  concepts} and of \emph{conjunctive queries (CQs)} in the presence of
an ontology. 
Concepts are the main building block of ontologies \cite{DL-Textbook}
and learning them is important for ontology engineering. CQs are
very prominent in ontology-mediated querying where 
data stored
in an ABox is enriched with an ontology
\cite{DBLP:journals/tods/BienvenuCLW14}. 
We
concentrate on the \EL family of DLs which underlies the OWL EL
profile of the OWL 2 ontology language
\cite{DBLP:conf/rweb/Krotzsch12} and is frequently used in biomedical
ontologies such as {\sc Snomed CT}.  We consider ontologies 
formulated in the DLs $\ELdr$ and $\ELI$ where \ELdr extends \EL with
range restrictions and \ELI extends \ELdr with inverse roles. In both
DLs, concepts can be viewed as a tree-shaped
conjunctive query, and from now on we shall treat them as such. In
fact, it is not uncommon to use concepts as queries in
ontology-mediated querying, which provides an additional motivation for
learning them.

We now describe the learning protocol in detail. It is an
instance of Angluin's model, which we do not repeat here in full
generality. The aim is to learn a target CQ $q_T(\bar x)$ in the
presence of an ontology \Omc. The learner and the oracle both know and
agree on the ontology~\Omc, the arity of~$q_T$, and the concept
and role names that are available for constructing $q_T$; we assume
that all concept and role names in \Omc can be used also in $q_T$.
The learner can ask two types of queries to the oracle.  In a
\emph{membership query}, the learner provides an ABox~\Amc
and a candidate answer $\bar a$ and asks whether
$\Amc,\Omc \models q_T(\bar a)$; the oracle faithfully answers ``yes'' or ``no''.
In an \emph{equivalence query}, the learner provides a hypothesis CQ
$q_H$ 
  and asks whether $q_H$ is
  equivalent to $q_T$ under~\Omc; the oracle answers ``yes'' or
  provides a counterexample, that is, an ABox \Amc and tuple $\bar a$
  such that $\Amc,\Omc \models q_T(\bar a)$ and
  $\Amc,\Omc \not\models q_H(\bar a)$ (\emph{positive counterexample}) or
  vice versa (\emph{negative counterexample}).
 When we learn a restricted class of CQs such as \EL-concepts, 
we assume that only CQs from that class are admitted in equivalence
queries. We are then interested in whether there is a learning
algorithm that constructs $q_T(\bar x)$, up to equivalence under \Omc,
such that at any given time, the running time of the algorithm is
bounded by a polynomial in the sizes of $q_T$, of~$\Omc$, and of the
largest counterexample given by the oracle so far. This is called
\emph{polynomial time learnability}. A weaker requirement is
\emph{polynomial query learnability} where only the sum of the
sizes of the
queries posed to the oracle up to the current time point has to be
bounded
by such a polynomial.
%

Our main results are that the following can be learned in polynomial
time under \ELdr-ontologies: (1)~\EL-concepts, (2)~\ELI-concepts that
are symmetry-free, and (3)~CQs that are chordal, symmetry-free, and of
bounded arity. In Point~(2), symmetry-freeness means that there is no
subconcept of the form $\exists r . (C \sqcap \exists r^- . D)$ with
$r$ a role name, a condition that has recently been introduced in
\cite{aaaithis}, in a slightly less general form where $r$ can also be
an inverse role. In Point~(3), chordal means that every cycle of
length at least four that contains at least one quantified variable
has a chord and symmetry-free means that the CQ contains no atoms
$r(x_1,y),r(x_2,y)$ such that $x_1 \neq x_2$, $y$ is a quantified
variable, neither $r(x_1,y)$ nor $r(x_2,y)$ occur on a cycle, and
there is no atom $s(z,z)$ for any $z \in \{x_1,x_2,y\}$. An analysis
of well-known benchmarks for ontology-mediated querying suggests that
the resulting class CQ$^{\text{csf}}$ of CQs is sufficiently general
to include many relevant CQs that occur in practical applications.
Our proofs crucially rely on the use of a finite version of the
universal model that is specifically tailored to the class
CQ$^{\text{csf}}$.  We also show that the restriction to bounded arity
can be removed from Point~(3) when we admit unrestricted CQs as the
argument to equivalence queries. Proving this requires very
substantial changes to the learning algorithm.

In addition, we prove several negative results. First, we show that none of
the classes of CQs in Points~(1) to~(3) can be learned under
\EL-ontologies using only membership queries or only equivalence
queries (unless $\Pclass=\NPclass$ in the latter case). Note that
polynomial time learning
with only membership queries is important because it is related to
whether CQs can be characterized up to equivalence using only polynomially many data
examples \cite{DBLP:journals/corr/abs-2008-06824}.
We also show
the much more involved result that none of the classes of CQs in
Points~(1) to~(3) is polynomial query learnable under \ELI-ontologies.
Note that while polynomial time learnability cannot be expected
because subsumption in \ELI is \ExpTime-complete, there could well
have been a polynomial time learning algorithm with access to an
oracle (in the classical sense) for subsumption/query containment
under \ELI-ontologies that attains polynomial query learnability.
Our result rules out this possibility.

Proof details are in the appendix.

\paragraph{Related work.} Learning \EL-ontologies, rather than
concepts or queries, was studied in
\cite{KLOW-JMLR18,DBLP:conf/aaai/KonevOW16}. 
It turns out
that \EL-ontologies are not polynomial time learnable while certain
fragments thereof are. In contrast, we attain polynomial time
learnability also under unrestricted \EL-ontologies. See also the surveys~\cite{DBLP:series/ssw/LehmannV14,OzakiSurvey} and  
\cite{DBLP:conf/aaai/OzakiPM20} for a variation less related to  
the current work.  
It has been
shown in \cite{DBLP:journals/tods/CateDK13,DBLP:conf/pods/CateK0T18}
that unions of CQs (UCQs) are polynomial time learnable, and the
presented algorithm can be adapted to CQs.  Active learning of CQs
with only membership queries is considered in
\cite{DBLP:journals/corr/abs-2008-06824} where among other results
it is shown that
\ELI-concepts can be learned in polynomial time with only membership
queries when the ontology is empty.  PAC learnability of
concepts formulated in the DL CLASSIC, without ontologies, was studied
in~\cite{DBLP:conf/kr/CohenH94,DBLP:journals/ml/CohenH94,DBLP:journals/ml/FrazierP96}.
%

\section{Preliminaries}
\label{sect:prelims}

\paragraph{Concepts and Ontologies.}
Let \NC, \NR, and \NI be countably infinite sets of \emph{concept
  names}, \emph{role names}, and \emph{individual names}, respectively. A \emph{role} $R$ takes the form $r$ or $r^-$ where $r$ is a role name and
$r^-$ is called an \emph{inverse role}. If $R=s^-$ is an inverse role,
then $R^-$ denotes the role name $s$. 
An \emph{\ELI-concept} is
formed according to the syntax rule
\[ C,D ::= \top \mid A \mid C \sqcap D \mid \exists R . C\] where $A$
ranges over \NC and $R$ over roles. An \emph{\EL-concept} is an
\ELI-concept that does not use inverse roles. 

An \emph{\ELI-ontology}~\Omc is a finite set of \emph{concept
  inclusions (CIs)} $C \sqsubseteq D$ where $C$ and~$D$ range over
\ELI-concepts. An \emph{\ELdr-ontology} is an \ELI-ontology where
inverse roles occur only in the form of \emph{range restrictions}
$\exists r^-.\top \sqsubseteq C$ with $C$ an \EL-concept. Note that
\emph{domain restrictions} $\exists r . \top \sqsubseteq C$ can be
expressed already in \EL. An \EL-ontology is an \ELI-ontology that
does not use inverse roles. An \ELdr-ontology is in
\emph{normal form} if all CIs in it are of one of the forms
\[
  A_1 \sqcap A_2 \sqsubseteq A,\
A_1 \sqsubseteq \exists r . A_2,\
\exists r . A_1 \sqsubseteq A_2, \ 
\exists r^- . \top \sqsubseteq A
\]
where $A,A_1,A_2$ are concept names
or $\top$. 
An \emph{ABox} \Amc is a finite set of \emph{concept assertions}
$A(a)$ and \emph{role assertions} $r(a,b)$ where
$A \in \NC \cup \{ \top \}$, $r \in \NR$, and $a,b \in \NI$.  We use
$\mn{ind}(\Amc)$ to denote the set of individual names that are used
in \Amc and may write $r^-(a,b)$ in place of $r(b,a)$. An ABox is a
\emph{ditree} if the directed graph $(\mn{ind}(\Amc),\{(a,b) \mid
r(a,b) \in \Amc\})$ is a tree and there are no multi-edges, that is,
$r(a,b),s(a,b) \in \Amc$ implies $r=s$.

The semantics is defined as usual in terms of \emph{interpretations}
\Imc, which we define to be a (possibly infinite and) non-empty set of
concept and role assertions. We use $\Delta^\Imc$ to denote the set of
individual names in \Imc, define $A^\Imc = \{ a \mid A(a) \in \Imc \}$
for all $A \in \NC$, and $r^\Imc = \{ (a,b) \mid r(a,b) \in \Imc \}$
for all $r \in \NR$.  The extension $C^\Imc$ of \ELI-concepts $C$ is
then defined as usual \cite{DL-Textbook}. This definition of
interpretation is slightly different from the usual one, but
equivalent; 
its virtue is uniformity as every ABox is a (finite) interpretation.
An interpretation~\Imc \emph{satisfies} a CI $C \sqsubseteq D$ if
$C^\Imc \subseteq D^\Imc$, and a (concept or role) assertion $\alpha$
if $\alpha \in \Imc$ or $\alpha$ has the form~$\top(a)$.  We say that
\Imc is a \emph{model} of an ontology/ABox if it satisfies all concept
inclusions/assertions in it and write
$\Omc \models C \sqsubseteq D$ if every model of the ontology \Omc
satisfies the CI $C \sqsubseteq D$.

A \emph{signature} is a set of concept and role names, uniformly
referred to as \emph{symbols}.  
For any syntactic object $O$ such as an
ontology or an ABox, we use  $\mn{sig}(O)$ to denote the symbols used
 in $O$ and
$||O||$ to denote the \emph{size} of $O$, that is, the
length of a word representation of $O$ in a suitable alphabet.

\paragraph{CQs and Homomorphisms.}
A {\em conjunctive query (CQ)} takes the form
$q(\bar x) \leftarrow \varphi(\bar x, \bar y)$
where $\varphi$ is a conjunction of \emph{concept atoms} $A(x)$ and
\emph{role atoms} $r(x,y)$ with $A \in \NC$ and $r \in \NR$. We may
write $r^-(x,y)$ in place of $r(y,x)$. Note that the tuple $\bar x$
used in the \emph{head} $q(\bar x)$ of the CQ may contain repeated
occurrences of variables. When we do not want to make the \emph{body}
$\varphi(\bar x, \bar y)$ explicit, we may denote
$q(\bar x) \leftarrow \varphi(\bar x, \bar y)$
simply with $q(\bar x)$.  We refer to the variables in $\bar x$ as the
{\em answer variables} of $q$.
 and to the variables in $\bar y$ as the
 {\em quantified variables}.
When we are not interested in order and multiplicity, we treat
$\bar x$ and $\bar y$ as sets of variables.
We use $\mn{var}(q)$ to denote the set of
all variables in $\bar x$ and $\bar y$.
The \emph{arity} of $q$
is the length of tuple $\bar x$ and $q$ is \emph{Boolean} if it has
arity zero. 
Every
CQ $q(\bar x) \leftarrow \varphi(\bar x, \bar y)$
gives rise to an ABox (and thus interpretation) $\Amc_q$ obtained from
$\varphi(\bar x, \bar y)$ by viewing variables as individual names and
atoms as assertions. A CQ is a \emph{ditree} if $\Amc_q$ is.

A \emph{homomorphism} $h$ from interpretation $\Imc_1$ to
interpretation $\Imc_2$ is a mapping from $\Delta^{\Imc_1}$ to
$\Delta^{\Imc_2}$ such that $d \in A^{\Imc_1}$ implies
$h(d) \in A^{\Imc_2}$ and $(d, e) \in r^{\Imc_1}$ implies
$(h(d), h(e)) \in r^{\Imc_2}$. For $\bar d_i$ a tuple over
$\Delta^{\Imc_i}$,
$i \in \{1,2\}$, we write $\Imc_1, \bar d_1 \to \Imc_2, \bar d_2$
if there is a homomorphism $h$ from
$\Imc_1$ to $\Imc_2$ with $h(\bar d_1) = \bar d_2$.
With a homomorphism from a CQ $q$ to an interpretation \Imc,
we mean a homomorphism from $\Amc_q$ to \Imc.

Let $q(\bar x) \leftarrow \varphi(\bar x, \bar y)$
be a CQ and $\Imc$ an interpretation. A tuple
$\bar d \in (\Delta^\Imc)^{|\bar x|}$ is an {\em answer to $q$ on}
$\Imc$, written $\Imc \models q(\bar d)$, if there is a homomorphism
$h$ from $q$ to \Imc with $h(\bar x) = \bar d$.  
Now let \Omc be an \ELI-ontology and \Amc an ABox. A tuple
$\bar a \in \mn{ind}(\Amc)^{|\bar x|}$ is an {\em answer to $q$ on
  $\Amc$ under} \Omc, written $\Amc,\Omc \models q(\bar a)$ if
$\bar a$ is an answer to $q$ on every model of \Omc and~\Amc.

For $q_1$ and $q_2$ CQs of the same arity $n$ and $\Omc$ an 
\ELI-ontology, we say that $q_1$ is \emph{contained} in $q_2$ under 
\Omc, written $q_1 \subseteq_\Omc q_2$, if for all ABoxes $\Amc$ and 
$\bar a \in \mn{ind}(\Amc)^n$, 
$\Amc, \Omc \models q_1(\bar a)$ implies 
$\Amc, \Omc \models q_2(\bar a)$.  We call $q_1$ and $q_2$
\emph{equivalent} under \Omc, written $q_1 \equiv_\Omc q_2$, if 
$q_1 \subseteq_\Omc q_2$ and $q_2 \subseteq_\Omc q_1$. 

Every \ELI-concept can be viewed as a unary tree-shaped CQ in an
obvious way. For example, the \EL-concept
$A \sqcap \exists s . \top \sqcap \exists r . B$  yields the CQ
$q(x) \leftarrow A(x) \wedge s(x,y) \wedge r
(x,z) \wedge B(z)$. We 
use ELQ to denote the class of all \EL-concepts
viewed as a CQ, and likewise for ELIQ and \ELI-concepts.

\paragraph{Important Classes of CQs.} We next define a class of CQs
that we show later to admit polynomial time learnability under
\ELdr-ontologies, one of the main results of this paper. Let \Amc be
an ABox. A \emph{path} in \Amc from $a$ to $b$ is a sequence
$p=R_0(a_0,a_1), \dots, R_{n-1}(a_{n-1},a_n) \in \Amc$, $n \geq 0$,
such that $a_0=a$ and $a_n=b$.  We say that $p$ is a
\emph{cycle of length $n$} if $a_0=a_n$, all assertions in $p$
are distinct, and all of $a_0,\dots,a_{n-1}$ are distinct. A
\emph{chord} of cycle $p$ is an assertion $R(a_i,a_j)$ 
with $0 \leq i,j < n-1$ and $i\notin \{ j,j -1 \!\mod n, j+1 \!\mod n\}$.  A
cycle in a CQ $q$ is a cycle in $\Amc_q$.
With $\text{CQ}^{\text{csf}}$, we denote the class of CQs
$q(\bar x) \leftarrow \varphi(\bar x, \bar y)$ that
are
\begin{enumerate}

\item \emph{chordal}, that is, every cycle 
  $R_0(x_0,x_1),\dots,$ $R_{n-2}(x_{n-2},x_{n-1})$ in $q$ of length at 
  least four that contains at least one quantified variable has a 
  chord;
  
\item \emph{symmetry-free}, that is, if $\vp$ contains atoms
  $r(y_1,x),r(y_2,x)$ with $y_1\neq y_2$, then $x$ is an answer variable or one of
  the atoms occurs on a cycle or $\vp$ contains an atom $s(z,z)$
  for some $z \in \{ x,y_1,y_2 \}$.


\end{enumerate}
In Point~2, $r$ is a role name and thus there are no restrictions on
`inverse symmetries': $\varphi$ may contain atoms $r(x,y_1),r(x,y_2)$
with $x$ a quantified variable and none of the atoms occurring on a
cycle and no reflexive loops present.
Note that CQ$^{\text{csf}}$ contains all CQs without quantified
variables (also called \emph{full} CQs), all ELQs, and all ELIQs
obtained from \ELI-concepts that are \emph{symmetry-free}, that is,
that do not contain a subconcept of the form
$\exists r . (C \sqcap \exists r^- . D)$ with $r$ a role name.  We
denote the latter class with ELIQ$^{\text{sf}}$.  CQ$^{\text{csf}}$
also includes all CQs obtained from such ELIQs by choosing a set of
variables and making them answer variables.  Note that CQs from
$\text{CQ}^{\text{csf}}$ need not be connected, in fact
$\text{CQ}^{\text{csf}}$ is closed under disjoint union. Every CQ
whose graph is a clique or a $k$-tree (a maximal graph of treewidth
$k$) with $k>1$ is in $\text{CQ}^{\text{csf}}$. Some concrete examples
for CQs in $\text{CQ}^{\text{csf}}$ are given below, filled circles
indicating answer variables:

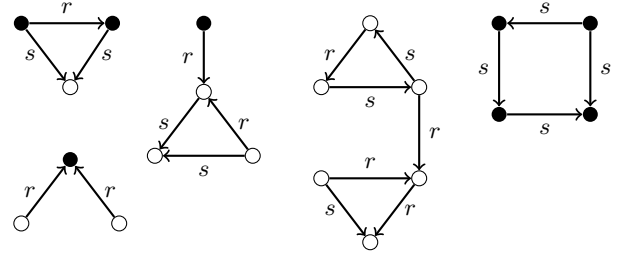
\begin{tikzpicture}[
    every node/.append style = {font=\footnotesize}
    ]

    \tikzstyle{answer} = [circle, fill, inner sep = 2pt]
    \tikzstyle{exists} = [circle, draw, inner sep = 2pt]
    \tikzstyle{role} = [->, thick]

    \node (a_1) [answer] {};
    \node (a_2) [answer, right = of a_1] {};
    \node (a_3) [exists, below right = 0.7cm and 0.5cm of a_1 ] {};
    \draw[role] (a_1) to node [above] {$r$} (a_2);
    \draw[role] (a_1) to node [left] {$s$} (a_3);
    \draw[role] (a_2) to node [right] {$s$} (a_3);

    \node (b_1) [answer, right = of a_2] {};
    \node (b_2) [exists, below = 0.7cm of b_1] {};
    \node (b_3) [exists, below left = 0.7cm and 0.5cm of b_2] {};
    \node (b_4) [exists, below right = 0.7cm and 0.5cm of b_2] {};
    \draw[role] (b_1) to node [left] {$r$} (b_2);
    \draw[role] (b_2) to node [left] {$s$} (b_3);
    \draw[role] (b_4) to node [right] {$r$} (b_2);
    \draw[role] (b_4) to node [below] {$s$} (b_3);

    \node (c_1) [answer, below = 0.75cm of a_3] {};
    \node (c_2) [exists, below left = 0.7cm and 0.5cm of c_1] {};
    \node (c_3) [exists, below right = 0.7cm and 0.5cm of c_1] {};
    \draw[role] (c_2) to node [left] {$r$} (c_1);
    \draw[role] (c_3) to node [right] {$r$} (c_1);

    \node (d_1) [exists, right = 2cm of b_1] {};
    \node (d_2) [exists, below left = 0.7cm and 0.5cm of d_1] {};
    \node (d_3) [exists, below right = 0.7cm and 0.5cm of d_1] {};
    \node (d_4) [exists, below = of d_2] {};
    \node (d_5) [exists, below = of d_3] {};
    \node (d_6) [exists, below right = 0.7cm and 0.5cm of d_4] {};
    \draw[role] (d_1) to node [left] {$r$} (d_2);
    \draw[role] (d_2) to node [below] {$s$} (d_3);
    \draw[role] (d_3) to node [right] {$s$} (d_1);
    \draw[role] (d_3) to node [right] {$r$} (d_5);
    \draw[role] (d_4) to node [above] {$r$} (d_5);
    \draw[role] (d_5) to node [right] {$r$} (d_6);
    \draw[role] (d_4) to node [left] {$s$} (d_6);

    \node (e_1) [answer, right = 1.5cm of d_1] {};
    \node (e_2) [answer, below = of e_1] {};
    \node (e_3) [answer, right = of e_1] {};
    \node (e_4) [answer, below = of e_3] {};
    \draw[role] (e_1) to node [left] {$s$} (e_2);
    \draw[role] (e_2) to node [below] {$s$} (e_4);
    \draw[role] (e_3) to node [above] {$s$} (e_1);
    \draw[role] (e_3) to node [right] {$s$} (e_4);
\end{tikzpicture}
%
%
%
We believe that CQ$^{\text{csf}}$ includes many relevant CQs that
occur in practical applications. To substantiate this, we have
analyzed the 65 queries that are part of three widely used benchmarks for
ontology-mediated querying, namely Fishmark, LUBM$^\exists$, and NPD
\cite{DBLP:conf/semweb/BailAPWHGG12,DBLP:conf/semweb/LutzSTW13,DBLP:conf/edbt/LantiRXC15}. We
found that more than 85\% of the queries fall into CQ$^{\text{csf}}$
while less than 5\% fall into ELIQ$^{\text{sf}}$.

\paragraph{Universal Models.}
Let \Amc be an ABox and \Omc an \ELdr-ontology.  The \emph{universal
  model of \Amc and~\Omc}, denoted $\Umc_{\Amc,\Omc}$, is the
interpretation obtained by starting with \Amc and then `chasing' with
the CIs in the ontology which adds (potentially infinite) ditrees
below every $a \in \mn{ind}(\Amc)$. The formal definition is in
the appendix. The model is universal in that
$\Umc_{\Amc,\Omc} \models q(\bar a)$ iff $\Amc,\Omc \models q(\bar a)$
for all CQs $q(\bar x)$ and tuples
$\bar a \in \mn{ind}(\Amc)^{|\bar x|}$. It can be useful to represent
  universal models in a finite way, as for
  example in the combined approach to ontology-mediated querying~\cite{DBLP:conf/ijcai/LutzTW09}.  Here, we introduce a
  finite representation that is tailored towards our class 
  CQ$^{\text{csf}}$.

  The \emph{3-compact model} $\Cmc^3_{\Amc,\Omc}$ of \Amc and \Omc is
  defined as follows. Let $\mn{sub}(\Omc)$ be the set of all
  concepts in \Omc, closed under subconcepts.  $\Cmc^3_{\Amc,\Omc}$
  uses the individual names from \Amc as well as individual names of
  the form $c_{a,i,r,C}$ where $a \in \mn{ind}(\Amc)$,
  $0 \leq i \leq 4$, $r$ is a role name from \Omc, and
  $C \in \mn{sub}(\Omc)$.  For every role name~$r$, we use $C_r$ to
  denote the conjunction over all $C$ such that
  $\exists r^- . \top \sqsubseteq C \in \Omc$, and $\top$ if the 
  conjunction is empty.  Let $i \oplus 1$ be short for
  $(i \bmod 4)+1$.  Define
\[
\begin{array}{r@{\;}c@{\;}l}
  \Cmc^3_{\Amc,\Omc}&:=& \Amc \cup \{ A(a) \mid \Amc,\Omc
                         \models A(a) \}  \,\cup \\[1mm]
                    && \{ A(c_{a,i,r,C}) \mid 
                       \Omc \models C \sqcap C_r
                       \sqsubseteq A \} \, \cup \\[1mm]
                    &&  \{ r(a, c_{a,0,r,C}) \mid \Amc,\Omc \models 
                       \exists r . C(a) \}
                       \, \cup \\[1mm]
                       &&  \{ r(c_{a,i,s,C}, c_{a,i \oplus 1 ,r,C'}) \mid \Omc \models 
                          C \sqcap C_s \sqsubseteq \exists r . C' \}.
\end{array}
\]
There is a homomorphism from $\Umc_{\Amc,\Omc}$ to
$\Cmc^3_{\Amc,\Omc}$ that is the identity on $\mn{ind}(\Amc)$,
but in general not vice versa. Nevertheless, $\Cmc^3_{\Amc,\Omc}$
is universal for CQ$^{\text{csf}}$.
\begin{restatable}{lemma}{lemcompactunivers}
  \label{lem:compactunivers}
  Let \Amc be an ABox and \Omc an \ELdr-ontology. Then
  %
%
  $\Cmc^3_{\Amc, \Omc}$ is a model of $\Amc$ and $\Omc$
  such that
%
  for every CQ
    $q(\bar x) \in CQ^{\text{csf}}$
    and $\bar a \in \mn{ind}(\Amc)^{|\bar x|}$, $\Cmc^3_{\Amc,\Omc}
    \models q(\bar a)$ iff $\Amc,\Omc \models q(\bar a)$.
  %
\end{restatable}
$\Cmc^3_{\Amc,\Omc}$ is defined so as to avoid spurious cycles of
length at most $3$ while larger spurious cycles are irrelevant for CQs
that are chordal. This explains the superscript $\cdot^3$ and enables
the lemma below. $\Cmc^3_{\Amc,\Omc}$ also avoids spurious
predecessors connected via different role names. Spurious predecessors
connected via the same role name cannot be avoided, but are irrelevant
for CQs that are symmetry-free.
\begin{restatable}{lemma}{lemnoanocycles}
  \label{lem:noanocycles}

    
  Every cycle in $\Cmc^3_{\Amc,\Omc}$ of length at most three consists
  only of individuals from $\mn{ind}(\Amc)$.
\end{restatable}
We also use the direct product $\Imc_1 \times \Imc_2$ of
interpretations $\Imc_1$ and $\Imc_2$, defined 
in the standard way (see appendix). For tuples of
individuals $\bar a_i = (a_{i,1},\dots,a_{i,n})$, $i \in \{1,2\}$, we
set
$\bar a_1 \otimes \bar a_2=((a_{1,1},a_{2,1}),\dots,
(a_{1,n},a_{2,n}))$.

\section{Learning under \texorpdfstring{\ELdr}{ELr}-Ontologies}

We establish polynomial time learnability results under
\ELdr-ontologies for the query classes CQ$^{\text{csf}}$, ELQ, and
ELIQ$^{\text{sf}}$.  For CQ$^{\text{csf}}$, we additionally have to
assume that the arity of CQs to be learned is bounded by a
constant or that unrestricted CQs can be used in equivalence
queries. When speaking of equivalence queries, we generally imply that
the CQs used in such queries must be from the class of CQs to be
learned. If this is not the case and unrestricted CQs are admitted in
equivalence queries, then we speak of \emph{CQ-equivalence
  queries}. When using CQ-equivalence queries, the learned
representation of the target query is a CQ, but need not
necessarily belong to \Cmc (though it is equivalent to a query from~\Cmc). For $w \geq 0$, let CQ$^{\text{csf}}_w$ be the restriction
of CQ$^{\text{csf}}$ to CQs of arity at most $w$. The following
are the main results obtained in this section.
\begin{theorem}
  \label{thm:mainthree}
  ~\\[-4mm]
  \begin{enumerate}
  \item   ELQ- and ELIQ$^{\text{sf}}$-queries are polynomial time learnable
    under \ELdr-ontologies using membership and equivalence queries;

    \item for every $w \geq 0$,
    CQ$^{\text{csf}}_{w}$-queries are polynomial time learnable
    under \ELdr-ontologies using membership and equivalence queries;

  \item   CQ$^{\text{csf}}$-queries are polynomial time learnable
    under \ELdr-ontologies using membership and CQ-equivalence queries.
    
  \end{enumerate}
\end{theorem}
\noindent
Before providing a proof of Theorem~\ref{thm:mainthree}, we show that
both membership and equivalence queries are needed for
polynomial learnability. Let AQ$^\wedge$ denote the class of unary CQs
of the form $q(x) \leftarrow A_1(x) \wedge \cdots \wedge A_n(x)$, 
and let a
\emph{conjunctive ontology} be an \EL-ontology without
role names.
\begin{restatable}{theorem}{thmonlyonequerytype}
  \label{thm:onlyonequerytype}
~\\[-4mm]
  \begin{enumerate}
  \item AQ$^\wedge$-queries are not polynomial query
    learnable under conjunctive ontologies using only
    membership queries;

   \item ELQ-queries are not polynomial time learnable (without
        ontologies) using only CQ-equivalence queries unless $\Pclass
        = \NPclass$.



  \end{enumerate}

\end{restatable}
Note that Points~1 and~2 of Theorem~\ref{thm:onlyonequerytype} imply
the same statements for all relevant query classes, that is, ELQ,
$\text{ELIQ}^{\text{sf}}$, $\text{CQ}^{\text{csf}}$,
$\text{CQ}^{\text{csf}}_w$ for all $w \geq 1$, and CQ, in place of the 
classes mentioned in the theorem. In particular, Point~2 implies that unrestricted CQs are not
polynomial time learnable with only equivalence queries in the
classical setting (without ontologies) unless $\Pclass = \NPclass$,
even when only unary and binary relations are admitted, see
\cite{COHEN19951,DBLP:journals/ml/Haussler89,DBLP:conf/alt/Hirata00}
for related results.  The proof of Point~1 follows basic lower bound
proofs for abstract learning problems
\cite{DBLP:journals/ml/Angluin87}. Point~2 is proved by exploiting
connections between active learning and inseparability
questions studied in
\cite{DBLP:conf/ijcai/FunkJLPW19,aaaithis,mauricemaster}.

\subsection{Reduction to Normal Form}
\label{sect:nf}

We show that the ontology under which we learn can w.l.o.g.\ be
assumed to be in normal form.  It is well-known that every
\ELdr-ontology \Omc can be converted into normal form
by introducing
fresh concept names~\cite{DL-Textbook}. We use such a
conversion to show that, for the relevant classes of CQs, a polynomial
time learning algorithm under \ELdr-ontologies in normal form can be
converted into a polynomial time learning algorithm under unrestricted
\ELdr-ontologies. Care has to be exercised as the fresh concept
names 
can occur in membership and
equivalence queries.
From now on, we thus assume that ontologies are in normal form. 
%
%
\begin{proposition}
  \label{prop:nf}
  Let
  $\Qmc \in \{ \text{ELQ}, \text{ELIQ}^{\text{sf}},
  \text{CQ}^{\text{csf}}_w \mid w \geq 0 \}$. If queries in \Qmc are
  polynomial time learnable under \ELdr-ontologies in normal form
  using membership and equivalence queries, then the same is true for
  unrestricted \ELdr-ontologies.
\end{proposition}

\subsection{Algorithm Overview}
\label{sect:algoview}

We start with 
proving Points~1 and~2 of
Theorem~\ref{thm:mainthree}. Thus let
$\Qmc \in \{ \text{ELQ}, \text{ELIQ}^{\text{sf}},
\text{CQ}^{\text{csf}}_w \mid w \geq 0 \}$.  The algorithm that
establishes polynomial time learnability of queries from \Qmc under
\ELdr-ontologies is displayed as Algorithm~\ref{alg:abox}. We next
explain some of its details.
\begin{algorithm}[t]
  \caption{Learning queries $q_T$ from 
ELQ / ELIQ$^{\text{sf}}$ / CQ$^{\text{csf}}_w$
under an 
    \ELdr-ontology \Omc.}\label{alg:abox}
\begin{algorithmic}
\Procedure{LearnCQ}{}
\State $q_H(\bar x) := \mn{refine}(q^\bot(\bar x_0))$
  \While{$q_H \not\equiv_\Omc q_T$ (equivalence query)}
  \State Let $\Amc, \bar a$ be the positive counterexample returned
  \State and let $q'_H(\bar x')$ be $\Cmc^3_{\Amc_{q_H},\Omc} \times \Cmc^3_{\Amc,\Omc}$ viewed as a CQ
  \State \phantom{and let} with answer variables $\bar x'=\bar x \otimes \bar a$
    \State $q_H(\bar x) := \mn{refine}(q'_H(\bar x'))$
  \EndWhile 
  \State \Return $q_H(\bar x)$
\EndProcedure 
\end{algorithmic}
\end{algorithm}

Let \Omc be an \ELdr-ontology, $\Sigma$ a finite signature that contains all
symbols in \Omc, and $\mn{ar} \leq w$ an arity for the query to be
learned with $\mn{ar}=1$ if
$\Qmc \in \{ \text{ELQ}, \text{ELIQ}^{\text{sf}} \}$, all known to the
learner and the oracle.  Further let $q_T(\bar y) \in \Qmc$ be the
target query known to the oracle, formulated in signature
$\Sigma$. The algorithm maintains and repeatedly updates a hypothesis
CQ $q_H(\bar x)$ of arity~$\mn{ar}$.  It starts with
the hypothesis
\[
q^\bot(\bar x_0) 
\leftarrow \{ A(x_0) \mid A \in \Sigma \cap \NC \} \cup \{ r(x_0, x_0) \mid r \in \Sigma \cap \NR\}
\]
where $\bar x_0$ contains only the variable $x_0$, repeated $\mn{ar}$ times. 
By construction, $q^\bot \subseteq_\Omc q$ for all
CQs $q$ of arity $\mn{ar}$ that use only symbols from $\Sigma$.
Note that $q^\bot \in \text{CQ}^{\text{csf}}_w$ for all
$w$, but $q^\bot$ is neither in  ELQ nor in ELIQ$^{\text{sf}}$.

If
$q_1(\bar x_1), q_2(\bar x_2),\dots$ are the hypotheses constructed
during a run of the algorithm, then for all
$i \geq 1$:
\begin{enumerate}

\item $q_i \in \Qmc$ and $q_i \subseteq_\Omc q_T$;
  
\item $q_i \subseteq_\Omc q_{i+1}$ and $q_i \not\equiv_\Omc q_{i+1}$;

\item $|\mn{var}(q_i)| \leq |\mn{var}(q_T)|$.
  
\end{enumerate}
Taken together, Points~1 and~2 mean that the hypotheses approximate
the target query from below in an increasingly better way and Point~3
is crucial for proving that we must reach $q_T$ after polynomially
many steps. The fact that \Omc is in normal form is used to attain
Point~3.

Point~1 also guarantees that the oracle always returns a
\emph{positive} counterexample $\Amc, \bar a$ to the equivalence query
used to check whether $q_H \not\equiv_\Omc q_T$ in the while loop. The
algorithm extracts the commonalities of $q_H(\bar x)$ and
$\Amc,\bar a$ by means of a direct product with the aim of obtaining a better
approximation
of the target. The same is done in the case
without ontologies \cite{DBLP:journals/tods/CateDK13} where 
$\Amc_{q_H} \times \Amc$ (viewed as a
CQ) is the new hypothesis, but this is not sufficient here as it misses
the impact of the ontology. The product
$\Umc_{\Amc_{q_H},\Omc} \times \Umc_{\Amc,\Omc}$ would work,
but need
not be finite. So we resort to
$\Cmc^3_{\Amc_{q_H},\Omc} \times \Cmc^3_{\Amc,\Omc}$ instead, viewed
as a CQ $q'_H(\bar x')$.  This new hypothesis need not belong to \Qmc,
so we call the subroutine $\mn{refine}$ detailed in the
subsequent section to convert it into a new hypothesis
$q_H(\bar x) \in \Qmc$ such that
$q'_H \subseteq_\Omc q_H \subseteq_\Omc q_T$.
The initial call to \mn{refine} serves the same
purpose as $q^\bot(\bar x_0)$ need not be in \Qmc,
depending on the choice of \Qmc.

It is not immediately clear that the described approach achieves the
containment in Point~2 since
$\Cmc^3_{\Amc_{q_H},\Omc} \times \Cmc^3_{\Amc,\Omc}$ is
potentially too
strong as a replacement of
$\Umc_{\Amc_{q_H},\Omc} \times \Umc_{\Amc,\Omc}$; in particular, there
might be cycles in the former product that do not exist in the latter.
What saves us, however, is that the CQ $q_H$ constructed by
\mn{refine} belongs to \Qmc while the models
$\Cmc^3_{\Amc_{q_H},\Omc}$ and $\Cmc^3_{\Amc,\Omc}$ are universal for
\Qmc as per Lemma~\ref{lem:compactunivers}.


\subsection{The \mn{refine} Subroutine}
\label{sec:short-cycles}

The \mn{refine} subroutine gets as input a CQ $q'_H(\bar x')$ that
does not need to be in \Qmc, but that satisfies $q'_H \subseteq_\Omc q_T$.
It 
produces a query $q_H(\bar x)$ from \Qmc such that
$q'_H \subseteq_\Omc q_H \subseteq_\Omc q_T$ and
$|\mn{var}(q_H)| \leq |\mn{var}(q_T)|$.  For notational convenience,
we prefer to view $q'_H(\bar x')$ as a pair $(\Amc,\bar a)$ where
$\Amc=\Amc_{q'_H}$ and $\bar a = \bar x'$.
Let $n_{\max}$ denote the maximum length of a chordless cycle in
any query in \Qmc, that is $n_{\max} = 0$ for
$\Qmc \in \{ \text{ELQ}, \text{ELIQ}^{\text{sf}} \}$ and $n_{\max} =3$
for $\Qmc= \text{CQ}^{\text{csf}}_w$, $w \geq 0$. We shall 
use the following.

\medskip\noindent\textbf{Minimize.} Let \Bmc be an ABox and $\bar b$ a
tuple such that $\Bmc,\Omc \models q_T(\bar b)$.  Then
$\mn{minimize}(\Bmc,\bar b)$ is the ABox $\Bmc'$ obtained from $\Bmc$
by exhaustively applying the following operations:

\smallskip  
\noindent  
 (1) choose $c \in \mn{ind}(\Bmc)\setminus \bar b$ and
    remove all assertions that involve $c$. Use a membership query to
    check whether, for the resulting ABox $\Bmc^-$, $\Bmc^-,\Omc \models
    q_T(\bar b)$. If so, proceed with $\Bmc^-$ in place of \Bmc.

\smallskip 
\noindent 
(2) choose  $r(a, b) \in \Bmc$ and use a membership query to
    check whether $\Bmc \setminus \{ r(a, b) \}, \Omc \models q_T(\bar
    b)$. If so, proceed with $\Bmc \setminus \{
    r(a, b) \}$ in place of \Bmc.

\smallskip  
\noindent  
    The \mn{refine} subroutine builds a
    sequence $(\Bmc_1,\bar b_1),(\Bmc_2,\bar b_2),\ldots$ starting
    with $(\Bmc_1,\bar b_1) = (\mn{minimize}(\Amc,\bar a),\bar a)$ and
    exhaustively applying the following step:

\medskip\noindent\textbf{Expand.} Choose a chordless cycle
$R_0(a_0,a_1),\ldots,$ $R_{n-1}(a_{n-1},a_n)$ in $\Bmc_i$ with
$n>n_{\max}$ and,
in case
that $\Qmc =\text{CQ}^{\text{csf}}_w$,
$\{a_0,\ldots,a_{n-1}\}\not\subseteq \bar b_i$.\footnote{This is
  because $\text{CQ}^{\text{csf}}$ admits cycles that consist only of answer
  variables while ELQ and ELIQ$^{\text{sf}}$ do not.}
Let $\Bmc'_{i}$
be the ABox obtained by doubling the length of the cycle: start with $\Bmc_i$, introduce copies
$a_0',\ldots,a_{n-1}'$ of $a_0,\ldots,a_{n-1}$, and then
  \begin{itemize}

    \item remove all assertions $R(a_{n-1},a_{0})$;


    \item add $B(a_i')$ if $B(a_i)\in
      \Bmc_i$;

    \item add $R(a_i',c)$ if $R(a_i,c)\in
      \Bmc_i$ with $0 \leq i < n$ and $c\in
      \mn{ind}(\Bmc_i)\setminus\{a_0,\ldots,a_{n-1}\}$;

    \item add $R(a_i',a_{j}')$ if $R(a_i,a_j)\in
      \Bmc_i$ with $0 \leq i,j <n$ and $\{i,j\}\neq\{0,n-1\}$;

    \item add $R(a_{n-1},a_0')$ and $R(a_{n-1}',a_0)$ if $R(a_{n-1},a_0)\in
      \Bmc_i$.

  \end{itemize}
  A similar construction is used in \cite{DBLP:conf/aaai/KonevOW16}.
%
%
  Let $\tau_i$ be the set of tuples $\bar b$ obtained from
  $\bar b_i=(b_1,\ldots,b_k)$ by replacing any number of components
  $b_j$ by $b_j'$. Use membership queries to identify
  $\bar b_{i+1} \in \tau_i$ with
  $\Bmc_{i}',\Omc\models q_T(\bar b_{i+1})$ and set
  $\Bmc_{i+1}=\mn{minimize}(\Bmc'_i,\bar b_{i+1})$.
  We prove in the
  appendix that such a $\bar b_{i+1}$ always exists and that the
  Expand step can only be applied polynomially many times. The
  resulting $(\Bmc_n,\bar b_n)$ viewed as a CQ with answer variables
  $\bar b_n$ is chordal, but not necessarily symmetry-free. To establish
  also the latter, we compute a sequence of ABoxes
  $\Bmc_n,\Bmc_{n+1},\dots$ by exhaustively applying the following
  step:

\medskip\noindent\textbf{Split.} Choose $r(a,b),r(c,b)\in \Bmc_i$ such
that $b\notin\bar b_n$ and neither $r(a,b)$ nor $r(c,b)$ occurs on a
cycle. Construct $\Bmc_{i}'$ by removing $r(a,b)$ from $\Bmc_i$,
taking a fresh individual $b'$, and adding $B(b')$ for
all $B(b)\in \Bmc_i$ and $S(d,b')$ for all $S(d,b)\in \Bmc_i$ with
$S(d,b) \neq r(c,b)$.  If
$\Bmc_i',\Omc\models q_T(\bar b_n)$, then
$\Bmc_{i+1}=\mn{minimize}(\Bmc_i',\bar b_n)$.

\medskip We prove in the appendix that only
polynomially many
applications are possible and that, for $\Bmc_m$ the resulting ABox,
%
%
$(\Bmc_m',\bar b_n)$ viewed as a CQ is chordal and symmetry-free.
Moreover, it is in ELQ if
$q_T$ is, and likewise for ELIQ$^{\text{sf}}$. \mn{Refine} returns this CQ as its result.
Note that the running
time of \mn{refine} depends exponentially on \mn{ar} due to the brute
force search for a tuple $\bar b_{i+1} \in \tau_i$ in the Expand step.

\subsection{Unbounded Arity}
\label{sect:refine}

To prove the remaining Point~3 of Theorem~\ref{thm:mainthree}, we have
to deal with CQs of unbounded arity and cannot use the
\mn{refine} subroutine presented in Section~\ref{sec:short-cycles}.
We thus introduce a second version of \mn{refine} that
works rather differently from the previous one. We give an
informal description, full details are in the appendix.

Recall that refinement starts with the product
$P=\Cmc^3_{\Amc_{q_H},\Omc} \times \Cmc^3_{\Amc,\Omc}$. In
Section~\ref{sec:short-cycles}, we blow up cycles in $P$, not
distinguishing the ABox part and the existentially generated part of
the 3-compact models involved. The second version of
\mn{refine} instead unravels the existentially generated part of the two
3-compact models inside the product $P$. A full such unraveling
would eventually result in $\Umc_{\Amc_{q_H},\Omc} \times
\Umc_{\Amc,\Omc}$, but we interleave with a Minimize step as in
Section~\ref{sec:short-cycles} and thus obtain a finite initial
piece thereof. Unlike in the previous version of \mn{refine}, we do
not have to redefine the answer variables at all (but note that they
may still change outside of \mn{refine} when we take the product).

The above suffices for target CQs from CQ$^{\text{csf}}$ in which
every variable is reachable from an answer variable. In the general
case, disconnected Boolean components might be present (or emerge
during unraveling and minimization) that are never unraveled. To
address this, we subsequently apply the original version of
\mn{refine} to such components, avoiding the Splitting step and
leaving the already unraveled parts untouched. Note that the exponential blowup
in the arity is avoided because the original \mn{refine} is only applied to
Boolean subqueries. However, the resulting queries are not guaranteed
to be in $\text{CQ}^{\text{csf}}$. 
We can thus
not rely on Lemma~\ref{lem:compactunivers} as before
which is why we need CQ-equivalence queries.

\section{{Learning under \texorpdfstring{\ELI}{ELI}-Ontologies}}
\label{sect:lower}

When we replace \ELdr-ontologies with \ELI-ontologies, polynomial time
learnability can no longer be expected since containment between ELQs
under \ELI-ontologies is \ExpTime-complete
\cite{BaaderEtAl-OWLED08DC}. In contrast, polynomial query
learnability is not ruled out and in fact it is natural to ask whether
there is a polynomial time learning algorithm with access to an oracle
(in the classical sense) for query containment under \ELI-ontologies.
Note that such an algorithm would show polynomial query learnability.
We answer this question to the negative and show that polynomial query
learnability cannot be attained under \ELI-ontologies for any of the
query classes considered in this paper. This is a consequence of the
following result, which also captures learning of unrestricted CQs.
\begin{restatable}{theorem}{thmellower} \label{thm:el-lower}
  \EL-concepts are not polynomial query learnable under
  \ELI-ontologies with membership queries and CQ-equivalence queries.
\end{restatable}
For the proof, we use the \ELI-ontologies $\Omc_n$, $n \geq 1$,
given in Figure~\ref{fig:eliont}.
\begin{figure}[t]
  \centering
  \begin{boxedminipage}{\columnwidth}
    {\small
\[
\begin{array}{@{}r@{\;}c@{\;}l@{\;}l}
  ~&&\\[-6mm]
  \top & \sqsubseteq & \multicolumn{2}{l}{\!\!\!\exists r . \top  \sqcap \exists s . \top}
  \\[2mm]
  L_i & \sqsubseteq & \multicolumn{2}{l}{\!\!\!\exists r . L_{i+1} \sqcap
                      \exists s . L_{i+1}} \\[1mm]
  &&& \text{for } 0 \leq i \leq n \\[1mm]
    L_i & \sqsubseteq & \exists r . L_{i+1} & \text{for } n \leq i < 2n\\[1mm]
L_{2n} & \sqsubseteq & A \\[2mm]
  \exists \sigma . L_{i+1} & \sqsubseteq & L_i &\text{for } \sigma \in \{r,s\}
                                                 \text{ and } 0 \leq i \leq
                                                 2n \\[2mm]
  K_i & \sqsubseteq & \multicolumn{2}{l}{\!\!\!\exists r . (K_{i+1}  \sqcap V^r_{i+1}) 
                      \sqcap \exists s . (K_{i+1} \sqcap V^s_{i+1}) } \\[1mm]
  &&& \text{for } \sigma \in \{r,s\} \text{ and } 0 \leq i \leq n \\[1mm]
    K_i \sqcap W^\sigma_{i+1} & \sqsubseteq& \exists r . K_{i+1} &
  \text{for } \sigma \in \{r,s\} \text{ and } n \leq i < 2n \\[1mm]
  \exists \sigma^- . (K_j \sqcap V^{\sigma'}_i) & \sqsubseteq &
                                                                V^{\sigma'}_i
  &  \text{for } \sigma,\sigma' \in \{r,s\}, 1 \leq i \leq n, \\[1mm]
&&&                                               \text{and }  i 
                                                            \leq j
    \leq 2n\\[1mm]
  K_{2n} \sqcap V^\sigma_i \sqcap W^{\overline{\sigma}}_i &\sqsubseteq& A
  & \text{for } \sigma \in \{r,s\} \text{ and } 1
                                                         \leq i \leq n\\[2mm]
  \exists \sigma . W^{\sigma'}_i & \sqsubseteq & W^{\sigma'}_i &
                                                                 \text{for
                                                                 }
                                                                 \sigma
                                                                 \in
                                                                 \{r,s,r^-,s^-\},\\[1mm]
&&&                                       \sigma'
                                                                 \in
                                                                 \{r,s\},         \text{ and } 1 \leq i
                                               \leq n \\[1mm]
  W^r_i \sqcap W^s_i & \sqsubseteq & L_0 &\text{for } 
                                                 0 \leq i \leq
                                           n\\[2mm]
    \exists \sigma . K_{i+1} & \sqsubseteq & K_i &\text{for } \sigma \in \{r,s\}
                                                 \text{ and } 0 \leq i \leq 
                                                   2n \\[2mm]
  \exists \sigma^- . \top & \sqsubseteq & U_1^\sigma & \text{for } \sigma
                                                  \in \{r,s\} \\[1mm]
  \exists \sigma^- . U^{\sigma'}_i & \sqsubseteq &  U^{\sigma'}_{i+1} &
                                                                      \text{for
                                                                      }
                                                                      \sigma,\sigma'
                                                                      \in
                                                                      \{r,s\}
                                                                      \text{
                                                                      and
                                                                      }
                                                                      1
                                                                      \leq
                                                                      i
                                                                      <2n \\[1mm]
  U^r_i \sqcap U^s_i & \sqsubseteq & D & \text{for } 1 \leq i \leq 2n \\[2mm]
  K_i \sqcap A &\sqsubseteq& D & \text{for } 0 \leq i < 2n \\[1mm]
  L_i \sqcap A &\sqsubseteq& D & \text{for } 0 \leq i < 2n \\[1mm]
  L_i \sqcap L_j &\sqsubseteq& D & \text{for } n \leq i < j \leq 2n \\[1mm]
  K_i \sqcap K_j &\sqsubseteq& D & \text{for } n \leq i < j \leq 2n \\[1mm]
  L_i \sqcap K_j &\sqsubseteq& D & \text{for } n \leq i,j \leq 2n \\[2mm]
  \exists \sigma . D & \sqsubseteq & D &\text{for } \sigma \in \{r,s,r^-,s^-\}
  \\[1mm]
  D &\sqsubseteq& L_0\\[-3mm]
\end{array}
\]
}
\end{boxedminipage}
\caption{\ELI-ontology $\Omc_n$}
\label{fig:eliont}
\vspace*{-4mm}
\end{figure}
There, $\overline{r}=s$ and $\overline{s}=r$. Every $\Omc_n$ is
associated with a set $\Hmc_n$ of $2^n$ potential target concepts
of the form
\[\exists \sigma_1 \cdots \exists \sigma_n .  \exists r^n .  A \text{
  with }
\sigma_1,\dots,\sigma_n \in \{r,s\}\]
where $\exists r^n$ denotes the $n$-fold nesting of $\exists r$.  The
idea of the proof is to show that if there was an algorithm for
learning \EL-concepts under \ELI-ontologies such that, at any given
time, the sum of the sizes of all (membership and CQ-equivalence)
queries asked to the oracle is bounded by a polynomial
$p(n_1,n_2,n_3)$ with $n_1$ is the size of the target query, $n_2$ is
the size of the ontology, and $n_3$ is the size of the largest
counterexample seen so far, then we can choose $n$ large enough so
that the learner needs more than $p(n_1,n_2,n_3)$ queries to
distinguish the targets in $\Hmc_n$ under $\Omc_n$ if the oracle uses
a `sufficiently destructive' strategy to answer the queries. Such a
strategy is presented in the appendix, we only give one example that
highlights a crucial aspect.

Assume that the learner poses as an equivalence query the \EL-concept
$C_H=\exists \sigma_1 \cdots \exists \sigma_n .  \exists r^n .  A$.
Then the oracle returns ``no'' and positive counterexample
$\Amc=\{ K_0(a_0),W^{\sigma_1}_1(a_0),\dots,W^{\sigma_n}_n(a_0) \}$.
It is instructive to verify that $\Amc,\Omc \models C'_H(a_0)$ for all
$C'_H \in \Hmc_n \setminus \{ C_H \}$ while
$\Amc,\Omc \not\models C_H(a_0)$ as this illustrates the use of
inverse roles in~$\Omc_n$. 



    
    


\section{Conclusion}

We conjecture that our results can be extended from \ELdr-ontologies
to \ELHdr-ontologies, thus adding role inclusions. 
In contrast, we do not know how to learn in polynomial time
unrestricted \ELI-concepts under \EL-ontologies, or symmetry-free CQs
under \EL-ontologies. 
We would not be surprised if these indeed turn out 
not to be learnable in polynomial time. 
It is an interesting 
question whether our results 
can be generalized 
to symmetry-free CQs that admit chordless cycles of length bounded by a constant 
larger than three. This would require the  use of  a different kind of 
compact universal model. 
%

\section*{Acknowledgements}
Supported by the DFG Collaborative Research Center 1320 EASE - Everyday Activity Science and Engineering.

%
%
%
%

\cleardoublepage

\bibliographystyle{named}
\bibliography{local}


\cleardoublepage

\appendix

\section{Appendix Preliminaries}

We introduce some additional preliminaries that are needed for the lemmas and
proofs in the appendix.

Let \Omc be an \ELdr-ontology in normal form. We say that an ABox
\Amc is \emph{\Omc-saturated} if $\Amc,\Omc \models A(a)$ implies
$A(a) \in \Amc$ for all concept names $A$ and $a \in \mn{ind}(\Amc)$.

\paragraph{Universal Models.}

The \emph{universal
  model of \Amc and~\Omc}, denoted $\Umc_{\Amc,\Omc}$, is the
interpretation defined as follows. For every role name~$r$, we use
$C_r$ to denote the conjunction over all $C$ such that $\exists r^- . \top
\sqsubseteq C \in \Omc$; note that $C_r=\top$ if there is no such
range restriction in \Omc. 
A
\emph{trace} for \Amc and \Omc is a sequence
$t=a r_1 C_1 r_2 C_2 \dots r_n C_n$, $n \geq 0$, such that
$a \in \mn{ind}(\Amc)$,
$\{\exists r_1 . C_1,\dots, \exists r_n . C_n\} \subseteq
\mn{sub}(\Omc)$, $\Amc,\Omc \models \exists r_1.C_1(a)$,
$\Omc \models C_i \sqcap C_{r_i} \sqsubseteq \exists r_{i+1} . C_{i+1}$ for
$1 \leq i < n$. 
%
Let \Tbf denote the set of all traces for \Amc and~\Omc.  Then
\[
\begin{array}{r@{\;}c@{\;}l}
  \Umc_{\Amc,\Omc}&:=& \Amc \cup \{ A(a) \mid \Amc,\Omc
                       \models A(a) \}  \,\cup \\[1mm]
                  && \{ A(trC) \mid trC\in \Tbf\ \text{and}\ 
                            \Omc \models C \sqcap C_r
                            \sqsubseteq A \} \, \cup \\[1mm]
                  &&  \{ r(t,trC) \mid t r
                     C \in \Tbf \}
\end{array}
\]

The following two lemmas are the main properties of 
$\Umc_{\Amc,\Omc}$. They connect the notions of the universal model, homomorphisms, and
queries and are well known.
\begin{lemma}
  \label{lem:univers}
  Let \Amc be an ABox and \Omc an \ELdr-ontology. Then
   \begin{enumerate}

   \item 
    $\Umc_{\Amc, \Omc}$ is a model of $\Amc$ and $\Omc$;

 \item
for every CQ
    $q(\bar x)$
    and $\bar a \in \mn{ind}(\Amc)^{|\bar x|}$, $\Umc_{\Amc,\Omc}
    \models q(\bar a)$ iff $\Amc,\Omc \models q(\bar a)$.
    
 \end{enumerate}
\end{lemma}
\begin{lemma}
    \label{lem:homlem}
    For an \ELdr-ontology $\Omc$ and CQs $q_1(\bar x_1)$ and $q_2(\bar x_2)$
    of the same arity, the following are equivalent:
    \begin{enumerate}
        \item $q_1 \subseteq_\Omc q_2$,
        \item there is a homomorphism $h$ from $q_2$ to
          $\Umc_{\Amc_{q_1}, \Omc}$ with $h(\bar x_1)=\bar x_2$.
    \end{enumerate}
\end{lemma}

Let $\Imc_1, \Imc_2$ be interpretations and let $h$ be a mapping from
$\Delta^{\Imc_1}$ to $\Delta^{\Imc_2}$.  The \emph{image} of $h$, denoted
$\mn{img}(h)$, is the set $\{ b \in \Delta^{\Imc_2} \mid \exists a : h(a) = b
\}$.
Let $g$ be a homomorphism from an ABox $\Amc$ to a universal model
$\Umc_{\Bmc, \Omc}$. Then let $g^*$ be defined to be a mapping from
$\mn{ind}(\Amc)$ to $\mn{ind}(\Bmc)$ by setting $g^*(a) = b$ when $g(a)$ is a
trace of shape $bw$ for $w$ a potentially empty sequence $r_1C_1 \cdots r_n
C_n$.

\paragraph{Direct Products.}
The \emph{direct product} of interpretations $\Imc_1$ and $\Imc_2$
is the interpretation $\Imc_1 \times \Imc_2$ defined as
\[
\begin{array}{l}
\{ \top(a_1,a_2) \mid a_i \in
            \Delta^{\Imc_i} \text{ for } i \in \{1,2\} \} \cup {}
  \\[1mm]
 \{ A(a_1,a_2) \mid A(a_i) \in \Imc_i \text{
                          for } i \in \{1,2\} \} \, \cup \\[1mm]
 \{ r((a_1,a_2),(b_1,b_2)) \mid r(a_i,b_i) \in \Imc_i \text{
  for } i \in \{1,2\} \}.
\end{array}
\]
If $\bar d_i = (d_{i,1},\dots,d_{i,n}) \in \Delta^{\Imc_i}$ for $i \in \{1,2\}$,
then we use $\bar d_1 \otimes \bar d_2$ to denote the tuple
$((d_{1,1},d_{2,1}),\dots, (d_{1,n},d_{2,n}))$.

The following are some basic facts about products that are straightforward to show.
\begin{lemma}
  \label{lem:prodfund}
  Let $\Imc$, $\Imc_1$ and $\Imc_2$ be interpretations. Then 
  \begin{enumerate}

      \item for $i \in \{1, 2\}$ there is a homomorphism $h$ from $\Imc_1 \times \Imc_2$ to $\Imc_i$
      such that $h(d_1, d_2) = d_i$ for all $(d_1, d_2) \in \Delta^{\Imc_1
      \times \Imc_2}$;
 
  \item if for all $i \in \{1, 2\}$ there is a homomorphism $h_i$ from
      $\Imc$ to $\Imc_i$ with $h_i(d) = d_i$, then there is a homomorphism $h$
      from $\Imc$ to $\Imc_1 \times \Imc_2$ with $h(d) = (d_1, d_2)$;

  \item if  $\Imc_1$ and 
    $\Imc_2$ are models of \Omc, then so is 
    $\Imc_1 \times \Imc_2$. 
    
  \end{enumerate}
\end{lemma}

\paragraph{Simulations.}

A \emph{simulation} from interpretation
$\Imc_1$ to interpretation $\Imc_2$ is a relation $S \subseteq
\Delta^{\Imc_1} \times \Delta^{\Imc_2}$ that satisfies the following
conditions:
\begin{enumerate}

\item if $A(d) \in \Imc_1$ and $(d,e) \in S$, then
  $A(e) \in \Imc_2$;

\item if $r(d,d') \in \Imc_1$ and $(d,e) \in S$, then
  there is an $r(e,e') \in \Imc_2$ with $(d',e') \in S'$.

\end{enumerate}
We further say that $S$ is a simulation from $\Imc_1,d_1$ to $\Imc_2,d_2$
if $S$ is a simulation from $\Imc_1$ to $\Imc_2$ with $(d_1,d_2)\in S$; we
write $\Imc_1,d_1\preceq\Imc_2,d_2$ if such a simulation $S$ exists. 
The following are a basic facts about
simulations and homomorphisms that are standard to proof.
\begin{lemma}
  \label{lem:simfund}
  Let $\Imc_i$ be an interpretation and $d_i \in \Delta^{\Imc_i}$,
  for $i \in \{1,2\}$. Then  $\Imc_1,d_1\preceq\Imc_2,d_2$
  and $d_1 \in C^{\Imc_1}$ implies $d_2 \in C^{\Imc_2}$ for all
  \EL-concepts $C$.
\end{lemma}

\begin{lemma}
    \label{lem:universal-homomorph}
    Let \Omc be an \ELdr-ontology and $\Amc_1, \Amc_2$ ABoxes. Then
    \begin{enumerate}

      \item every homomorphism $h$ from $\Amc_1$ to $\Amc_2$ can be
	extended to a homomorphism $h'$ from $\Umc_{\Amc_1, \Omc}$ to
	$\Umc_{\Amc_2, \Omc}$ such that if $a\notin \mn{ind}(\Amc_1)$
	is a trace $a=bw$, then $h'(a)$ is a trace of shape
	$h(a)w'$, and in particular, $a \notin \mn{ind}(\Amc_1)$
	implies $h'(a)\notin \mn{ind}(\Amc_2)$;

      \item if $\Amc_1,a_1\preceq \Amc_2,a_2$, then
	$\Umc_{\Amc_1,\Omc},a_1\preceq \Umc_{\Amc_2,\Omc},a_2$.
    \end{enumerate}
 
\end{lemma}

\begin{lemma}
  \label{lem:simlem}
  Let $\Amc_1, \Amc_2$ be ABoxes, $a_1 \in \mn{ind}(\Amc_i)$ for $i \in \{1,
  2\}$, and $\Omc$ an $\ELdr$-ontology. If there is a simulation $S$ from
  $\Amc_1$ to $\Amc_2$ with $(a_1, a_2) \in S$, then $\Amc_1, \Omc \models
  C(a_1)$ implies $\Amc_2, \Omc \models C(a_2)$ for all \EL-concepts $C$.
\end{lemma}

Now we continue with a proof of a central property of the 3-compact model.

\lemnoanocycles*
\begin{proof}\
%
  The statement is clear by construction of
  $\Cmc^3_{\Amc,\Omc}$ for cycles of length~1. It is also clear for
  cycles of length~2 since for any pair of individuals of which at
  least one is of the form $c_{a,i,r,C}$, the ABox
  $\Cmc^3_{\Amc,\Omc}$ contains at most one assertion that involves
  both of them.
  
  Now for cycles of length~3. Assume to the contrary of what is to be
  shown that $\Cmc^3_{\Amc,\Omc}$ contains a cycle of length~3 that
  contains an individual not from $\mn{ind}(\Amc)$.  First assume that
  there is an individual $a \in \mn{ind}(\Amc)$ on the cycle. Since all
  individuals of the form $c_{b,i,r,C}$ that are on the cycle are adjacent to
  $a$ on the cycle, $b=a$ and $i=0$ for all such
  $c_{b,i,r,C}$. This implies that $a$ is the only individual from
  $a \in \mn{ind}(\Amc)$ on the cycle. But then the cycle contains two
  distinct individuals of the form $c_{a,0,r,C}$ that are connected by
  an edge, which is never the case in $\Cmc^3_{\Amc,\Omc}$.

  Now assume that the cycle contains only individuals of the form
  $c_{b,i,r,C}$. Then all these individuals are connected in
  $\Cmc^3_{\Amc,\Omc}$ by an edge. This is impossible due to the use
  of the index $i$ in the construction of $\Cmc^3_{\Amc,\Omc}$.
\end{proof}

To simplify some of our proofs, it is useful to consider a
strengthening of the condition of symmetry-freeness from the
definition of CQ$^{\text{csf}}$. We say that a CQ is \emph{strongly
  symmetry-free} if $\vp$ contains atoms $r(y_1,x),r(y_2,x)$, then $x$
is an answer variable or one of the atoms occurs on a cycle. Thus, the
possibility that $\vp$ contains an atom $s(z,z)$ for some
$z \in \{ x,y_1,y_2 \}$ from the original definition of
symmetry-freeness is excluded. In the following, we show that
every CQ from CQ$^{\text{csf}}$ is equivalent to one that is
strongly symmetry-free. 

Two CQs $q_1$ and $q_2$ of the same arity are \emph{equivalent},
written $q_1 \equiv q_2$, if $q_1 \equiv_\emptyset q_2$ for
$\emptyset$ the empty ontology.
\begin{lemma}
  \label{lem:strongsymfree}
  For every $q \in \text{CQ}^{\text{csf}}$, there is a
  $q' \in \text{CQ}^{\text{csf}}$ such that $q \equiv q'$ and $q'$ is
  strongly symmetry-free.
\end{lemma}
\begin{proof}
  To construct $q'$, start from $q$. Then introduce, for every
  quantified variable $x$ that occurs in an atom of the form $r(x,x)$
  in $q$, a fresh quantified variable $x'$ and add the atom $B(x')$
  for every atom $B(x)$ in $q$ and $S(y,x')$ for every atom $S(y,x)$
  in $q$. We say that $x'$ is a \emph{copy} of $x$.

  It is clear that $q \equiv q'$ as there is a homomorphism from
  $q'$ to $q$ and by Lemma~\ref{lem:homlem}.
  Moreover, $q'$ is strongly symmetry-free. To see this, assume
  that $q'$ contains atoms $r(x_1,y),r(x_2,y)$ with $y$ a quantified
  variable. We distinguish three cases.

  First assume $r(x_1,y),r(x_2,y)$ are already in $q$. Since $q$ is
  symmetry-free, $q$ also contains an atom of the form $s(x_1,x_1)$,
  $s(x_2,x_2)$, or $s(y,y)$. In the first case, $q'$ contains the
  cycle $r(x_1,y),r^-(y,x'_1),s^-(x'_1,x_1)$ and thus atom $r(x_1,y)$
  occurs on a cycle in $q'$. In the second case, atom $r(x_2,y)$ occurs on a
  cycle and in the third case both atoms do.

  Now assume that $r(x_1,y)$ is not in $q$. Then $x_1$ is a copy of a
  variable $x^0_1$ in $q$ or $y$ is a copy of a variable $y^0$ in $q$
  (or both).  In the first case, $q'$ contains a cycle of the form
  $r(x_1,y),r^-(y,x^0_1),s(x_1,x_1^0)$ and thus $r(x_1,y)$ occurs on a
  cycle in $q'$. In the second case, $q'$ contains a cycle of the form
  $r(x_1,y),s^-(y,y^0),r^-(y^0,x_1)$ and thus again $r(x_1,y)$ occurs on a
 cycle in $q'$.

  The case that $r(x_2,y)$ is not in $q$ is symmetric.
\end{proof}
It might seem that we should change the definition of the class
CQ$^{\text{csf}}$ to be based on strong symmetry-freeness. This,
however, is not possible because the CQ
produced by the first version of \mn{refine} is only symmetry-free,
but not strongly symmetry-free.  We can also not use the construction
from the proof of Lemma~\ref{lem:strongsymfree} as part of \mn{refine} to
attain strong symmetry-freeness as this interferes with minimization,
that is, it would no longer be guaranteed that
$|\mn{var}(q_H)| \leq |\mn{var}(q_T)|$.

\section{Proof of Lemma~\ref{lem:compactunivers}}

The following is easy to show, details are omitted.
\begin{lemma}
  \label{lem:cthreesimlem}
  Let \Amc be an ABox and \Omc an \ELdr-ontology.
  Further let $c_{a,i,r,C} \in \Delta^{\Cmc^3_{\Amc,\Omc}}$ and
  $trC \in \Delta^{\Umc_{\Amc,\Omc}}$. Then $\Cmc^3_{\Amc,\Omc},
  c_{a,i,r,C}
  \preceq \Umc_{\Amc,\Omc},trC$.
\end{lemma}

\lemcompactunivers*
\begin{proof}
  It is not difficult to prove that $\Cmc^3_{\Omc,\Amc}$ is indeed a
model of \Amc and \Omc, details are omitted. Let
$q(\bar x) \in \text{CQ}^{\text{csf}}$ and
$\bar a \in \mn{ind}(\Amc)^{|\bar x |}$. 
We have to show that
$\Cmc^3_{\Amc,\Omc} \models q(\bar a)$ iff
$\Amc,\Omc \models q(\bar a)$. By Lemma~\ref{lem:strongsymfree},
we can assume w.l.o.g.\ that $q$ is strongly symmetry-free.
The ``if'' direction is trivial by
Lemma~\ref{lem:univers} and because there is an obvious homomorphism
from $\Umc_{\Amc,\Omc}$ to $\Cmc^3_{\Amc,\Omc}$ that is the identity
on $\mn{ind}(\Amc)$. We thus concentrate on ``only if''. Let
$q(\bar x) \in \text{CQ}^{\text{csf}}$ and assume that
$\Cmc^3_{\Amc,\Omc} \models q(\bar a)$. Then there is a homomorphism
$h$ from $q$ to $\Cmc^3_{\Amc,\Omc}$ with $h(\bar x)=\bar a$.  In what
follows, we construct a homomorphism $g$ from $q$ to
$\Umc_{\Amc,\Omc}$ with $g(\bar x)=\bar a$.  Thus,
$\Amc,\Omc \models q(\bar a)$ as required.

  To start the definition of $g$, set $g(x)=h(x)$
  whenever $h(x) \in \mn{ind}(\Amc)$.  It follows from the
  construction of $\Cmc^3_{\Amc,\Omc}$ and $\Umc_{\Amc,\Omc}$ that $g$
  is a homomorphism from the restriction of $q$ to the domain of $g$
  to $\Umc_{\Amc,\Omc}$.

  As a consequence of Lemma~\ref{lem:noanocycles}, if a variable $x$
  occurs on a cycle of length~1 or~2 in $q$, then $g(x)$ is defined
  at this point.  We next define $g(x)$ for all variables $x_0$ that
  are on a cycle $R_0(x_0,x_1), R_1(x_1,x_2), R_2(x_2,x_0)$ of
  length~3 in $q$. Assume that $g(x_0)$ was not yet defined. It
  follows from Lemma~\ref{lem:noanocycles} that then
  $h(x_1)=h(x_2) \in \mn{ind}(\Amc)$, and thus $\Amc$ contains a
  reflexive $R_1$-cycle on $h(x_1)$, $R_0 =R_2^-$, and
  $h(x_0) \notin \mn{ind}(\Amc)$. Let $h(x_1)=a$. By construction of
  $\Cmc^3_{\Amc,\Omc}$, $h(x_0)=c_{a,0,r,C}$ for some 
  $C$ and where $r=R_0$ if $R_0$ is a role name and $r=R_2$
  otherwise. Set $g(x_0)=arC$. Also after the extension, $g$ is
  a homomorphism from the restriction of $q$ to the (now extended)
  domain of $g$ to $\Umc_{\Amc,\Omc}$.  This is easily seen to be a
  consequence of the definition of the extension and of the
  construction of $\Cmc^3_{\Amc,\Omc}$ and $\Umc_{\Amc,\Omc}$.

  At this point, $g(x)$ is defined for all variables $x$ that occur on
  a cycle in $q$. Assume that $x$ is such a variable. If $x$ is
  an answer variable, then $g(x)$ is clearly already defined. Otherwise
chordality of
$q$ implies that $x$ also occurs on
  a cycle of length at most~3 and thus $g(x)$ has been defined above.
  It remains to define $g(x)$ for variables $x$ that do not occur on
  a cycle. 

  Let $q'$ be the subquery of
  $q$ consisting of all atoms that contain at least one variable $x$ with
  $g(x)$ undefined at this point.
  We argue that
  \begin{enumerate}

  \item $q'$ is a disjoint union of ditrees such that

  \item if $g(x)$ is defined for a variable $x$ in $q'$, then $x$
    is the root of a ditree.

  \end{enumerate}
  First note that
  \begin{itemize}

  \item[($*$)] none of the atoms $r(x_1,x_2)$ in $q'$ occur on a cycle
    in $q$.

  \end{itemize}
  In fact, if an atom $r(x_1,x_2)$ in $q'$ occurs on a cycle in $q$,
  then $g(x_1)$ and $g(x_2)$ are already defined and thus $r(x_1,x_2)$
  is not part of $q'$.
  
  For Point~1, first observe that $q'$ does not contain a cycle. In
  fact, any cycle $C$ in $q'$ is also a cycle in $q$, so by ($*$) $q'$
  does not contain any of the atoms in $C$. To establish Point~1, it
  remains to show that $q'$ contains no atoms $r_1(x_1,y),r_2(x_2,y)$
  with $x_1 \neq x_2$. By definition of
  $q'$, one of $g(x_1)$, $g(y)$ and one of $g(x_2)$, $g(y)$ must be
  undefined. There are two cases:
  \begin{itemize}

  \item $g(y)$ is undefined.

    Then $y$ is a quantified variable and $r_1=r_2$, the latter
    because
    $r_1(h(x_1),h(y)), r_2(h(x_2),h(y)) \subseteq \Cmc^3_{\Amc,\Omc}$,
    $h(y) \notin \mn{ind}(\Amc)$ as $g(y)$ is undefined, and by
    definition of $\Cmc^3_{\Amc,\Omc}$.  Moreover, by ($*$) none of
    $r_1(x_1,y),r_2(x_2,y)$ occurs on a cycle in $q$. Thus, $q$ is not
    strongly symmetry-free, a contradiction.

  \item $g(x_1),g(x_2)$ are undefined.

    Then $h(x_1)$ and $h(x_2)$ are not in $\mn{ind}(\Amc)$.  From
    $r_1(h(x_1),h(y)), r_2(h(x_2),h(y)) \subseteq \Cmc^3_{\Amc,\Omc}$
    and the definition of $\Cmc^3_{\Amc,\Omc}$, we obtain
    $h(y) \notin \mn{ind}(\Amc)$. We can now argue as in the previous
    case that $y$ is a quantified variable and $r_1=r_2$, and again
    obtain a contradiction to $q$ being strongly symmetry-free.

  \end{itemize}
  Now for Point~2. It suffices to observe that if $r(x,y)$ is an atom in
  $q'$, then $g(y)$ is undefined. Assume to the contrary that $g(y)$
  is already defined. By choice of $q'$, it follows that $g(x)$ is
  undefined. As $g(y)$ is defined, one of the following applies:
  \begin{itemize}

  \item $h(y) \in \mn{ind}(\Amc)$.

    Then $r(h(x),h(y)) \in \Cmc^3_{\Amc,\Omc}$ and the definition
    of $\Cmc^3_{\Amc,\Omc}$ imply that $h(x) \in \mn{ind}(\Amc)$,
    in contradiction to $g(x)$ being undefined. 

  \item there is an atom $s(x',y)$ in $q$ with $h(x') \in
    \mn{ind}(\Amc)$, and $h(y) \notin \mn{ind}(\Amc)$.

    From $h(x') \in \mn{ind}(\Amc)$, $h(y) \notin \mn{ind}(\Amc)$,
    $s(h(x'),h(y)),r(h(x),h(y)) \in \Cmc^3_{\Amc,\Omc}$ and the
    definition of $\Cmc^3_{\Amc,\Omc}$, we obtain $h(x')=h(x)$.
    But then $h(x) \in \mn{ind}(\Amc)$, in contradiction to $g(x)$
    not yet being defined.

  \end{itemize}
  We next traverse the ditrees in $q'$ in a top-down fashion to extend
  $g$.  The initial piece of $g$ constructed so far is such that for
  all variables $x$, $h(x)=c_{a,i,r,C}$ implies that $g(x)$ is of the
  form $trC$. We shall maintain this invariant during the extension of
  $g$.

  To extend $g$, repeatedly and exhaustively choose atoms
  $r(x,y) \in q'$ with $g(x)$ defined and $g(y)$ undefined. Then
  $h(y) \notin \mn{ind}(\Amc)$ and thus $h(y)$ has the form
  $c_{a,i,r,C}$. Define $g(y)$ to be $g(x)r C$. If
  $h(x) \in \mn{ind}(\Amc)$, then it is immediate by definition of
  $\Cmc^3_{\Amc,\Omc}$ and $\Umc_{\Amc,\Omc}$ that
  $r(h(x),h(y)) \in \Cmc^3_{\Amc,\Omc}$ implies
  $r(g(x),g(y)) \in \Umc_{\Amc,\Omc}$.  If
  $h(x) \notin \mn{ind}(\Amc)$, we need to additionally invoke
  Lemma~\ref{lem:cthreesimlem}, applied to $h(x)=c_{a',i',r',C'}$ and
  to $g(x)=tr'C'$. 

  By construction, $g$ satisfies all binary atoms in $q'$ and thus in
  $q$. All unary atoms are satisfied, too, because of the invariant
  mentioned above and by definition of $\Cmc^3_{\Amc,\Omc}$ and
  $\Umc_{\Amc,\Omc}$.
\end{proof}

\section{Proof of Theorem~\ref{thm:onlyonequerytype}}

\thmonlyonequerytype*
\begin{proof}
    For Point~1, we use a proof strategy that is inspired by basic lower bound
    proofs for abstract learning problems due to Angluin
    \cite{DBLP:journals/ml/Angluin87}. Here, it is convenient to view the
    oracle as an adversary who maintains a set $S$ of candidate target
    concepts that the learner cannot distinguish based on the queries made so
    far. In our case, $S$ of AQ$^\wedge$-queries. We have to choose $S$ and the
ontology $\Omc$ carefully so that each
    membership query removes only few candidate targets and after a
    polynomial number of queries there is still more than one candidate that the
    learner cannot distinguish.

    For each $n \geq 1$, let   
    \begin{align*}
        \Omc_n &= \{ A_i \sqcap A_i' \sqsubseteq A_1 \sqcap A_1' \sqcap \cdots \sqcap A_n \sqcap A_n' \mid 1 \leq i \leq n \}\\
        \intertext{and}
        S_n &= \{ q(x) \gets \alpha_1(x) \land \cdots \land \alpha_n(x) \mid  \\
            & \quad \quad \quad \alpha_i \in \{ A_i, A_i'\}\ \text{for all}\ i\ \text{with}\ 1 \leq i \leq n \}.
    \end{align*}
    The set $S_n$ contains $2^n$ queries.

    Assume to the contrary of what is to be shown that AQ$^\wedge$-queries are
    polynomial query learnable under conjunctive ontologies. Then there exists
    a learning algorithm and polynomial $p$ such that the number of membership
    queries is bounded by $p(n_1, n_2)$, where $n_1$ is the size of the
    target query $q_T$ and $n_2$ is the size of the conjunctive ontology.  We
    choose $n$ such that $2^n > p(r_1(n), r_2(n))$, where $r_1$ is a polynomial
    such that every query $q(x) \in S_m$ satisfies $||q(x)|| = r_1(m)$ and
    $r_2$ is a polynomial such that $r_2(m) > ||\Omc_m||$ for every $m \geq 1$.

    Now, consider a membership query posed by the learning algorithm with ABox
    and answer variable $(\Amc, a)$. The oracle responds as follows:
    \begin{enumerate}
        \item if $\Amc, \Omc_n \models q(a)$ for no $q(x) \in S_n$, then answer \emph{no}
        \item if $\Amc, \Omc_n \models q(a)$ for a single $q(x) \in S_n$, then answer \emph{no} and remove $q(x)$ from $S_n$
        \item if $\Amc, \Omc_n \models q(a)$ for more than one $q(x) \in S_n$, then answer \emph{yes}.
    \end{enumerate}
    Note that the third response is consistent since $\Amc$ must then contain
    $A_i(a)$ and $A_i'(a)$ for some $i$ and thus $\Omc_n$ implies that $a$ is
    an answer to every query in $S_n$. Moreover the answers are always correct
    with respect to the updated set $S_n$. Thus the learner cannot distinguish
    the remaining candidate queries by answers to queries posed to far.

    It follows that the learning algorithm removes at most $p(r_1(n),
    r_2(n))$ many
    queries from $S_n$. By the choice of $n$, at least two candidate concepts
    remain in $S_n$ after the algorithm is finished. Thus the learner cannot
    distinguish between them and we have derived a contradiction.

    \smallskip

  For Point~2, we exploit a classic connection between active
  learning with equivalence queries and certain separability problems.
  We start by recalling the latter.  A \emph{labeled KB} takes the
  form $\Kmc=(\Omc,\Amc,P,N)$ with \Omc an ontology, \Amc an ABox, and
  $P, N \subseteq \mn{ind}(\Amc)^n$ sets of positive and negative
  examples, respectively, all of them tuples of the same length $n$. A
  query $q(\bar x)$ of arity $n$ \emph{separates \Kmc} if
    \begin{enumerate}
        \item $\Amc,\Omc \models q(\bar a)$ for all $\bar a\in P$, and
        \item $\Amc,\Omc \not\models q(\bar a)$ for all $\bar a \in N$.
    \end{enumerate}
    Every choice of ontology language \Lmc and query language $\Qmc$
    gives rise to an \emph{$\Lmc,\Qmc$-separability problem} which is
    to decide, given a labeled KB $\Kmc=(\Omc,\Amc,P,N)$ with \Omc
    formulated in \Lmc, whether there is a query $q(\bar x) \in \Qmc$
    which separates $P$ and $N$. We are going to concentrate on the
    case where the ontology is empty, which we simply refer to as
    \emph{\Qmc-separability}. For simplicity, we then drop the
    ontology from labeled KBs. It was shown in
    \cite{aaaithis,mauricemaster} that ELQ-separability is
    \NPclass-hard.  An analysis of the proof reveals a class \Cmf of
    labeled KBs $(\Amc,P,N)$ for which ELQ-separability is
    \NPclass-hard and a polynomial $t$ such that the following
    conditions are satisfied:
    %
    \begin{enumerate}


    \item if \Kmc is ELQ-separable, then there is a 
            separating ELQ-query of size $t(n)$, where $n=||\Kmc||$; 

      \item \Amc is a disjoint union of ditrees of depth 1 and the
        elements
        of $P$ and $N$ are the roots of these ditrees;

      \item only a single role name $r$ is used.
            
      \end{enumerate}
      %
    Condition~2
    implies the following.

    \smallskip
\noindent\textit{Claim 1.}   Given an ABox \Amc that satisfies the properties
    given in Condition~2, an $a \in \mn{ind}(\Amc)$, and a unary CQ
    $q(x)$, it can be decided in polynomal time whether
    $\Amc \models q(a)$.
    \smallskip

    We only sketch the proof. To check whether there is a homomorphism
    $h$ from $q(x)$ to \Amc with $h(x)=a$, we treat each maximal
    connected component of $q$ separately. For the component that
    contains $x$, we start with setting $h(x)=a$. We then repeatedly
    extend $h$ to variables $y$ such that $q$ contains some atom
    $r(y,z)$ or $r(z,y)$ with $h(z)$ already defined. If there are
    atoms of both forms, then $\Amc \not\models q(a)$. If there is an
    atom $r(y,z)$ and $h(z)$ is a non-root in \Amc, then $h(y)$ is the
    unique predecessor of $h(z)$ in \Amc; and if $h(z)$ is a root,
    then $\Amc \not\models q(a)$. If there is an atom $r(z,y)$, then
    $\Amc \not\models q(a)$ if $h(z)$ is a non-root. Otherwise, we
    consider all atoms $A(z)$ in $q$ and set $h(y)$ to some successor
    $c$ of $h(z)$ in \Amc such that $A(c) \in \Amc$ for all these
    atoms; if there is no such successor, then again
    $\Amc \not\models q(a)$. For components that do not contain
    $x$ we start with an arbitrarily chosen variable, iterate over all
    individuals in \Amc as targets, and for each target proceed as
    described above.
    
    \smallskip
    
    Now 
    assume that ELQ-queries are polynomial
    time learnable using only CQ-equivalence queries. Then there exists a
    learning algorithm $L$ for ELQ-queries and a polynomial $p$ such that at
    any time, the running time of $L$ so far is bounded by $p(n_1, n_2)$, where
    $n_1$ is the size of the target query $q_T$ and $n_2$ is the size of the largest
    counterexample seen so far. We show how to use $L$ to 
    construct an algorithm $L'$ that decides ELQ-separability for the
    class of labeled KBs \Cmf in
    polynomial time.

    The new algorithm $L'$ takes as input a labeled KB
    $\Kmc=(\Amc,P,N) \in \Cmf$. Let $n = ||\Kmc||$.  $L'$ then runs $L$ for at most
    $p(t(n), n)$ steps.

    Whenever $L$ asks a CQ-equivalence query with $q_H(x)$ as the
    hypothesis, $L'$ answers it by testing whether $q_H$
    separates~\Kmc.  More precisely, $L'$ checks whether
    $\Amc \models q_H(a)$ for all $a \in P$ and
    $\Amc \not \models q_H(a)$ for all $a \in N$. Note that, by the
    above claim, this is possible in time polynomial in $||\Kmc||$.
    If a check fails for some $a \in P$, then $L'$ answers the
    equivalence query by giving $(\Amc, a)$ as a positive
    counterexample to $L$.  If a check fails for some $a \in N$, then
    $L'$ answers the equivalence query by giving $(\Amc, a)$ as a
    negative counterexample to $L$. If all checks succeed, then $L'$
    terminates and returns ``separable''.
    
    If $L$ terminates returning a learned ELQ-query
    $q_H$, then $L$ tests whether $q_H$ separates~$\Kmc$.  If this is the
    case, then $L'$ returns ``separable''.  If $L'$ does not
    terminate within $p(t(n), n)$ steps or returns a query that does
    not separate $\Kmc$, then $L'$ return ``not separable''.  The
    following claim shows correctness of $L'$.

    \smallskip
\noindent\textit{Claim 2.} $L'$ returns ``separable'' iff $\Kmc$ is ELQ-separable.

    \smallskip
\noindent\textit{Proof of Claim 2.}
The ``only if'' direction follows directly from 
the fact that $L'$ only returns ``separable'' if there is a separating
CQ.  For the ``if'' direction, assume that $\Kmc$ is not
ELQ-separable. By Condition~1 above, there is an ELQ-query
$q_T$ of size $t(n)$ that separates $\Kmc$. Note that 
$q_T$ is consistent with the
counterexamples that $L'$ provides to $L$. Since by assumption $L$ is
able to learn any \Qmc-query of size $t(n)$ with counterexamples of
size $n$ in $p(t(n), n)$ steps, it must within this number of steps
either ask an equivalence query with a hypothesis that separates $\Kmc$ (but may not be equivalent to $q_T$), or return a \Qmc-query
that is equivalent to $q_T$.  In both cases $L'$ returns
``separable''.
\end{proof}


\section{Proofs for Section~\ref{sect:nf}}

It is well-known that every \ELdr-ontology \Omc can be converted into an
\ELdr-ontology $\Omc'$ in normal form by introducing additional concept
names~\cite{DL-Textbook}.  For the reduction, it is convenient to use
a suitable form of conversion. An \ELdr-ontology $\Omc_2$ is a
\emph{conservative extension} of an \ELdr-ontology $\Omc_1$ if
$\mn{sig}(\Omc_1) \subseteq \mn{sig}(\Omc_2)$, every model of $\Omc_2$
is a model of $\Omc_1$, and for every model $\Imc_1$ of $\Omc_1$,
there exists a model $\Imc_2$ of $\Omc_2$ such that
$S^{\Imc_1}=S^{\Imc_2}$ for all symbols $S \notin \mn{sig}(\Omc_2)
\setminus \mn{sig}(\Omc_1)$.
\begin{lemma}
    \label{lem:tbox-normalform}
    Given an \ELdr-ontology \Omc, one can compute in polynomial time
    an \ELdr-ontology $\Omc'$ in normal form such that:
    \begin{enumerate}

      \item $\Omc'$ is a conservative extension of $\Omc$,

      \item $\mn{sig}(\Omc') = \mn{sig}(\Omc) \cup \{ X_C \mid C \in \mn{sub}(\Omc) \}$,

      \item $\Omc' \models X_C \equiv C$ for each
        $C \in \mn{sub}(\Omc)$.
        

    \end{enumerate}
\end{lemma}
\begin{proof}\
  Introduce a fresh concept name $X_C$ for every
  $C\in\mn{sub}(\Omc)$ and define $\Omc'$ to contain, for every $C\in
  \mn{sub}(\Omc)$, the following concept inclusions and range restrictions:
  \begin{itemize}

    \item $X_C \sqsubseteq C, C \sqsubseteq X_C$ if $C$ is a concept name or
      $\top$; 

    \item $X_C \sqsubseteq X_{D_1}, X_C \sqsubseteq X_{D_2}$, and
      $X_{D_1} \sqcap X_{D_2} \sqsubseteq X_C$ if $C = D_1 \sqcap
      D_2$; 

  \item $X_C \sqsubseteq \exists r.X_D$ and $\exists r. X_D \sqsubseteq X_C$
    if $C = \exists r.D$;

  \item $X_C \sqsubseteq X_D$ for each concept inclusion
    $C \sqsubseteq D \in \Omc$;

  \item $\exists r^-.\top \sqsubseteq X_C$ for each range restriction
      $\exists r^-.\top \sqsubseteq C \in \Omc$.

  \end{itemize}
  $\Omc'$ is in normal
  form and can be computed in polynomial time. Moreover, it can be
  verified that Points~1--3 hold.
%
%
\end{proof}

A CQ $q'$ can be \emph{obtained from a CQ q by attaching ditrees} if
$q'$ can be constructed by choosing variables $x_1, \ldots, x_n$ from
$q$ and Boolean ditree CQs $q_1, \ldots q_n$ whose sets of variables
are pairwise disjoint and disjoint from the set of variables in $q$,
and then taking the union of $q$ and $q_1, \ldots q_n$, identifying
the root of $q_i$ with $x_i$ for $1 \leq 1 \leq n$. A class of CQs
\Qmc is \emph{closed under attaching ditrees} if every CQ $q'$ that
can be obtained from a $q \in \Qmc$ by attaching ditrees is also in
\Qmc. Note that all of CQ$^{\text{csf}}$, ELQ, and ELIQ$^{\text{sf}}$
are closed under attaching ditrees. We prove the following
generalization of Proposition~\ref{prop:nf}.
\begin{proposition}
Let \Qmc be a class of CQs closed under attaching ditrees. If queries in \Qmc
are polynomial time learnable under \ELdr-ontologies in normal form using
membership and equivalence queries, then the same is true for unrestricted \ELdr-ontologies.
\end{proposition}
\begin{proof}
Let $L'$ be a polynomial time learning algorithm for \Qmc under ontologies in
normal form. We show how $L'$ can be modified into an algorithm $L$ that is
able to learn \Qmc under unrestricted ontologies in polynomial time.

Given an \ELdr-ontology \Omc and a finite $\Sigma \subseteq \mn{N_C} \cup \mn{N_R}$
such that $\mn{sig}(\Omc) \subseteq \Sigma$ and $\mn{sig}(q_T) \subseteq
\Sigma$, algorithm $L$ first computes the ontology $\Omc'$ in normal form as
per Lemma~\ref{lem:tbox-normalform}, choosing the fresh concept names $X_C$ so
that they are not from $\Sigma$. It then runs $L'$ on $\Omc'$ and $\Sigma' =
\Sigma \cup \{ X_C \mid C \in \mn{sub}(\Omc)\}$; note that
$\mn{sig}(\Omc') \subseteq \Sigma'$ as required. In contrast to the
learning algorithm, the oracle still works with the original ontology $\Omc$.
To bridge this gap, algorithm $L$ adopts modifications
during the run of $L'$, as follows.

First, whenever $L'$ asks a membership query $\Amc', \Omc' \models q_T(\bar
a)$, $L$ instead asks the membership query $\Amc, \Omc \models q_T(\bar a)$,
where \Amc is obtained from $\Amc'$ by replacing each assertion $X_C(x)$ with
the $C$ viewed as an ABox, identifying the root with $x$.

By the following claim, the answer to the modified membership query coincides
with that to the original
query.

\smallskip \noindent\textit{Claim 1.} 
$\Amc',\Omc' \models q(\bar a)$ iff $\Amc, \Omc \models q(\bar
a)$ for all 
CQs
$q$ that use only symbols from $\Sigma$.

\smallskip \noindent\textit{Proof of Claim 1.}
For ``if'', suppose that $\Amc, \Omc \models q(\bar a)$ and let $\Imc$ be a
model of $\Amc'$ and $\Omc'$. Then $\Imc$ is a model of $\Omc$ since $\Omc'$
is a conservative extension of $\Omc$. By Property~3 of
Lemma~\ref{lem:tbox-normalform}, $\Imc$ is a model of $\Amc$. Hence $\Imc
\models q(\bar a)$ as required.
For ``only if'', suppose that $\Amc', \Omc' \models q(\bar a)$ and let $\Imc$
be a model of $\Amc$ and $\Omc$. Since $\Omc'$ is a conservative extension of
$\Omc$, there is a model $\Imc'$ of $\Omc'$ that coincides with $\Imc$ on all
symbols from $\Sigma$. By Property~3 of Lemma~\ref{lem:tbox-normalform}, $\Imc$ is a model of $\Amc'$. Since $\mn{sig}(q) \subseteq \Sigma$ and $\Imc'$ and $\Imc$ coincide on $\Sigma$, it follows that $\Imc \models q(\bar a)$ as required.

\smallskip
Second, whenever $L'$ asks an equivalence query $q_H' \equiv_{\Omc'} q_T$, $L$
instead asks the equivalence query $q_H \equiv_\Omc q_T$, where $q_H$ is
obtained from $q_H'$ by replacing each assertion $X_C(x)$ with the Boolean
ditree CQ obtained from ELQ $C$ by quantifying the root, identifying the root with
$x$. Since \Qmc is closed under attaching ditrees, $q_H$ can be used in an
equivalence query.  Furthermore, when the counterexample returned is
$\Amc$, the algorithm replaces it with the restriction $\Amc|_\Sigma$
to signature $\Sigma$ before passing it on to $L'$.

Applying the following claim to both $q'_H$ and $q'_T=q_T$, the answer
to the modified equivalence query coincides with that to the original query.

\smallskip \noindent\textit{Claim 2.}
Let $q'$ be a CQ that uses only symbols from $\Sigma'$ and let
$q$ be
obtained from $q'$ by replacing each assertion $X_C(x)$ with the Boolean
ditree CQ obtained from ELQ $C$ by quantifying the root, the root identified with
$x$. Then
$\Amc|_\Sigma, \Omc' \models q'_H(\bar a)$ iff $\Amc, \Omc \models q_H(\bar
a)$ for all ABoxes $\Amc$. 

\smallskip \noindent\textit{Proof of Claim 2.}
For ``if'', suppose $\Amc, \Omc \models q(\bar a)$ and let $\Imc$ be a model
of $\Amc|_\Sigma$ and $\Omc'$.
Since $q$ and $\Omc$ contain only symbols from $\Sigma$, $\Amc|_\Sigma, \Omc \models q(\bar a)$.
Since $\Omc'$ is a conservative extension of $\Omc$, $\Imc$ is also a model of
$\Omc$. Thus $\Imc \models q(\bar a)$ and by Property~3 of
Lemma~\ref{lem:tbox-normalform}, $\Imc \models q'(\bar a)$ follows as
required.

For ``only if'', suppose $\Amc|_\Sigma, \Omc' \models q'(\bar a)$ and let
$\Imc$ be a model of $\Amc$ and $\Omc$.  Since $\Omc$ contains only symbols
from $\Sigma$, $\Imc|_\Sigma$ is a model of $\Amc|_\Sigma$ and since $\Omc'$ is
a conservative extension of \Omc, there is a model $\Imc'$ of $\Omc'$ that
coincides on all symbols from $\Sigma$ with $\Imc|_\Sigma$. Thus $\Imc' \models
q'(\bar a)$ and by Property~3 of Lemma~\ref{lem:tbox-normalform},
$\Imc|_\Sigma \models q(\bar a)$. Then $\Imc \models q(\bar
a)$ since $q$ uses only symbols from $\Sigma$, as required.
\end{proof}

\section{Proofs for Section~\ref{sec:short-cycles}}

We analyze central properties of the \mn{refine} subroutine.  Recall
that we first construct a sequence
$(\Bmc_1,\bar b_1), (\Bmc_2,\bar b_2),\dots$ using the Expand and
Minimize steps. With $\Bmc_i'$, $i \geq 1$, we denote the result of
only applying the Expand step to $\Bmc_i$, but not the Minimize
step. Also recall that the fresh individuals introduced in $\Bmc_i'$
are denoted with $a'$ in case that the original individual was
$a$. The following can easily be shown.
\begin{lemma} \label{lem:auxunfolding}
  Let  $i \geq 1$. Then $\Bmc_i', a \preceq \Bmc_i, a$ and $\Bmc_i, a
  \preceq \Bmc_i', a$, for all $a \in \mn{ind}(\Bmc_i)$, and $\Bmc_i'
  , a' \preceq \Bmc_i, a$ and $\Bmc_i , a \preceq \Bmc_i', a'$, for
  all fresh individuals $a'$.
\end{lemma}

We start with proving properties of  the Expand/Minimize phase. 

\begin{lemma} \label{lem:expansioncorrect}
  For all $i \geq 1$, the following properties hold:
  \begin{enumerate}

    \item $\Bmc_i,\Omc\models q_T(\bar b_i)$;

    \item $\Bmc'_i,\Omc\models q_T(\bar b)$ for some
      $\bar b \in \tau_i$;

  \end{enumerate}
\end{lemma}

\begin{proof}
  We prove both points simultaneously by induction on~$i$. For
  Point~1, the case $i=1$ is immediate since
  $(\Bmc_1,\bar b_1)= \mn{minimize}(\Amc,\bar a)$ and
  $\Amc,\Omc\models q_T(\bar a)$, and the case $i>1$ is an immediate
  consequence of the inductive hypothesis (Point~2), the choice of
  $\bar b_i$, and the definition of the Minimize step.

  For Point~2, the induction start and step are identical.
  Thus let $i \geq 1$. Assume that $\Bmc'_i$ was obtained from $\Bmc_i$ by expanding cycle
  $R_0(a_0,a_1),\ldots,R_{n-1}(a_{n-1},a_{n})$. 
  By Point~1, there is a homomorphism $h$ from $q_T$ to
  $\Umc_{\Bmc_i,\Omc}$ with $h(\bar x)=\bar b_i$.  We construct a
  homomorphism $g$ from $q_T$ to $\Umc_{\Bmc'_i,\Omc}$ with
  $g(\bar x)=\bar b$ for some $\bar b\in \tau_i$, which yields
  $\Bmc'_i,\Omc\models q_T(\bar b)$ as desired.  Let us partition
  $\mn{var}(q_T)$ into sets $M_0,M_1,M_2$ such that:
  \begin{itemize}

    \item $x\in M_0$ if $h(x)\in \{a_0,\ldots,a_{n-1}\}$, that is,
      $h(x)$ lies on the expanded cycle;

    \item $x\in M_1$ if $h(x)\notin\mn{ind}(\Bmc_i)$, that is, $h(x)$ is in the part of
      $\Umc_{\Bmc_i,\Omc}$ generated by existential quantification;

    \item all other variables are in $M_2$.

  \end{itemize}
  We start with setting 
  \[g(x)=h(x)\text{\quad for all }x\in M_2.\]
  To define $g(x)$ for the variables in $x \in M_0$, we first
  construct an auxiliary query $q'_T$ of treewidth 1. If
  $\Qmc \in \{ \text{ELQ}, \text{ELIQ}^{\text{sf}} \}$, then $q'_T$ is
  simply $q_T$. Now assume that $\Qmc=\text{CQ}^{\text{csf}}$.
Then $q'_T$ is obtained by
  starting with the restriction of $q_T$ to the variables in $M_0$ and
  then exhaustively choosing and identifying variables $x_1,x_2$ such that
  \begin{enumerate}

  \item there is a cycle $R_0(y_0,y_1),R_1(y_1,y_2),
    R_3(y_2,y_0)$ with $\{ x_1,x_2 \} \subseteq \{ y_0,y_1,y_2
    \}\subseteq M_0$ and

  \item $h(x_1)=h(x_2)$

  \end{enumerate}
  The result of identifying an answer variable and a quantified
  variable is an answer variable.  Note that we may also identify two
  answer variables.
  
  We next observe that since $q_T$ is chordal, all CQs
  $q_T=p_0,p_1,\dots,p_k=q'_T$ encountered during the construction of
  $q'_T$ are chordal as well. This can be shown by induction on the
  index $i$ to the CQs $p_i$.  The induction start is clear since
  $q_T$ is chordal. For the induction step, assume 
  that $p_i$ contains a cycle
  $C=S_0(z_0,z_1),\ldots,S_{n-1}(z_{n-1},z_n)$ of length at least four
  with at least one quantified variable, where $i >
  0$. 
  Then $p_{i-1}$ contains $C$ or a cycle $C'$ that can be obtained
  from $C$ by replacing some edge $S_i(z_i,z_{i+1})$ with two edges
  $S_{i,1}(z_i,u),S_{i,2}(u,z_{i+1})$ (because $u$ and $z_{i+1}$ were
  identified when constructing $p_i$). In the first case, $C$ has a
  chord in $p_{i-1}$ and thus also in $p_i$.  In the second case, $C'$
  contains at least one quantified variable since $C$ does and
  consequently has a chord in $p_{i-1}$. If this chord is not between
  $z_i$ and $z_{i+1}$, then $C$ contains a chord in $p$. If the
  chord is between $z_i$ and $z_{i+1}$, then we are in the first case.
  
  We now show that $q'_T$ has treewidth~1, that is, it takes the form
  of a disjoint union of (not necessarily directed) trees with
  multi-edges and self loops.  Assume to the contrary that $q'_T$
  contains a cycle $C$ of length exceeding~2. If there is a quantified
  variable $x$ on $C$, then $q'_T$ being chordal implies that $x$
  occurs on a cycle of length~3, in contradiction to the construction
  of $q'_T$. Now assume that there is no quantified variable on
  $C$. As the image of $C$ under $h$ is a cycle in $\Bmc_i$ and the
  cycle chosen by the Expansion step is chordless, the image
  of $C$ under $h$ must contain \emph{all} individuals
  $\{a_0,\dots,a_{n-1}\}$. Since all variables on $C$ are answer
  variables, this means that all individual in $M_0$ are from
  $\bar b$, in contradiction to
  $\{a_1,\ldots,a_{n}\}\not\subseteq \bar b_i$.

  This finishes the construction of $q'_T$.  For defining $g(x)$ for
  the variables $x \in \mn{var}(q'_T)$, we can now start at some
  arbitrary variable in each tree in $q'_T$ and then follow the tree
  structure, switching between the individuals $a_0,\dots,a_{n-1}$ and
  their copies $a'_0,\dots,a'_{n-1}$ as necessary.  We next make this
  precise. 
  For each connected component of $q'_T$, choose an arbitrary
  variable $z$ from that component and set
  \[g(z)=h(z).\] 
  Then exhaustively apply the following rule: 
if $q'_T$ contains an atom $R(x,y)$ with $g(y)$ defined and 
      $g(x)$ undefined, set

      \begin{itemize}

	\item $g(x)=h(x)$ if $g(y)=a_i$ and either $h(x)=a_{i+1}$
	  and $i<n-2$ or  $h(x)=a_{i-1}$ and $i>0$;

	\item $g(x)=h(x)'$ if $g(y)=a_i'$ and either $h(x)=a_{i+1}$
	  and $i<n-2$ or  $h(x)=a_{i-1}$ and $i>0$;

	\item $g(x)=a_0'$ if $g(y)=a_{n-2}$ and
	  $h(x)=a_{n-1}=a_0$;

	\item $g(x)=a_0$ if $g(y)=a_{n-2}'$ and
	  $h(x)=a_{n-1}=a_0$;

	\item $g(x)=a_{n-2}'$ if $g(y)=a_0$ and
	  $h(x)=a_{n-2}$;

	\item $g(x)=a_{n-2}$ if $g(y)=a_0'$ and
	  $h(x)=a_{n-2}$.

      \end{itemize}
      It can be verified that in all cases, $R(g(x),g(y)) \in \Bmc'_i$
      by construction of $\Bmc'_i$ in the unravelling step.
      
      We can next extend $g$ to all variables in $M_0$ by setting
      $g(y)=g(x)$ if $y$ was identified with $x \in \mn{var}(q'_T)$
      during the construction of $q'_T$ (note that this implies
      $h(y)=h(x)$).

      \smallskip
      
      By definition, $g(x)\in \{h(x),h(x)'\}$ for all $x\in
      M_0$. Thus, $g(\bar x)\in \tau_i$ as announced.

  \medskip It remains to define $g(x)$ for the variables $x\in
  M_1$. By definition of $M_1$, $h(x)$ is a trace $cw$ with
$c\in\mn{ind}(\Bmc_i)$ and $w\neq\varepsilon$, that is,
  $x$ is mapped into the subtree below $c$ in $\Umc_{\Bmc_i,\Omc}$. 
  Now do the following:
  \begin{itemize}

  \item if there is a path in $q_T$ from some variable $z \in M_0$ to
    $x$, then choose such a $z$ such that the path is shortest (thus,
    $h(z)=c$ and $g(z)$ has already been defined) and set $g(x)=g(z)w$;

  \item otherwise, set $g(x)=h(x)$.

  \end{itemize}
%
%
%
%
  %
  This is well-defined since, due to
  Lemma~\ref{lem:auxunfolding}, the following holds:
  \begin{enumerate}

  \item for each $c \in \mn{ind}(\Bmc_i)$, the subtrees below $c$
    in $\Umc_{\Bmc_i,\Omc}$ and in $\Umc_{\Bmc'_i,\Omc}$ are
    identical;

  \item for $0 \leq j < n$, the subtree below $a_j$ in
    $\Umc_{\Bmc_i,\Omc}$
    and the subtree below $a'_j$ in $\Umc_{\Bmc'_i,\Omc}$ are
    identical.

    \end{enumerate}
    Set $\bar b=g(\bar x)$. To prove that $\Bmc'_i,\Omc \models q_T(\bar b)$, it remains to
    show the following.

  \smallskip\noindent\textit{Claim.} $g$ is a homomorphism from $q_T$ to
  $\Umc_{\Bmc'_i,\Omc}$.

  \smallskip\noindent\textit{Proof of the claim.}  Let $A(x)$ be a
  concept atom in $q_T$. Then $A(h(x))\in \Umc_{\Bmc_i,\Omc}$. If
  $h(x) \in \mn{ind}(\Bmc_i)$, then $g(x)$ was defined such that, by
  Lemma~\ref{lem:auxunfolding}, $\Bmc_i,h(x) \preceq
  \Bmc_i',g(x)$.
  By Lemma~\ref{lem:universal-homomorph},  $\Umc_{\Bmc_i,\Omc},h(x) \preceq
  \Umc_{\Bmc'_i,\Omc},g(x)$ and thus by Lemma~\ref{lem:simfund} $A(g(x)) \in
  \Umc_{\Bmc'_i,\Omc}$. If $h(x)\notin\mn{ind}(\Bmc_i)$ then the
  remark before the claim and the definition of $g(x)$ ensures that
  $A(g(x))\in \Umc_{\Bmc_i',\Omc}$.

  Now let $R(x_1,x_2)$ be a role atom in $q_T$. We distinguish cases
  according to $x_1,x_2$ belonging to $M_0,M_1,M_2$:
  \begin{itemize}

  \item If $x_1,x_2\in M_0$, then $q'_T$ contains an atom
    $R(x'_1,x'_2)$ such that each $x_i$ was identified with $x'_i$
    during the construction of $q'_T$. If $x'_1 \neq x'_2$, then
    $R(g(x'_1),g(x'_2)) \in \Bmc'_i$, as argued in the definition of
    $g$ for variables from $q'_T$. By that definition and the construction
    of $\Bmc'_i$, the same is true when $x_1=x_2$. We have
    $g(x_i)=g(x'_i)$ for $i \in \{1,2\}$. Thus $R(g(x'_1),g(x'_2)) \in
    \Umc_{\Bmc'_i,\Omc}$ as required.

  \item If $x_1,x_2\in M_1$, then $h(x_1)=bv$ and $h(x_2)=bw$ for some
    $b\in\mn{ind}(\Bmc_i)$ and some non-empty $v,w$, and $R(h(x_1),h(x_2))\in \Umc_{\Bmc_i,\Omc}$. By definition of
    $g$, we have $g(x_1)=\widehat bv$ and $g(x_2)=\widehat bw$ for some
    $\widehat b\in\{b,b'\}$. By Lemma~\ref{lem:auxunfolding}, the
    subtree below $\widehat b$ in $\Umc_{\Bmc_i',\Omc}$ is identical
    to the
    subtree below $b$ in $\Umc_{\Bmc_i,\Omc}$. 
    This implies
    $R(g(x_1),g(x_2))\in \Umc_{\Bmc_i',\Omc}$.

    \item If $x_1,x_2 \in M_2$, then $g(x_1)=h(x_1)$,
      $g(x_2)=h(x_2)$, and $R(h(x_1),h(x_2))\in \Bmc_i$ because $h$
      is a homomorphism from $q_T$ to $\Umc_{\Bmc_i,\Omc}$
      and $h(x_1),h(x_2)\in\mn{ind}(\Bmc_i)$. Since additionally
      $h(x_1),h(x_2)\notin\{a_0,\ldots,a_{n-1}\}$, 
      $R(h(x_1),h(x_2))\in \Bmc'_i$ and thus $R(g(x_1),g(x_2))\in
      \Umc_{\Bmc'_i,\Omc}$.

    \item If $x_1\in M_0$ and $x_2\in M_1$, then
      $h(x_1) \in \{a_0,\ldots,a_{n-1}\}$ and $h(x_2)$ takes the form
      $h(x_1)rC$. Moreover, $g(x_1) \in \{ h(x_1), h(x_1)' \}$ and
      $g(x_2)=g(x_1)rC$. It thus follows from
      Lemma~\ref{lem:auxunfolding}, 
      $R(h(x_1),h(x_2)) \in \Umc_{\Bmc_i,\Omc}$, and the construction
      of universal models that
      $R(g(x_1),g(x_2)) \in \Umc_{\Bmc'_i,\Omc}$.

    \item If $x_1\in M_0$ and $x_2 \in M_2$, then
      $h(x_1) \in \{a_0,\ldots,a_{n-1}\}$ and
      $h(x_2) \in \mn{ind}(\Bmc_i) \setminus
      \{a_0,\ldots,a_{n-1}\}$. Moreover,
      $g(x_1) \in \{ h(x_1), h(x_1)' \}$ and $g(x_2)=h(x_2)$.
      It follows from $R(h(x_1),h(x_2)) \in  \Umc_{\Bmc_i,\Omc}$
      that  $R(h(x_1),h(x_2)) \in \Bmc_i$. By construction of
      $\Bmc'_i$, we thus have $R(g(x_1),g(x_2)) \in \Bmc_i'
      \subseteq \Umc_{\Bmc'_i,\Omc}$.

    \item if $x_1\in M_1$, $x_2\in M_2$, then
      $h(x_2) \in \mn{ind}(\Bmc_i) \setminus \{a_0,\ldots,a_{n-1}\}$
      and $h(x_1)$ takes the form $h(x_2)rC$ with $R=r^-$. Moreover
      $g(x_i)=h(x_i)$ for $i \in \{1,2\}$ and it remains to use
      Lemma~\ref{lem:auxunfolding} as in previous cases.

  \end{itemize}
\end{proof}

\begin{lemma} \label{lem:expansioncorrecttwo}
  For all $i \geq 1$,
  \begin{enumerate}

    \item $\Bmc_i$ is \Omc-saturated;

    \item if $h$ is a homomorphism from $q_T$ to $\Umc_{\Bmc_i,
      \Omc}$ with $h(\bar x) = \bar b_i$, then $\mn{ind}(\Bmc_i)\subseteq
      \mn{img}(h^*)$;

    \item $\Bmc_{i+1}, \bar b_{i + 1} \to \Bmc_i, \bar b_i$;

    \item $|\mn{ind}(\Bmc_{i + 1})| > |\mn{ind}(\Bmc_{i})|$.

  \end{enumerate}
\end{lemma}

\begin{proof}
  We prove Point~1 by induction on $i$. For $i=1$, recall that the
  initial ABox \Amc is of the form
  $\Cmc^3_{\Amc_{q_H},\Omc} \times \Cmc^3_{\Amc,\Omc}$
  or $\Amc_{q^{\bot}}$ for all uses of the subrouting
  \mn{refine} and that $\Bmc_1, \bar b_1 = \mn{minimize}(\Amc, \bar a)$. In the first case,
  both $\Cmc^3_{\Amc_{q_H},\Omc}$ and
  $\Cmc^3_{\Amc,\Omc}$ are \Omc-saturated and thus their product is also
  \Omc-saturated by Lemma~\ref{lem:prodfund} Point~3.
  In the second case $\Amc_{q^{\bot}}$ is \Omc-saturated since it
  contains $A(x_0)$ for all concept names $A \in \Sigma$. Moreover, the
  Minimize step does not remove any concept assertions.
  For the induction step, suppose
  $\Bmc_{i+1},\Omc \models A(\widehat a)$ for some concept name $A$
  and some $\widehat a \in \mn{ind}(\Bmc_{i+1})$ with $\widehat a$
  either $a\in \mn{ind}(\Bmc_i)$ or $a'$ for some $a\in
  \mn{ind}(\Bmc_i)$. By monotonicity, $\Bmc'_{i},\Omc \models
  A(\widehat a)$ where $\Bmc'_i$ is the result of applying the Expand step
  to $\Bmc_i$ before the Minimize step. By Lemma~\ref{lem:auxunfolding}, 
  $\Bmc_i',\widehat a\preceq \Bmc_i,a$ and thus $\Bmc_i,\Omc\models
  A(a)$ by Lemma~\ref{lem:simlem}. By induction, we know that $A(a)\in \Bmc_i$, and the
  application of the rules ensures that $A(\widehat a)\in \Bmc_{i+1}$.

  \smallskip For Point~2, let $h$ be a homomorphism from $q_T$ to
  $\Umc_{\Bmc_i, \Omc}$ with $h(\bar x) = \bar b_i$, and suppose that
  there is an $a \in \mn{ind}(\Bmc_i)$ that is not in $\mn{img}(h^*)$.
  Let $\Bmc'$ be the result of removing from $\Bmc_i$ all assertions
  that involve $a$. We show that
  \begin{itemize}
    \item[($*$)] $h$ is a homomorphism from $q_T$ to $\Umc_{\Bmc',
      \Omc}$
  \end{itemize}
  which witnesses that $\Bmc', \Omc \models q_T(\bar b_i)$. Hence,
  $a$ is dropped during the Minimize step, in contradiction to $a \in
  \mn{ind}(\Bmc_i)$.  To see that ($*$) holds, first note that for
  all $b, b' \in \mn{ind}(\Bmc_i) \setminus \{a\}$, the following
  holds by Point~1 and construction of universal models:
  \begin{enumerate}

    \item[(a)] $A(b) \in \Umc_{\Bmc_i, \Omc}$ iff $A(b) \in \Umc_{\Bmc',
      \Omc}$;

    \item[(b)] $r(b, b') \in \Umc_{\Bmc_i, \Omc}$ iff $r(b, b') \in
      \Umc_{\Bmc', \Omc}$.

  \end{enumerate}
  From (a), in turn, it follows that the subtree in
  $\Umc_{\Bmc_i, \Omc}$ below each
  $b \in \mn{ind}(\Bmc_i) \setminus \{a\}$ is identical to the subtree
  in $\Umc_{\Bmc', \Omc}$ below $b$. Now ($*$) is an easy consequence.

  \smallskip For Points~3 and~4, define a mapping $g$ from
  $\mn{ind}(\Bmc_{i+1})$ to $\mn{ind}(\Bmc_i)$ by taking $g(a) = a$ for all
  $a \in \mn{ind}(\Bmc_i) \cap \mn{ind}(\Bmc_{i + 1})$ and $g(a') = a$ for
  all $a' \in \mn{ind}(\Bmc_{i + 1}) \setminus \mn{ind}(\Bmc_i)$. For
  Point~3, we verify the following Claim.

  \smallskip\noindent\textit{Claim 1.} $g$ is a homomorphism from
  $\Bmc_{i+1}$ to $\Bmc_i$ with $g(\bar b_{i+1})=\bar b_i$.

  \smallskip\noindent\textit{Proof of Claim 1.} 
  If $A(a) \in \Bmc_{i + 1}$, then $A(a) \in \Bmc'_i$ by definition of
  the Minimize step, and Lemma~\ref{lem:auxunfolding}
  implies that $A(g(a)) \in \Bmc_i$, as required. If
  $r(a, b) \in \Bmc_{i + 1}$, then $r(a, b) \in \Bmc'_{i}$.
  The definition of the Expand step then yields
  $r(g(a), g(b)) \in \Bmc_i$, as required.

  \smallskip For Point~4, it suffices to show that $g$ is surjective,
  but not injective.  

  \smallskip\noindent\textit{Claim 2.} $g$ is surjective.

  \smallskip\noindent\textit{Proof of Claim~2.} Suppose that $g$ is not
  surjective. Then $\mn{ind}(\Bmc_i) \not\subseteq \mn{img}(g)$. By
  Lemma~\ref{lem:expansioncorrect} Point~1, there is a homomorphism
  $h_1$ from $q_T$ to $\Umc_{\Bmc_{i + 1},\Omc}$ with $h_1(\bar x)=\bar
  b_{i+1}$. Let $h_2$ be the extension of $g$ to a homomorphism from
  $\Umc_{\Bmc_{i+1},\Omc}$ to $\Umc_{\Bmc_i,\Omc}$ as in
  Lemma~\ref{lem:universal-homomorph} Point~1. Then 
  $\mn{img}(h_2^*) = \mn{img}(g)$.
  Composing $h_1$ and $h_2$ yields a homomorphism $h_3$
  from $q_T$ to $\Umc_{\Bmc_i, \Omc}$ with
  $h_3(\bar x) = \bar b_i$, but with $\mn{ind}(\Bmc_i)
  \not\subseteq \mn{img}(h_3^*)$, in contradiction to Point~2.

  For an injective and surjective function, we use $g^-$ to denote its
  inverse.

  \smallskip\noindent\textit{Claim 3.} If $g$ is injective, then
  $r(a,b) \in \Bmc_i$ implies $r(g^-(a), g^-(b)) \in \Bmc_{i + 1}$.

  \smallskip\noindent\textit{Proof of Claim 3.} Suppose to the
  contrary that there is an $r(a,b) \in \Bmc_i$ with
  $r(g^-(a), g^-(b)) \notin \Bmc_{i + 1}$. Since $g$ is injective, it
  is then also a
  homomorphism from $\Bmc_{i + 1}$ to $\Bmc_i \setminus \{ r(a, b) \}$
  and using composition-of-homomorphisms argument as in the proof of
  Claim~2,
we find a homomorphism
  $h$ from $q_T$ to
  $\Umc_{\Bmc_i \setminus \{ r(a, b) \}, \Omc}$.  Hence
  $r(a, b)$ is dropped during the Minimize step, in
  contradiction to $r(a, b) \in \Bmc_i$.

  \smallskip\noindent\textit{Claim 4.} $g$ is not an injective
  homomorphism.

  \smallskip\noindent\textit{Proof of Claim 4.}  Let
  $R_0(a_0, a_1), \ldots R_{n - 1}(a_{n - 1}, a_n) \in \Bmc_i$ be the
  chordless cycle that is expanded during the construction of
  $\Bmc_{i + 1}$ from $\Bmc_i$. Recall that $a_0 = a_n$. Without loss
  of generality, assume that $R_{n - 1} = r_{n-1}$ is a role name, but
  not an inverse role.  Suppose for contradiction that $g$ is
  injective. The construction of $g$, together with $g$ being
  surjective
  and injective, implies that exactly one of $a_j, a_j'$ is in
  $\mn{ind}(\Bmc_{i + 1})$ for all $j$ with $0 \leq j \leq n$.

  Assume that $a_{n-1} \in \mn{ind}(\Bmc_{i + 1})$ (the case
  $a'_{n-1} \in \mn{ind}(\Bmc_{i + 1})$ is analogous) and thus
  $g(a_{n-1})=a_{n-1}$.  We prove by induction on $i$ that
  $a_i \notin \mn{ind}(\Bmc_{i +1})$ for $0 \leq i < n$, thus
  obtaining a contradiction to $a_{n-1} \in \mn{ind}(\Bmc_{i + 1})$.
  
  For the induction start, assume to the contrary of what is to be
  shown that $a_0 \in \mn{ind}(\Bmc_{i +1})$. Then $g(a_0)=a_0$ and
  $r_{n-1}(a_{n - 1}, a_0) \in \Bmc_i$ implies
  $r_{n-1}(a_{n - 1}, a_0) \in \Bmc_{i + 1}$ by Claim~3, in
  contradiction to the definition of the Expand step. 

  
  For the induction step, let $i \geq 0$. We know that
  $a_{i-1} \notin \mn{ind}(\Bmc_{i + 1})$ and thus
  $a'_{i-1} \in \mn{ind}(\Bmc_{i + 1})$.  Then
  $g(a'_{i-1})=a_{i-1}$. Assume to the contrary of what is to be shown
  that $a_i \in \mn{ind}(\Bmc_{i + 1})$.  Then $g(a_i)=a_i$ and
  $R_{i-1}(a_{i-1},a_i) \in \Bmc_i$ and Claim~3 yield
  $R_{i-1}(a'_{i-1},a_i) \in \Bmc_i$, in contradiction to the
  definition of the Expand step.
\end{proof}
It is proved as part of Lemma~\ref{lem:refine-termination} below that
the Expand/Minimize phase terminates after polynomially many steps,
let $(\Bmc_n,\bar b_n)$ be the result.

We next construct a sequence $\Bmc_{n},\Bmc_{n+1},\dots$ using the
Split and Minimize steps. With $\Bmc_i'$, $i \geq n$, we denote the
result of only applying the Split step to $\Bmc_i$, but not the
Minimize step.
\begin{lemma} \label{lem:splitcorrect}
  $\Bmc_i,\Omc\models 
  q_T(\bar b_n)$ for all $i \geq n$.
\end{lemma}

\begin{proof}
  We show the lemma by induction on $i$. For $i=n$, this is a
  consequence of Point~1 of Lemma~\ref{lem:expansioncorrect}. For
  $i>n$, it is immediate from the induction hypothesis and the facts
  that a split is only taking place if $\Bmc_i',\Omc\models q_T(\bar
  b_n)$, and that the Minimize step preserves this. 
\end{proof}

\begin{lemma} \label{lem:splitcorrecttwo}
  For all $i\geq n$,
  \begin{enumerate}

    \item $\Bmc_i$ is \Omc-saturated;

    \item if $h$ is a homomorphism from $q_T$ to $\Umc_{\Bmc_i,
      \Omc}$ with $h(\bar x) = \bar b_n$, then $\mn{ind}(\Bmc_i)\subseteq
      \mn{img}(h^*)$;

    \item $\Bmc_{i+1}, \bar b_n \to \Bmc_i, \bar b_n$;

    \item $|\mn{ind}(\Bmc_{i + 1})| > |\mn{ind}(\Bmc_{i})|$.

  \end{enumerate}
\end{lemma}

\begin{proof}
  We show Point~1 by induction over $i$. For $i = n$, this follows
  from Lemma~\ref{lem:expansioncorrecttwo} Point~1. For the induction
  step, suppose $\Bmc_{i + 1}, \Omc \models A(\hat a)$ for some
  $\widehat a \in \mn{ind}(\Bmc_{i + 1})$, with $\widehat a$ either $a
  \in \mn{ind}(\Bmc_i)$ or $a'$ for some $a \in \mn{ind}(\Bmc_i)$. Let
  $\Bmc_i'$ be the result of applying the Split step, but not yet the
  Minimize step. Then $\Bmc_i', \Omc \models A(\hat a)$ by
  monotonicity. By Lemma~\ref{lem:auxunfolding} we have $\Bmc_i',
  \widehat a \preceq \Bmc_i, a$ and thus $\Bmc_i, \Omc \models A(a)$ by
  Lemma~\ref{lem:simlem}.
  By the induction hypothesis, we have $A(a) \in \Bmc_i$ and the
  definition of the Split step ensures $A(a) \in \Bmc_{i + 1}$.

  \smallskip For Point~2, let $h$ be a homomorphism from $q_T$ to
  $\Umc_{\Bmc_i, \Omc}$ with $h(\bar x) = \bar b_n$, and suppose that
  there is an $a \in \mn{ind}(\Bmc_i)$ that is not in $\mn{img}(h^*)$.
  Let $\Bmc'$ be the result of removing from $\Bmc_i$ all assertions
  that involve $a$. We show that 
  \begin{itemize}
    \item[($*$)] $h$ is a homomorphism from $q_T$ to $\Umc_{\Bmc',
      \Omc}$
  \end{itemize} 
  which witnesses that $\Bmc', \Omc \models q_T(\bar b_n)$. Hence $a$
  is dropped during the Minimize step, in contradiction to $a \in
  \mn{ind}(\Bmc_i)$. To see that ($*$) holds, first note that for all
  $b, b' \in \mn{ind}(\Bmc_i) \setminus \{a\}$, the following holds by
  Point~1 and construction of universal models: 
  \begin{enumerate} 

    \item $A(b) \in \Umc_{\Bmc_i, \Omc}$ iff $A(b) \in \Umc_{\Bmc',
      \Omc}$; 

    \item $r(b, b') \in \Umc_{\Bmc_i, \Omc}$ iff $r(b, b') \in
      \Umc_{\Bmc', \Omc}$.  

  \end{enumerate} 
  From Point~1, in turn it follows that the subtree in $\Umc_{\Bmc_i,
  \Omc}$ below each $b \in \mn{ind}(\Bmc_i) \setminus \{ a\}$ is
  identical to the subtree in $\Umc_{\Bmc', \Omc}$ below $b$. In
  summary, ($*$) follows.

  \smallskip For Points~3 and~4, recall that $\Bmc_{i + 1}$ is the
  result of applying the Split and Minimize step to $\Bmc_i$.  Let $b
  \in \mn{ind}(\Bmc_i)$ be the individual that is duplicated by the
  Split step and let $b'$ be the fresh individual. We define a mapping
  $h$ from $\mn{ind}(\Bmc_{i + 1})$ to $\mn{ind}(\Bmc_i)$ by taking
  $h(a) = a$ for all $a \in \mn{ind}(\Bmc_i) \cap \mn{ind}(\Bmc_{i
  +1})$ and $h(b') = b$ if $b' \in \mn{ind}(\Bmc_{i + 1})$, that is,
  $b'$ was not removed during minimization. Clearly, we have
  $h(\bar b_n) = b_n$. To establish Point~3, we
  argue that $h$ is a homomorphism. First, let $A(a) \in \Bmc_{i +
  1}$. By construction of $\Bmc_{i + 1}$, we also have $A(h(a)) \in
  \Bmc_i$. Now, let $r(a,c) \in \Bmc_{i + 1}$. By definition of the
  Split step, also $r(h(a), h(c)) \in \Bmc_{i}$. 

  \smallskip
  For Point~4, it suffices to verify that $h$ is
  surjective but not injective.

  \smallskip \noindent \textit{Claim 1.} $h$ is surjective.

  \smallskip \noindent \textit{Proof of Claim 1.} Assume to the
  contrary that there is a $a \in \mn{ind}(\Bmc_i)$ such that $a
  \notin \mn{img}(h)$. By Lemma~\ref{lem:universal-homomorph} Point~1, $h$ can
  be extended to a homomorphism $h_1$ from $\Umc_{\Bmc_{i + 1}, \Omc}$
  to $\Umc_{\Bmc_i, \Omc}$ with $h_1(\bar b_n) = \bar b_n$, such that
  $\mn{img}(h_1^*) = \mn{img}(h)$.  Composing $h_1$ and a homomorphism
  $h_2$ from $q_T$ to $\Umc_{\Bmc_{i + 1}, \Omc}$ with $h_2(\bar x) =
  \bar b_n$ (which exists by Lemma~\ref{lem:splitcorrect}) yields a
  homomorphism $h_3$ from $q_T$ to $\Umc_{\Bmc_i, \Omc}$ with
  $h_3(\bar x) = \bar b_n$ such that $a \notin \mn{img}(h_3^*)$, in
  contradiction to Point~2.

  \smallskip \noindent \textit{Claim 2.} $h$ is not injective.

  \smallskip \noindent \textit{Proof of Claim 2.} Assume to the
  contrary that $h$ is injective. Then at most one of $b \in
  \mn{ind}(\Bmc_{i + 1})$ or $b' \in \mn{ind}(\Bmc_{i + 1})$.  Again,
  $h$ can be extended to a homomorphism $h_1$ from $\Umc_{\Bmc_{i +
  1}, \Omc}$ to $\Umc_{\Bmc_i, \Omc}$ by
  Lemma~\ref{lem:universal-homomorph} Point~1.  Composing $h_1$ and a
  homomorphism $h_2$ from $q_T$ to $\Umc_{\Bmc_{i + 1}, \Omc}$ with
  $h_2(\bar x)=\bar b_n$ (exists by Lemma~\ref{lem:splitcorrect})
  yields a homomorphism $g$ from $q_T$ to $\Umc_{\Bmc_i, \Omc}$ with
  $g(\bar x)=\bar b_n$.  Recall that there is a symmetry $r(a, b),
  r(c, b) \in \Bmc_i$.  If neither $b$ or $b'$ are in
  $\mn{ind}(\Bmc_{i + 1})$, then $b \notin \mn{img}(g^*)$, in
  contradiction to Point~2.  If $b \in \mn{ind}(\Bmc_{i + 1})$, then
  the Minimize step removed $b'$ in the construction of $\Bmc_{i +
  1}$.  It follows that there is a homomorphism $g'$ from $q_T$ to
  $\Umc_{\Bmc_i, \Omc}$ with $g'(\bar x)=\bar b_n$ and such that there
  is no $r(x, y) \in q_T$ with $r(g'(x), g'(y)) = r(c, b)$. This
  contradicts $\Bmc_i$ being the result of the Minimize step.  The
  case for $b' \in \mn{ind}(\Bmc_{i + 1})$ is symmetric.  Thus both
  $b$ and $b'$ are in $\mn{ind}(\Bmc_{i + 1})$.
\end{proof}

We analyze the time requirement of the \mn{refine} subroutine.

\begin{lemma} \label{lem:refine-termination}
  $\mn{refine}(q(\bar x))$ can be computed in time polynomial in
  $||q_T|| + ||q||$ (but exponential in $\mn{ar}$) using membership
  queries.
\end{lemma}

\begin{proof}
  Let $(\Bmc_1,\bar b_1), (\Bmc_2,\bar b_2),\dots$ be the sequence
  constructed by the Expand/Minimize phase.  By
  Lemma~\ref{lem:expansioncorrecttwo} Point~2,
  $|\mn{ind}(\Bmc_i)| \leq |\mn{var}(q_T)|$ for all $i \geq 1$.  By
  Lemma~\ref{lem:expansioncorrecttwo} Point~4, the number of
  individuals in the ABoxes $\Bmc_i$ increases in every step. Thus the
  number $n$ of steps is at most $|\mn{var}(q_T)|$. Now let
  $\Bmc_n,\Bmc_{n+1},\dots$ be the sequence constructed by the
  Split/Minimize phase. We can argue in the same way, using
  Lemma~\ref{lem:splitcorrecttwo} Point~2 and
  Lemma~\ref{lem:splitcorrecttwo} Point~4 that the number $m-n$ of
  steps is at most $|\mn{var}(q_T)|$.

  It remains to show that every step runs in polynomial time. For
  this, let $\Omega=\mn{sig}(q)$ be the set of concept and role names
  that occur in the input query $q$. Clearly, $|\Omega|\leq ||q||$.
  Note that none of the applied operations introduces new concept or
  role names, that is, $\mn{sig}(\Bmc_i)\subseteq \Omega$, for all
  $i$.

  For Minimize this is the case, because at most $|\mn{ind}(\Bmc_i)|$
  membership queries are posed in operation~(1) and at most
  $|\Omega|\cdot |\mn{ind}(\Bmc_i)|^2$ membership queries are posed
  in operation~(2).

  For Expand, note that chordless cycles of length $n>n_{\text{max}}$
  can be identified in time polynomial in~$|\mn{ind}(\Bmc)|\leq
  |\mn{var}(q_T)|$. We then need at
  most $2^\mn{ar}$ membership queries to identify the right tuple
  $\bar b_{i+1}\in\tau_i$.

  Finally, for Split, observe that there are at most
  $|\Omega|\cdot |\mn{ind}(\Bmc_i)|^3$ possible triples
  $r(a,b),r(c,b)\in\Bmc_i$. Thus, at most as many membership queries
  are posed. 
%
%
%
%
\end{proof}

Before we show that the result of \mn{refine} is always in the desired
class, we give an example that demonstrates the necessity of the Split
step.  Let
\[q_T(x_1, x_2) \gets A(x_1) \land B(x_2) \land
r(x_1, x_1') \land r(x_2, x_2')\] 
be the target query and let 
\[q(y_1, y_2) \gets A(y_1) \land B(y_2) \land r(y_1, y)
\land r(y_2, y)\] 
be the input to \mn{refine}. Then the result of the Expand and
Minimize phase is $q(y_1,y_2)$ which is not symmetry-free. Thus, the
Split step is needed.

\begin{lemma} \label{lem:expansion-queryclass} 
  If $q_T(\bar y)\in
  \Qmc$ for $\Qmc \in \{ \text{ELQ}, \text{ELIQ}^{\text{sf}},
  \text{CQ}^{\text{csf}}_w\mid w\geq 0 \}$, then $\mn{refine}(q(\bar
  x))\in \Qmc$, for every CQ $q(\bar x)$.
\end{lemma}

\begin{proof}
  Let $p(\bar y) = \mn{refine}(q(\bar x))$.  Assume that there is a
  symmetry $r(x_1, x), r(x_2, x)$, $x_1\neq x_2$ in $p(\bar y)$ such
  that $x \notin \bar y$, none of the atoms occurs on a cycle, and
  there is no atom $s(z,z)$ for any $z\in\{x,x_1,x_2\}$. Note that
  $x_i\neq x$ due to the last condition, for $i\in\{1,2\}$.

  Recall that the query $p(\bar y)$ is the result of exhaustively
  applying the steps Split and Minimize.  Thus, for every homomorphism
  $h$ from $q_T$ to $\Umc_{\Amc_{p}, \Omc}$ with $h(\bar x) = \bar y$,
  there must be atoms $r(y_1, y), r(y_2, y)$ in $q_T$ such that $h(y) = x, h(y_1) = x_1, h(y_2) = x_2$.

  It follows that $y$ is not an answer variable of $q_T$. Furthermore
  there is no atom $s(y', y')$ for $y' \in \{y, y_1, y_2\}$ in $q_T$
  since otherwise there must be an atom $s(h(y'), h(y'))$ in $p$.
  Since $q_T$ is in CQ$^{\text{csf}}$ for all choices of $\Qmc$, at
  least one of the atoms $r(y_1, y), r(y_2, y)$ must occur on a cycle.
  Assume that $r(y_1, y)$ occurs on a cycle in $q_T$, the case for
  $r(y_2, y)$ is similar. Since $q_T$ is chordal, $r(y_1, y)$ must
  also be part of a cycle $r(y_1, y), S_1(y, y_3), S_2(y_3, y_1)$ of
  length three. Consider the atoms $r(h(y_1),h(y))=r(x_1,x)$,
  $S_1(h(y),h(y_3))=S_1(x,h(y_3))$, and
  $S_2(h(y_3),h(y_1))=S_2(h(y_3),x_1)$ which occur in $p$. We
  distinguish cases. 
  \begin{itemize}

    \item If $h(y_3)\notin\{x,x_1\}$, then $r(x_1,x)$,
      $S_1(x,h(y_3))$, $S_2(h(y_3),x_1)$ is a cycle of length three in
      $p$ which contains $r(x_1,x)$, contradicting our initial
      assumption.

    \item If $h(y_3)=x$, then $S_1(x,x)$ is an atom in $p$,
      contradicting our initial assumption.

    \item If $h(y_3)=x_1$, then $S_2(x_1,x_1)$ is an atom in $p$,
      contradicting our initial assumption.

  \end{itemize}
  Thus, $p$ is symmetry-free.
%

  It remains to show that $p(\bar y)$ is chordal if $q_T$ is, an ELQ
  if $q_T$ is, and an ELIQ if $q_T$ is. Let $p'(\bar z)$ be the
  intermediate query obtained after the first phase of Expand and
  Minimize. By non-applicability of Expand, there is no chordless
  cycle $R_0(x_0, x_1), \ldots, R_{n - 1}(x_{n - 1}, x_n)$ in $p'(\bar
  z)$ of length $n > n_\text{max}$ and in case of $\Qmc =
  \text{CQ}^\text{csf}$, $\{ x_0, \ldots, x_{n - 1}\} \not \subseteq
  \bar z$. For $\Qmc = \text{CQ}^\text{csf}$ this means that every
  cycle in $p'(\bar z)$ of length at least four that contains at least
  one quantified variable has a chord.  For $\Qmc \in \{ \text{ELQ},
  \text{ELIQ$^\text{sf}$} \}$, with $n_\text{max} = 0$, this means
  that $p'(\bar z)$ does not contain any cycle at all. By
  Lemma~\ref{lem:expansioncorrecttwo} Point~2, we have $\mn{var}(p')
  \subseteq \mn{img}(h^*)$, for all homomorphisms $h$ from $q_T$ to
  $\Umc_{\Amc_{p'}, \Omc}$ with $h(\bar x) = \bar z$. Since
  $q_T$ is connected and tree-shaped, $p'$ must be connected and
  tree-shaped.  Moreover in the case of $q_T \in \text{ELQ}$, all
  variables $a \in \mn{img}(h^*)$ are reachable by a path $r_0(a_0,
  a_1), \ldots, r_{n-1}(a_{n - 1}, a)$ from the root $a_0$, thus $p$
  is  ditree-shaped in this case. Thus $p'(\bar z) \in \Qmc$ in all
  cases.

  It remains to observe that the Split operation preserves these
  properties, thus $p(\bar y)$ is as required. 
\end{proof}

\begin{lemma} \label{lem:progress}
  Let \Omc be an \ELdr-ontology. 
  Let $\Amc_1$ and $\Amc_2$ be ABoxes and $\bar a_i$, $i\in\{1,2\}$, be
  tuples of individuals from $\Amc_i$ of the same length. Moreover, let
  $q(\bar z)$ be $\Cmc^3_{\Amc_1,\Omc}\times \Cmc^3_{\Amc_2,\Omc}$
  viewed as CQ with answer variables $\bar z=\bar a_1\otimes \bar a_2$ and 
  let $p(\bar x)$ be the result of $\mn{refine}(q(\bar z))$ with
  respect to some target query $q_T(\bar y)$.
  Then there is a homomorphism $h_i$ from $p(\bar x)$ to 
  $\Umc_{\Amc_i,\Omc}$ with $h_i(\bar x)=(\bar a_i)$, for $i\in\{1,2\}$.
\end{lemma}

\begin{proof}
    By Lemma~\ref{lem:expansioncorrecttwo} Point~3, there is a homomorphism $h$
    from $p$ to $\Cmc^3_{\Amc_1,\Omc}\times \Cmc^3_{\Amc_2,\Omc}$ with $h(\bar
    x) = \bar z$ and it is a property of products that there are homomorphisms
    $g_i$ from $\Cmc^3_{\Amc_1,\Omc}\times \Cmc^3_{\Amc_2,\Omc}$ to
    $\Cmc^3_{\Amc_i, \Omc}$ with $g_i(\bar z) = \bar a_i$.
    Composing $h$ and $g_i$ yields homomorphisms $g_i'$ from $p$ to
    $\Cmc^3_{\Amc_i, \Omc}$ with $g_i'(\bar x) = \bar a_i$.
    It follows then, since $p(\bar x) \in \text{CQ}^\text{csf}$ by
    Lemma~\ref{lem:compactunivers} that there are also homomorphisms $h_i$ from
    $p(\bar x)$ to $\Umc_{\Amc_i, \Omc}$ with $h_i(\bar x) = a_i$, as required.
\end{proof}

\section{Proofs for Section~\ref{sect:refine}}

We start with describing the second version of the \mn{refine}
subroutine in full detail. It gets as input a CQ $q'_H(\bar x')$ such
that $q'_H \subseteq_\Omc q_T$ and produces a CQ $q_H(\bar x)$ such
that $q'_H \subseteq_\Omc q_H \subseteq_\Omc q_T$ and $|\mn{var}(q_H)| \leq |\mn{var}(q_T)|$.
The initial call to \mn{refine} in Algorithm~\ref{alg:abox} is
dropped when this version of \mn{refine} is used.
Thus, the argument $q'_H(\bar x')$ is always the
product of two 3-compact models. 
For notational convenience, we prefer to view $q'_H(\bar x')$ as a
pair $(\Amc,\bar a)$ where $\Amc=\Cmc^3_{\Amc_1, \Omc} \times
\Cmc^3_{\Amc_2, \Omc}$ and $\bar a = \bar a_1 \otimes \bar a_2$. We
know that $\Amc_i,\Omc \models q_T(\bar a_i)$ for $i \in \{1,2\}$.

As in the first version of \mn{refine}, minimization is a crucial
ingredient. However, we minimize in a slightly different
way here. 

\medskip\noindent\textbf{Minimize.} Let \Bmc be an ABox that
contains all individuals from $\bar a_1 \otimes \bar a_2$.
Then $\mn{minimize}(\Bmc)$ is the ABox $\Bmc'$ obtained from $\Bmc$ by
exhaustively applying the following operation: choose a
$c \in \mn{ind}(\Bmc) \setminus (\bar a_1 \otimes \bar a_2)$ and remove all assertions that involve
$c$. Use a membership query to check whether, for the
resulting ABox $\Bmc'$,
$\Bmc',\Omc \models q_T(\bar a_1 \otimes \bar a_2)$. If this is the
case, proceed with $\Bmc'$ in place of \Bmc.

\smallskip
\noindent

%
The modified \mn{refine} subroutine constructs a sequence of ABoxes
$\Bmc_1,\Bmc_2,\dots$ starting with
\[
    \Bmc_1=\mn{minimize}(\Cmc^3_{\Amc_1,\Omc} \times \Cmc^3_{\Amc_2,\Omc})
\]
and such that $\Bmc_i,\Omc \models q_T(\bar a_1 \otimes \bar a_2)$ for all
$i \geq 1$. Note that in contrast to the first version of refine, the
individuals in the answer tuple (which correspond to the answer
variables) are never modified.

Each ABox $\Bmc_{i+1}$ is obtained from $\Bmc_i$ by a local unraveling.
All individuals in $\Bmc_1$ are pairs $(c_1,c_2)$ and the same shall
be true for the individuals in the ABoxes $\Bmc_2,
\Bmc_3,\dots$. Informally, unraveling replaces components $c_i$ that
are individuals from
$\mn{ind}(\Cmc^3_{\Amc_i,\Omc}) \setminus \mn{ind}(\Amc_i)$ with
corresponding individuals from
$\mn{ind}(\Umc_{\Amc_i,\Omc}) \setminus \mn{ind}(\Amc_i)$ in a
step-by-step fashion. To make this formal, we call
$(c_1,c_2) \in \mn{ind}(\Bmc_i)$ \emph{unraveled} if
$c_i \in \mn{ind}(\Umc_{\Amc_i,\Omc})$ for each $i \in \{1,2\}$. Note
that $(c_1,c_2) \in \Bmc_1$ and
$c_i \notin \mn{ind}(\Umc_{\Amc_i,\Omc})$ implies that $c_i$ is of the
form $c_{a,i,s,C}$. The same will be true for all ABoxes
$\Bmc_2,\Bmc_3,\dots$.
%
%
%
%
%
We now describe the unraveling step.



\medskip
\noindent 
\textbf{Unravel.} Remove every assertion
$r((c_1,c_2),(d_1,d_2)) \in \Bmc_i$ with $(c_1,c_2)$ unraveled and
$(d_1,d_2)$ not unraveled.  Let $d'_j=d_j$ if $d_j$ occurs in 
$\Umc_{\Amc_j,\Omc}$ and $d_j=c_jrC$ if $d_j=c_{a,\ell,s,C}$, for $j \in
\{1,2\}$. Compensate by adding the following assertions:
\begin{itemize}

\item $r((c_1,c_2),(d'_1,d'_2))$;

\item $A(d'_1,d'_2)$ for all $A(d_1,d_2) \in \Bmc_i$;

\item $r((d'_1,d'_2),(e_1,e_2))$ for all $r((d_1,d_2),(e_1,e_2)) \in
  \Bmc_i$. 

\end{itemize}

\medskip We call $(d'_1,d'_2)$ a \emph{copy} of $(d_1,d_2)$. Note that
unraveling might introduce several copies of the same original element
$(d_1,d_2)$ and that $(d_1,d_2)$ might or might not be present after
unraveling, the latter being the case when
$r( (c_1,c_2),(d_1,d_2))$ is the only assertion that mentions
$(d_1,d_2)$.

\medskip 
\noindent 
After unraveling, we apply the Minimize step and the resulting ABox is
$\Bmc_{i+1}$.

\medskip 
\noindent
We prove later that the Unravel step can only be applied polynomially
many times, let the resulting ABox be $\Bmc_n$.  Let $I$ denote the
set of all individuals in $\Bmc_{n}$ that are reachable from some
individual in $\mn{ind}(\Amc_1) \times \mn{ind}(\Amc_2)$ in the
directed graph
$G_{\Bmc_{n}} = (\mn{ind}(\Bmc_{n}),\{ (a,b) \mid r(a,b) \in
\Bmc_{n}\})$.  It is easy to see that all individuals in $I$ are
unraveled. However, the restriction of $\Bmc_{n}$ to
$\mn{ind}(\Bmc_{n}) \setminus I$ might contain individuals that are
not unraveled. 

For example, consider the \ELdr-ontology $\Omc = \{ A \sqsubseteq \exists r.A,
\ B \sqsubseteq \exists r.B \}$ and the boolean target query $q_T \gets r(x_1,
x_2) \land r(x_2, x_3) \land r(x_3, x_4)$.  At some point during the learning
algorithm, the \mn{refine} subroutine might be called with $\Amc = \Cmc^3_{\Amc_1, \Omc}
\times \Cmc^3_{\Amc_2, \Omc}$ with $\Amc_1 = \{A(a)\}$ and $\Amc_2 =
\{B(b)\}$. By construction of the 3-compact model, $\Amc$ contains a cycle of
length $4$, consisting of individuals $(c_{a, i, r, A}, c_{b, i, r, B})$ for $i
\in \{1, \ldots 4\}$, that is reachable from from $\mn{ind}(\Amc_1) \times
\mn{ind}(\Amc_2) = \{ (a, b) \}$. The first Minimize step might then remove all
individuals from $\Amc$ that are not on the cycle, since there is a
homomorphism $h$ from $q_T$ to $\Umc_{\Amc, \Omc}$ with $h(x_i) = (c_{a, i, r,
A}, c_{b, i, r, B})$ for $i \in \{1, \ldots, 4\}$. Since only the cycle
consisting of not unraveled individuals remains, the Unravel step cannot be
applied.

To deal with this issue, we apply the original \mn{refine} subroutine
from Section~\ref{sec:short-cycles} to $\Bmc_{n}$ (with
$n_{\text{max}}=3$), resulting in a sequence
$(\Bmc_n,\bar b_n), (\Bmc_{n+1},\bar b_{n+1}),\dots$ where
$\bar b_n = a_1 \otimes a_2$, in a slightly adapted way:
\begin{enumerate}
  
\item the individuals in $I$ are not touched, that is, no cycle that
  involves an individual from $I$ is considered in the Expansion step
  nor is any assertion removed during the Minimize step that
  contains an individual from $I$;

\item as a consequence, the Expansion step cannot involve individuals
  in $\bar a_1 \otimes \bar a_2$, and thus the exponential blowup in
  the arity is avoided. In fact, $\bar b_n = \bar b_{n+1} = \cdots$.
  
\item the Splitting step is not applied.
  
\end{enumerate}
We now analyze the second version of \mn{refine}, starting
with the following lemma.
\begin{lemma}
  \label{lem:prodcycles}
  Let $i \geq 1$. Every cycle in $\Bmc_i$ of
  length at most three consists only of individuals from $\mn{ind}(\Amc_1)\times\mn{ind}(\Amc_2)$.
\end{lemma}
\begin{proof}
  We prove the lemma by induction on $i$. In the induction start,
  $\Bmc_1=\mn{minimize}(\Cmc^3_{\Amc_1,\Omc} \times
  \Cmc^3_{\Amc_2,\Omc})$.
  If $\Bmc_1$
  contains a cycle $r((a_1,a_2),(a_1,a_2))$ of length 1 with
  $a_i \notin \Amc_i$ for some $i \in \{1,2\}$, then $r(a_1,a_1)$ is a
  cycle of length~1 in $\Cmc^3_{\Amc_1,\Omc}$ which is not the case by
  Lemma~\ref{lem:noanocycles}.  Next assume that
  $\Bmc_1$
  contains a cycle
  $r_0((a_1,a_2),(b_1,b_2)), r_1((b_1,b_2),(a_1,a_2))$ of length 2.
  Assume w.l.o.g.\ that $a_1 \notin \mn{ind}(\Amc_1)$.  If $a_1=b_1$,
  then $r_0(a_1,a_1)$ is a cycle of length~1 in
  $\Cmc^3_{\Amc_1,\Omc}$, but this is not the case by
  Lemma~\ref{lem:noanocycles}. If $a_1 \neq b_1$, then
  $r_0(a_1,b_1),r_1(b_1,a_1)$ is a cycle of length~2 in
  $\Cmc^3_{\Amc_1,\Omc}$, which again contradicts
  Lemma~\ref{lem:noanocycles}.  If
  $\Bmc_1$
  contains a cycle
  $r_0((a_1,a_2),(b_1,b_2))$, $r_1((b_1,b_2),(c_1,c_2))$, $r_2((c_1,c_2),(a_1,a_2))$ of length 3, we
  can argue similarly that $\Cmc^3_{\Amc_1,\Omc}$ 
  contains a cycle of
  length~1 or~3 that involves an individual not in
  $\mn{ind}(\Amc_i)$, again obtaining a contradiction.

  For the induction step, we show that both the Minimize step and the
  Unravel step do not create cycles of length 1, 2, or 3 that involve
  individuals not from $\mn{ind}(\Amc_1) \times \mn{ind}(\Amc_2)$. Since the
  Minimize step only removes assertions, it cannot create any new cycles.
  For the Unravel step, let the lemma hold for $\Bmc_i$ and let $\Bmc$ be
  $\Bmc_i$ after the Unravel step. Let $r_0((a_1, a_2), (b_1, b_2))$,
  $r_1((b_1, b_2), (c_1, c_2))$, $r_2((c_1, c_2), (a_1, a_2))$ be a new cycle of
  length $3$ in $\Bmc$. Since the cycle is new, one of $(a_1, a_2)$, $(b_1,
  b_2)$ or $(c_1, c_2)$ must be $(d_1', d_2')$, a new individual created by the
  Unravel. But by the definition of the Unravel step, replacing $(d_1',
  d_2')$ with $(d_1, d_2)$ in the cycle must yield a cycle in $\Bmc_i$ which
  contradicts the induction hypothesis.
  The same argument can be applied to cycles of length 1 and 2.
\end{proof}
Recall that we first construct a sequence 
$\Bmc_1, \Bmc_2,\dots$ using the Unravel and 
Minimize steps.  With $\Bmc_i'$, $i \geq 1$, we denote the result of 
 only applying the Unravel step to $\Bmc_i$, but not the Minimize 
 step. 
\begin{lemma} \label{lem:localunraveling}
    For all $i \geq 1$, $\Bmc_i', a \preceq \Bmc_i, a$ and
    $\Bmc_i, a \preceq \Bmc_i', a$, for all $a \in
    \mn{ind}(\Bmc_i)\cap \mn{ind}(\Bmc_i')$,
    and $\Bmc_i' , a' \preceq \Bmc_i, a$ and $\Bmc_i , a \preceq
    \Bmc_i', a'$ for all copies $a' \in \mn{ind}(\Bmc_i') \setminus
    \mn{ind}(\Bmc_i)$ of some $a\in\mn{ind}(\Bmc_i)$.
\end{lemma}
\begin{proof}
  Define a relation $S\subseteq \mn{ind}(\Bmc_i)\times
  \mn{ind}(\Bmc_i')$ by taking:
  \begin{itemize}

    \item $(a,a)\in S$, for all $a\in
      \mn{ind}(\Bmc_i)\cap\mn{ind}(\Bmc_i')$, and 

    \item $(a,a')\in S$, for all copies $a'\in
      \mn{ind}(\Bmc_i')$ of some element $a\in\mn{ind}(\Bmc_i)$.

  \end{itemize}
  It is routine to verify that $S$ serves as witness for the claimed
  simulations from $\Bmc_i$ to $\Bmc_i'$, and its inverse $S^-$ serves
  as witness for the claimed
  simulations from $\Bmc_i'$ to $\Bmc_i$.
\end{proof}
The next lemma is the most intricate to prove in the analysis of 
the second version of \mn{refine}.
\begin{lemma}\label{lem:refinement-invariant}
  For all $i \geq 1$, $\Bmc_i,\Omc \models q_T(\bar a_1 \otimes \bar a_2)$.
\end{lemma}
\begin{proof}\
  We prove the lemma by induction on $i$. The induction start is
  immediate since $\Umc_{\Amc_1,\Omc} \times \Umc_{\Amc_2,\Omc}
  \models q_T(\bar a_1 \otimes \bar a_2)$ and there is a homomorphism
  from $\Umc_{\Amc_1,\Omc} \times \Umc_{\Amc_2,\Omc}$ to
  $\Cmc^3_{\Amc_1,\Omc} \times \Cmc^3_{\Amc_2,\Omc}$
  that is the identity on $\bar a_1 \otimes \bar a_2$. Thus,
  $\Cmc^3_{\Amc_1,\Omc} \times \Cmc^3_{\Amc_2,\Omc},\Omc\models
  q_T(\bar a_1\otimes \bar a_2)$.
  It remains to note that 
  $\Bmc_1=\mn{minimize}(\Cmc^3_{\Amc_1,\Omc} \times
  \Cmc^3_{\Amc_2,\Omc})$ and that the Minimize step preserves
  $\Bmc_1,\Omc \models q_T(\bar a_1\otimes \bar a_2)$.
  
  For the induction step, consider $\Bmc_{i+1}$ with $i \geq 1$. By
  induction hypothesis, there is a homomorphism $h$ from $q_T$ to
  $\Umc_{\Bmc_{i},\Omc}$ with $h(\bar x)=\bar a_1 \otimes \bar a_2$.
  By Lemma~\ref{lem:strongsymfree}, we can assume that $q_T$ is
  strongly symmetry-free. Let \Bmc be the result of applying the
  unraveling step to $\Bmc_{i}$, and let $U$ be the set of all
  individuals $(d_1,d_2) \in \mn{Ind}(\Bmc_i)$ such that some assertion
  $r((c_1,c_2),(d_1,d_2))$ was removed in that step. Note that
  if $(d_1,d_2) \in U$, then $(d_1,d_2) \notin \mn{Ind}(\Amc_1) \times
  \mn{Ind}(\Amc_2)$. In what follows,
  we construct a homomorphism $g$ from $q_T$ to $\Umc_{\Bmc,\Omc}$
  with $g(\bar x)=\bar a_1 \otimes \bar a_2$.  Thus, $\Bmc,\Omc
  \models q_T(\bar a_1 \otimes \bar a_2)$. By definition of the
  Minimize step, this implies $\Bmc_{i+1},\Omc \models q_T(\bar a_1
  \otimes \bar a_2)$ as desired.

  We first observe the following, which can be proved by a
  straightforward induction on $j$. 

  \smallskip \noindent\textit{Claim 1.} For all $j \geq 0$,
  if $R((c_1,c_2),(d_1,d_2)) \in \Bmc_j$ with $(c_1,c_2)$ unraveled
  and $(d_1,d_2)$ not unraveled, then $R$ is a role name, but not
  an inverse role.

  \smallskip\noindent For a variable $x$ in $q_T$, let us denote with
  $V_{x}$ the set of all atoms $R(x,y)\in q_T$ such that
  $h(y) \in \mn{Ind}(\Bmc_i)$ is unraveled.  We observe the following.

  \smallskip \noindent\textit{Claim 2.} Let $x\in\mn{var}(q_T)$ such
  that $h(x)=(d_1,d_2)\in U$. 
  Then  there is a role name $r$ such
  that all atoms in $V_x$ are of shape $r(y,x)$ and
  one of the following
  is the case:
  \begin{enumerate}[label=(\roman*)]

    \item $V_x$ is a singleton;

    \item $d_1$ has the form $c_{b,0,r,C}$ and for every
      $r(y,x)\in V_x$, $\Amc_2$ contains an assertion $r(b',d_2)$
      with $h(y)=(b,b')$;



    \item $d_2$ has the form $c_{b,0,r,C}$ and for every
      $r(y,x)\in V_x$, $\Amc_1$ contains an assertion $r(b',d_1)$
      with $h(y)=(b',b)$;

    \item $d_1$ has the form $c_{b_1,0,r,C_1}$, $d_2$ has the form
      $c_{b_2,0,r,C_2}$, and $h(y)=(b_1,b_2)$ for every $r(y,x)\in V_x$.


  \end{enumerate}

  \smallskip \noindent\textit{Proof of Claim~2.} 
  To show the first part, let
  $R(y_1,x),S(y_2,x)\in V_x$.  Since $h(x)=(d_1,d_2)$ is not
  unraveled, but $h(y_1)$ and $h(y_2)$ are unraveled,
  $R$ and $S$ are role names by Claim~1. Moreover, $(d_1,d_2)$
  not being unraveled means that at least one
  one of the $d_j$ takes the shape $c_{a,k,r,C}$ for some
  role name $r$. By definition of $\Cmc^3_{\Amc_j,\Omc}$, for every
  $s(d,c_{a,k,r,C})\in \Cmc^3_{\Amc_j,\Omc}$, we have $s=r$. Hence,
  for every $s(d,(d_1,d_2))\in \Bmc_i$ we have $s=r$ as well. Thus,
  $R=S=r$ and all assertions in $V_x$ are based on the same role name
  $r$.

  Now for the second part.  Assume that Case~(i) does not apply. Then
  we find $r(y_1,x),r(y_2,x)\in V_x$ with $y_1\neq y_2$. Since $q_T$
  is strongly symmetry-free and $x$ is not an answer variable (which
  follows from $h(x)\in U$) one of the atoms, say $r(y_1,x)$, occurs
  on a cycle $p$ in $q_T$. Since $q_T$ is chordal, we can assume that
  $p$ has length at most three. Since $h$ is a homomorphism from $q_T$
  to $\Umc_{\Bmc_i,\Omc}$, the `$h$-image of $p$' contains a cycle
  $p'$ of length at most three in $\Bmc_i$.  By
  Lemma~\ref{lem:prodcycles}, $p'$ consists only of elements from
  $\mn{ind}(\Amc_1)\times\mn{ind}(\Amc_2)$. Thus $h(x)$ cannot be
  involved in $p'$, since $h(x)=(d_1,d_2)$ is not
  unraveled. Consequently, the cycle $p$ has to be of the
  shape \[r(y_1,x), r^-(x,z),s(y_1,z)\] and
  $h(y_1)=h(z)=(b_1,b_2)\in \mn{ind}(\Amc_1)\times\mn{ind}(\Amc_2)$.
  It follows that we must have $i=1$ since for $i > 1$, all successors
  of elements of $ \mn{ind}(\Amc_1)\times\mn{ind}(\Amc_2)$ are
  unraveled. We distinguish the following cases:
  \begin{itemize}

  \item $d_1\in\mn{ind}(\Amc_1)$ and $d_2\in\mn{ind}(\Amc_2)$.

    Impossible because $(d_1,d_2)$ is not unraveled.

  \item $d_1$ has shape $c_{b_1,0,r,C_1}$ and $d_2$ has shape
    $c_{b_2,0,r,C_2}$.  By definition of the models
    $\Cmc_{\Amc_i,\Omc}^3$ and since in $\Bmc_1\subseteq \Cmc^3_{\Amc_1,\Omc}
    \times \Cmc^3_{\Amc_2,\Omc}$, $(b_1,b_2)$ is the unique unraveled
    $r$-predecessor of $(d_1,d_2)$ in $\Bmc_i=\Bmc_1$.  Then we are in
    Case~(iv).


    \item $d_1$ has shape $c_{b_1,0,r,C}$ 
      and $d_2\in
      \mn{ind}(\Amc_2)$.

      Then $b_1$ is the unique $r$-predecessor of $d_1$ in
      $\Cmc^3_{\Amc_1,\Omc}$ that can appear in the first component of
      an unraveled element.
      Let $r(y,x)\in V_x$. Because $h(y)$ is
      unraveled and $i=1$,
      $h(y) \in \mn{ind}(\Amc_1) \times \mn{ind}(\Amc_2)$. Since
      $\Bmc_1\subseteq \Cmc^3_{\Amc_1,\Omc} \times \Cmc^3_{\Amc_2,\Omc}$, $r(y,x)\in V_x$ thus implies that there
      is an assertion $r(b',d_2)\in \Amc_2$ such that
      $h(y)=(b_1,b')$. Thus, we are in Case~(ii).

    \item $d_2$ has shape $c_{b_2,0,r,C}$ 
      and
      $d_1\in \mn{ind}(\Amc_1)$. 

      We argue as in the previous case, but end up in
      Case~(iii).

  \end{itemize}
  This finishes the proof of Claim~2. For the next claim, we associate
  with every variable $x\in \mn{var}(q_T)$ with
  $h(x) \in \mn{ind}(\Bmc_i)$ the set $Z_x$ that consists of all
  variables $y \in \mn{var}(q_T)$ such that $q_T$ contains a path
  $R_0(z_0,z_1),\ldots,R_{m-1}(z_{m-1},z_m)$ from $x$ to $y$ where
  $h(z_1),\dots,h(z_m)$ are all located in the subtree of
  $\Umc_{\Bmc_i,\Omc}$ rooted at $h(x)$, but are different from
  $h(x)$.
%
%

      \medskip
  \noindent\textit{Claim 3.} For all $y \in \mn{var}(q_T)$
  with $h(y)\notin\mn{ind}(\Bmc_i)$, there is at most one 
  $x\in\mn{var}(q_T)$ with $y\in Z_x$ and $h(x)\in U$.

  \medskip \noindent\textit{Proof of Claim~3.} Suppose that 
$y \in \mn{var}(q_T)$ with $h(y) \notin\mn{ind}(\Bmc_i)$ and
that there are
  distinct variables $x_1,x_2\in \mn{var}(q_T)$ with $y\in Z_{x_j}$
  and $h(x_j)\in U$ for $j\in\{1,2\}$. Let 
  \begin{gather*}
  p_1=R_0(z_0,z_1),\ldots,R_n(z_{n-1},z_n)\quad\text{and}\quad \\
  p_2=S_0(z_0',z_1'),\ldots,S_m(z_{m-1}',z_m')
  \end{gather*}
  be paths in $q_T$ from $x_1$ to
  $y$ and from $x_2$ to $y$, respectively, such that $h(z_j)\neq h(x_1)$ for all
  $j\in\{1,\ldots,n\}$ and $h(z_j')\neq h(x_2)$ for all
  $j\in\{1,\ldots,m\}$. Note that $h$ is a
  homomorphism from $p_j$ to the subtree of
  $\Umc_{\Bmc_i,\Omc}$ rooted at $h(x_j)$, for $j\in\{1,2\}$. Since
  $h(y)$ is both in the subtree below $h(x_1)$ and below $h(x_2)$, we
  have $h(x_1)=h(x_2)$. 

  We analyze the structure of the paths $p_1$ and $p_2$. Let us first
  verify that all $R_j$ and all $S_j$ can be assumed to be role
  names. We do this explicitly only for the $R_j$. Let $\Sbf$ denote
  the subtree of $\Umc_{\Bmc_i,\Omc}$ rooted at $h(x_1)$, that is,
  the restriction of $\Umc_{\Bmc_i,\Omc}$ to all traces that start
  with $h(x_1)$, including $h(x_1)$ itself. By construction of
  $\Umc_{\Bmc_i,\Omc}$, \Sbf is a ditree. Then $R_0$
must be a role name since $R_0(h(x_1),h(z_1)) \in
  \Sbf$, $h(x_1)$ is the root of \Sbf, and $h(z_1)$ in \Sbf. Now, let
  $\ell$ be minimal such that $R_\ell$ is an inverse role $r^-$ and consider the
  atoms $R_{\ell-1}(z_{\ell-1},z_\ell), r^-(z_\ell,z_{\ell+1})$ in $q_T$. Since $h$
  is a homomorphism and $\Sbf$ is a ditree, we know that $R_{\ell-1}=r$,
  and thus there are atoms $r(z_{\ell-1},z_\ell), r(z_{\ell+1},z_{\ell})$ in
  $q_T$. 

  Now, if $z_{\ell-1}=z_{\ell+1}$, we can drop these two atoms from
  the path. Otherwise, since $q_T$ is strongly symmetry-free and $z_\ell$ is
  not an answer variable (as $h(z_\ell)$ is in \Sbf but different from
  its root), one of these atoms occurs on a cycle $p$ in $q_T$.  Let
  us assume that this is atom $r(z_{\ell-1}, z_\ell)$, the case of
  atom $r(z_{\ell+1},z_{\ell})$ is analogous.  Since $q_T$ is chordal,
  we can assume that $p$ has length at most three. Since $h$ is a
  homomorphism from $q_T$ to $\Umc_{\Bmc_i,\Omc}$, the image of $p$
  contains a cycle $p'$ of length at most three in
  $\Umc_{\Bmc_i,\Omc}$. Even if $h$ is not injective, the cycle $p'$
  must contain $h(z_{\ell})$ or $h(z_{\ell-1})$.  However, both
  possibilities lead to a contradiction. If $p'$ contains
  $h(z_{\ell})$, then
  $h(z_\ell) \in \mn{ind}(\Amc_1)\times\mn{ind}(\Amc_2)$ by
  Lemma~\ref{lem:prodcycles} but this is not the case since
  $h(z_\ell)$ is in \Sbf and different from $h(x_1)$. If $p'$ contains
  $h(z_{\ell-1})$, then $h(z_{\ell-1})$ must be $h(x_1)$, and $p'$
  witnesses that $h(x_1) \in \mn{ind}(\Amc_1)\times\mn{ind}(\Amc_2)$,
  in contradiction to $h(x_1)\in U$.
 
  At this point, we have established that all $R_j$ and $S_j$ are 
  role names $r_j,s_j$. Since \Sbf is a ditree, it follows that $m=n$ and
  $r_j=s_j$ for all $j$. Since $z_0\neq z_0'$ and $z_n=z_m'$, there
  is some $\ell>0$ such that $z_\ell=z'_\ell$, $z_{\ell-1}\neq
  z'_{\ell-1}$. But then $q_T$ contains atoms
  $r_\ell(z_{\ell-1},z_\ell),r_\ell(z_{\ell-1}',z_\ell)$. This leads
  to a contradiction in the same way as above. This finishes the proof
  of Claim~3.

  \medskip We now define the required homomorphism $g$ in four stages, as follows.
 
  \begin{enumerate}

    \item Define $g(x)=h(x)$ for all $x\in\mn{var}(q_T)$ such that
      $h(x)\in \mn{ind}(\Bmc_i)\setminus U$ or $h(x)$ is in the
      subtree below some element $d\notin U$.


    \item For every $x\in\mn{var}(q_T)$ with $h(x)=(d_1,d_2)\in U$, we distinguish cases according to Claim~2:

      \begin{enumerate}

      \item If $V_x=\emptyset$, then define $g(x)=h(x)$.  We argue
	that this is well-defined, that is, $h(x)\in \mn{ind}(\Bmc)$.
	Suppose to the contrary that $h(x)\notin \mn{ind}(\Bmc)$. By
	definition of the unraveling operation, this can only be the
	case if $\Bmc_i$ contains only a single assertion that mentions $h(x)$ and this
	assertion is of shape $r((c_1,c_2),h(x))$ with $(c_1,c_2)$
	unraveled. Since $x$ has to occur in some atom in
	$q_T$ and $h$ is a homomorphism, $x$ occurs in an atom
	$r(z,x)\in q_T$ such that $h(z)=(c_1,c_2)$. Hence, $r(z,x)\in
	V_x\neq \emptyset$, contradiction.


	\item If Case~(i) applies and 
	  $V_x=\{r(y,x)\}$, 
	  define $g(x)$ to be the copy $(d_1',d_2')$ of
	  $(d_1,d_2)$ introduced when unraveling 
	  $r(h(y),h(x))\in \Bmc_i$.

	\item If $V_x\neq\emptyset$ and Case~(ii) applies (but
          Case~(i) does not), then define $g(x)$ to be the copy
          $(brC,d_2)$ of $(d_1,d_2)$, where $b,C$ are as in Case~(ii)
          of Claim~2.


	\item If $V_x\neq\emptyset$ and Case~(iii) applies (but
	  Case~(i) does not), analogously define $g(x)$ to be the
          copy $(d_1,brC)$.

	\item If $V_x\neq\emptyset$ and Case~(iv) applies (but
	  Case~(i) does not), define $g(x)$ to be the
          copy $(b_1rC_1,b_2rC_2)$ where $b_1,b_2,C_1,C_2$ are
	  as in Case~(iv).

      \end{enumerate}

    \item For every $x$ with $h(x) \in U$ and every $y \in Z_x$,
      $h(y)$ is a trace that starts with $h(x)$, c.f.\ the definition
      of $\Umc_{\Bmc_i,\Omc}$. Define $g(y)$ to be the same trace, but
      with the first element $h(x)$ replaced by $g(x)$. It can be
      verified that $g(y)$ is indeed an element in $\Umc_{\Bmc,\Omc}$
      using the fact that, by Lemma~\ref{lem:localunraveling}, the 
      subtrees below $g(x)$ and $h(x)$ in $\Umc_{\Bmc,\Omc}$ and
      $\Umc_{\Bmc_i,\Omc}$, respectively, are identical. 

    \item For every $y$ with $h(y)$ in the subtree below some
      $(d_1,d_2) \in U$ but different from $(d_1,d_2)$, and such that
      $y \notin Z_x$ for all $x$ with $h(x) \in U$, choose some
      copy $(d'_1,d'_2)$ of $(d_1,d_2)$ and define $g(x)$ to be
      the trace $h(x)$ with the first element $(d_1,d_2)$ replaced by
      $(d'_1,d'_2)$.


%
%
  \end{enumerate}
  It is easy to see that the four stages above define $g(x)$ for all
  $x \in \mn{var}(q_T)$.

  \medskip
  \noindent\textit{Claim~4.} $g$ is a homomorphism from $q_T$ to
  $\Umc_{\Bmc,\Omc}$ with $g(\bar x)=\bar a_1\otimes\bar a_2$.

  \smallskip\noindent\textit{Proof of Claim~4.} For
  $g(\bar x)=\bar a_1\otimes\bar a_2$, observe that
  $h(x)\in\mn{ind}(\Amc_1)\times\mn{ind}(\Amc_2)$ for every $x\in\bar
  x$ while $U \cap (\mn{ind}(\Amc_1)\times\mn{ind}(\Amc_2))=\emptyset$. Thus, Stage~1 of the definition of $g$ implies 
  $g(\bar x)=h(\bar x)$.

  Now, let $A(x)\in q_T$ and thus $A(h(x))\in \Umc_{\Bmc_i,\Omc}$. We
  distinguish the following cases:
  \begin{itemize}

  \item If $g(x)$ was defined in Stage~1, then $g(x)=h(x)$. First
    assume that $g(x) \in \mn{Ind}(\Bmc_i)$. By
    Lemma~\ref{lem:localunraveling}, we have
    $\Bmc_i,h(x)\preceq \Bmc,g(x)$ and thus, by
    Lemma~\ref{lem:universal-homomorph} Point~2
    $\Umc_{\Bmc_i,\Omc},h(x)\preceq \Umc_{\Bmc,\Omc},g(x)$. Hence,
    also
    $A(g(x))\in \Umc_{\Bmc,\Omc}$ by Lemma~\ref{lem:simfund}.  Now
    assume that $g(x) \notin \mn{Ind}(\Bmc_i)$. Then $h(x)=g(x)$ is a
    trace and traces in $\Umc_{\Bmc_i,\Omc}$ and $\Umc_{\Bmc,\Omc}$
    that end with the same concept $C$ must satisfy the same concept
    names.
    
  \item If $g(x)$ was defined in Stage~2, then $g(x)=(d_1',d_2')$ is a
    copy of $h(x)=(d_1,d_2)$ or $g(x)=h(x)$. By
    Lemma~\ref{lem:localunraveling}, we have
    $\Bmc_i,h(x)\preceq \Bmc,g(x)$, thus
    $A(g(x))\in \Umc_{\Bmc,\Omc}$ by Lemmas~\ref{lem:simfund} and~\ref{lem:universal-homomorph}.

  \item If $g(x)$ was defined in Stage~3 or~4, then $h(x)$ and $g(x)$
    are both traces that end with the same concept $C$ and, by
    construction of universal models, thus make true the same concept
    names. Consequently, $A(h(x))\in \Umc_{\Bmc_i,\Omc}$ implies
    $A(g(x))\in \Umc_{\Bmc,\Omc}$.

  \end{itemize}
  Finally, let $r(x,y)\in q_T$ and thus $r(h(x),h(y))\in
  \Umc_{\Bmc_i,\Omc}$. We distinguish the following cases:
  \begin{itemize}

    \item It cannot be that both $h(x)$ and $h(y)$ are elements of
      $U$, by definition of the unraveling step.  

    \item If both $h(x)$ and $h(y)$ are not elements of $U$, then
      both $g(x)$ and $g(y)$ were defined in the same stage,
      one of Stage~1,~3, and~4. We can then argue very similar to
      the case of concept atoms that $r(g(x),g(y))\in \Umc_{\Bmc,\Omc}$.

    \item If $h(x)=(d_1,d_2)\in U$ and $h(y)\notin U$, then we
      distinguish cases:   
      \begin{itemize}

	\item If $h(y)\notin \mn{ind}(\Bmc_i)$, then it is an
	  $r$-successor of $(d_1,d_2)$ in the tree below $(d_1,d_2)$
	  in $\Umc_{\Bmc_i,\Omc}$. Thus $g(y)$ was defined in Stage~3.
	  If $h(y)$ is trace $(d_1,d_2)rC$, then $g(y)$ is trace
	  $(d'_1,d'_2)rC$ for $g(x)=(d'_1,d'_2)$. It follows that
	  $r(g(x),g(y))\in \Umc_{\Bmc,\Omc}$.

	\item If $h(y)\in \mn{ind}(\Bmc_i)$ is not unraveled, then by
	  definition of the unraveling step, we have
	  $r((d_1',d_2'),h(y))\in \Bmc$ for all copies $(d'_1,d_2')$
	  of $(d_1,d_2)$, and $r((d_1,d_2),h(y))\in \Bmc$. We know
	  that $g(x)$ was defined in Stage~2 and is either $h(x)$ or
	  some copy thereof, and $h(y)$ was defined in Stage~1, thus
	  $g(y)=h(y)$. Consequently,
	  $r(g(x),g(y))\in\Bmc\subseteq\Umc_{\Bmc,\Omc}$.

	\item It cannot be the case that $h(y)\in \mn{ind}(\Bmc_i)$
	  is unraveled: By Claim~2, $S$ is a role name for
	  every atom $S(z,x)\in q_T$ such that $h(z)$ is unraveled.
	  However, this is not the case for the atom $r^-(y,x)\in q_T$
	  we started with.

      \end{itemize}

    \item If $h(x)\notin U$ and $h(y)=(d_1,d_2)\in U$, then $h(x)\in
      \mn{ind}(\Bmc_i)$ since $r(h(x),h(y)) \in \Umc_{\Bmc_i,\Omc}$
      and by definition of universal models. We distinguish cases
      according to Claim~2:
      \begin{itemize}

	\item If $V_y=\emptyset$, then $g(y)=h(y)$,
	    by Stage~2(a). Moreover, as $h(x)\in
	    \mn{ind}(\Bmc_i)\setminus U$, we have $g(x)=h(x)$, by
	    Stage~1.  Hence, $r(g(x),g(y))\in \Bmc\subseteq
	    \Umc_{\Bmc,\Omc}$.

	\item If Case~(i) applies and $V_y=\{r(x,y)\}$ with $h(x)$
	  unraveled, then $g(y)$ was defined in  Stage~2(b) and
	  $r(g(x),g(y))\in \Umc_{\Bmc,\Omc}$.

	\item If Case~(ii) applies to $V_y$, then $d_1$ has the form
	  $c_{b,0,r,C}$ and for every $r(z,y)\in V_y$, $\Amc_2$
	  contains an assertion $r(b',d_2)$ with $h(z)=(b,b')$.
	  Moreover, $g(y)$ was defined in Stage~2(c) and
	  $g(y)=(brC,d_2)$.
	  \begin{itemize}

	    \item If $h(x)$ is unraveled, then $h(x)=g(x)=(b,b')$.  By
	      definition of the unraveling, $r((b,b'),(brC,d_2)\in
	      \Bmc$. Hence, $r(g(x),g(y))\in \Bmc\subseteq
	      \Umc_{\Bmc,\Omc}$.

	    \item If $h(x)$ is not unraveled, then it was defined in 
	      Stage~1 and $h(x)=g(x)$. By definition of the
	      unraveling, $r(h(x),(brC,d_2))\in\Bmc$.

	  \end{itemize}

%
	\item If Case~(iii) applies to $V_y$, the argument is symmetric.

	\item If Case~(iv) applies to $V_y$, then $d_1$ has the form
          $c_{b_1,0,r,C_1}$, $d_2$ has the form 
          $c_{b_2,0,r,C_2}$, $h(x)=(b_1,b_2)$, and $g(y)$ was defined in
	  Stage~2(e) and $g(y)=(b_1rC_1,b_2rC_2)$. Since $h(x)$ is
	  unraveled, we have $g(x)=h(x)$ by Stage~1, and the
	  definition of the unraveling yields $r(g(x),g(y))\in \Bmc$.

      \end{itemize}

  \end{itemize}
  This finishes the proof of Claim~4 and thus of the lemma.
\end{proof}

\begin{lemma}\label{lem:minimization-works}
  For all $i \geq 1$, 
  \begin{enumerate}

    \item $\Bmc_i$ is \Omc-saturated;

    \item If $h$ is a homomorphism from $q_T$ to $\Umc_{\Bmc_i,
      \Omc}$ with $h(\bar y) = \bar a_1 \otimes \bar a_2$, then 
      $\mn{ind}(\Bmc_i) \subseteq \mn{img}(h^*)$;

    \item $\Bmc_{i + 1}, \bar a_1 \otimes \bar a_2 \rightarrow \Bmc_i, \bar a_1 \otimes \bar a_2$.

  \end{enumerate}
\end{lemma}
\begin{proof} We prove Point~1 by induction on $i$. For $i=1$, observe that
    either
  $\Bmc_1  = \Cmc^3_{\Amc_1,\Omc} \times \Cmc^3_{\Amc_2,\Omc}$,
  or $\Bmc_1 = \Amc_{q^{\bot}}$.
  In the first case
  both $\Cmc^3_{\Amc_i,\Omc}$ are \Omc-saturated and thus also their product by
  Lemma~\ref{lem:prodfund} Point~3.  In the second case,
  $\Amc_{q^{\bot}}$ is \Omc-saturated since it
  contains $A(x_0)$ for all concept names $A \in \Sigma$.
  For the induction step, suppose
  $\Bmc_{i+1},\Omc \models A(a)$ for some concept name $A$ and some
  $a \in \mn{ind}(\Bmc_{i+1})$. By monotonicity, $\Bmc'_{i},\Omc \models A(a)$
  where $\Bmc'_i$ is the result of applying the unraveling step to
  $\Bmc_i$.  By Point~2 of Lemma~\ref{lem:univers},
  $A(a) \in \Umc_{\Bmc'_i,\Omc}$. We distinguish cases:
  \begin{itemize}

    \item If $a \in \mn{ind}(\Bmc_i)$, then $a$ was not affected by
      the unraveling and $\Bmc_{i}',a \preceq \Bmc_i,a$ by
      Lemma~\ref{lem:localunraveling}, 
      thus $\Umc_{\Bmc_{i}',\Omc},a \preceq \Umc_{\Bmc_i,\Omc},a$ by
      Lemma~\ref{lem:universal-homomorph}. Hence $A(a)\in
      \Umc_{\Bmc_i,\Omc}$ and $\Bmc_i,\Omc\models A(a)$ by
      Point~2 of Lemma~\ref{lem:univers}. Induction yields $A(a)\in
      \Bmc_i$ and thus $A(a)\in \Bmc_{i+1}$. 

  \item If $a \notin \mn{ind}(\Bmc_i)$, then $a$ is the copy of some
    element $b \in
    \mn{ind}(\Bmc_i)$.
       By 
      Lemma~\ref{lem:localunraveling},
      $\Bmc_{i}',a \preceq \Bmc_i, b$ and 
     thus
      $\Umc_{\Bmc_{i}',\Omc},a \preceq \Umc_{\Bmc_i,\Omc}, b$ by
      Lemma~\ref{lem:universal-homomorph}. Hence $A(b) \in
      \Umc_{\Bmc_i,\Omc}$ and $\Bmc_i,\Omc\models A(b)$ by
      Point~2 of Lemma~\ref{lem:univers}. Induction yields $A(b)\in
      \Bmc_i$ and the definition of unraveling implies
      $A(a)\in \Bmc_{i+1}$.
  \end{itemize}

  \smallskip For Point~2, let $h$ be a 
  homomorphism from $q_T$ to $\Umc_{\Bmc_i,\Omc}$ with
  $h(\bar x) = \bar a_1 \otimes \bar a_2$ for some $i \geq 2$.
  Assume to
  the contrary of what is to be shown that there is an $a \in
  \mn{ind}(\Bmc_i)$ that is not in $\mn{img}(h^*)$.
  Let $\Bmc'$ be the result of removing from $\Bmc_i$ all assertions
  that involve $a$.
  We
  show that
  \begin{itemize}
  \item[($*$)]
  $h$ is a homomorphism from $q_T$ to
  $\Umc_{\Bmc',\Omc}$
  \end{itemize}
  and thus witnesses that
  $\Bmc',\Omc \models q_T(\bar a_1 \otimes \bar a_2)$. Hence,
  $a$ is dropped during the Minimize step, in contradiction to $a \in
  \mn{ind}(\Bmc_i)$.

  To see that $(*)$ holds, first note that for all $b,b' \in
  \mn{ind}(\Bmc_i) \setminus \{ a \}$, the following holds by
  Point~1 and construction of universal models:
  \begin{enumerate}

  \item $A(b) \in \Umc_{\Bmc_i,\Omc}$ iff $A(b) \in \Umc_{\Bmc',\Omc}$;

  \item $r(b,b') \in \Umc_{\Bmc_i,\Omc}$ iff $r(b, b') \in \Umc_{\Bmc',\Omc}$.

  \end{enumerate}
  From Point~1, in turn, it follows that the subtree in
  $\Umc_{\Bmc_i,\Omc}$ below each
  $b \in \mn{ind}(\Bmc_i) \setminus \{ a \}$ is identical to the
  subtree in $\Umc_{\Bmc',\Omc}$ below $b$. In summary,
  ($*$) follows.
  
  \smallskip
  Now for Point~3, we define a mapping $h$ from $\Bmc_{i + 1}$ to $\Bmc_i$ with
  $h(\bar x) = \bar x$.  Recall that $\Bmc_{i + 1}$ is constructed from
  $\Bmc_i$ by applying the Unravel and Minimize steps.  Let $h(s) = a$ for all
  $a \in \mn{ind}(\Bmc_i) \cap \mn{ind}(\Bmc_{i + 1})$ and $h(a') = a$
  for all copies $a'$ of $a$ with $a' \in \mn{ind}(\Bmc_{i + 1})$.  It
  follows from the definition of the unravelling step, that $h$ is a
  homomorphism as required.
\end{proof}
It is proved as part of Lemma~\ref{lem:termtwo} below that the
Unravel/Minimize phase terminates after polynomially many steps, let
$\Bmc_n$ be the result.  We next construct a sequence
$\Bmc_{n},\Bmc_{n+1},\dots$ using the modified version of the original
\mn{refine} subroutine, that is, using the Expand and Minimize steps.
With $\Bmc_i'$, $i \geq n$, we denote the result of only applying the
Expand step to $\Bmc_i$, but not the Minimize step.
\begin{lemma} \label{lem:expansioncorrectthree}
  For all $i \geq n$, 
  \begin{enumerate}

    \item $\Bmc_i,\Omc\models q_T(\bar a_1 \otimes \bar a_2)$;

    \item $\Bmc_i$ is \Omc-saturated;

    \item if $h$ is a homomorphism from $q_T$ to $\Umc_{\Bmc_i,
      \Omc}$ with $h(\bar y) = \bar a_1 \otimes \bar a_2$, then $\mn{ind}(\Bmc_i)\subseteq
      \mn{img}(h^*)$;

    \item $\Bmc_{i+1}, \bar a_1 \otimes \bar a_2\to \Bmc_i, \bar a_1 \otimes \bar a_2$.

   \item $|\mn{ind}(\Bmc_{i + 1})| > |\mn{ind}(\Bmc_{i})|$.

  \end{enumerate}
\end{lemma}
\begin{proof}[Proof sketch.]
  Point~1 is a direct consequence of Lemma~\ref{lem:expansioncorrect}.
  Points~2 to~4 can be proved in the same way as Points~1 to~4 of
  Lemma~\ref{lem:expansioncorrecttwo}. While the proofs of Points~2
  and~5 go through without modification, a slight extension is
  required for the proof of Point~3 in the case that $i>1$. There, we
  start with a homomorphism $h$ from $q_T$ to $\Umc_{\Bmc_i, \Omc}$
  with $h(\bar x) = \bar b_i$, and suppose that there is an
  $a \in \mn{ind}(\Bmc_i)$ that is not in $\mn{img}(h^*)$. If $a$ is
  not reachable from some individual in
  $\mn{ind}(\Amc_1) \times \mn{ind}(\Amc_2)$ in $\Bmc_i$ viewed as a
  directed graph, then we can argue as in the proof of Point~2 of
  Lemma~\ref{lem:expansioncorrecttwo}, that is, obtain a contradiction
  against exhaustive application of Minimize to $\Bmc_{i-1}$. If $a$
  is reachable, however, this does not work as we do not apply the
  Minimize step to such individuals in the modified version of the
  \mn{refine} subroutine.

  However, by Point~4, there is a homomorphism $h$ from $\Bmc_i$ to
  $\Bmc_1$ with
  $h(\bar a_1 \otimes \bar a_2)=\bar a_1 \otimes \bar a_2$.  By
  Point~1 there is a homomorphism $h_1$ from $q_T$ to
  $\Umc_{\Bmc_{i},\Omc}$ with $h_1(\bar x)=\bar a_1 \otimes \bar a_2$. Let
  $h_2$ be the extension of $h$ to a homomorphism from
  $\Umc_{\Bmc_{i},\Omc}$ to $\Umc_{\Bmc_1,\Omc}$ as in
  Lemma~\ref{lem:universal-homomorph} Point~1. Then
  $\mn{img}(h_2^*) = \mn{img}(h)$.  Composing $h_1$ and $h_2$ yields a
  homomorphism $h_3$ from $q_T$ to $\Umc_{\Bmc_i, \Omc}$ with
  $h_3(\bar x) = \bar a_1 \otimes \bar a_2$, but with
  $\mn{ind}(\Bmc_i) \not\subseteq \mn{img}(h_3^*)$, in contradiction
  to Point~2 of Lemma~\ref{lem:minimization-works}.
\end{proof}


%
\begin{lemma} \label{lem:termtwo}
  $\mn{refine}(q(\bar x))$ can be computed in time polynomial in
  $||q_T|| + ||q||$ using membership queries.
\end{lemma}
\begin{proof} We first note that the length of the sequence $\Bmc_1,
  \Bmc_2, \ldots$ computed in the Unravel/Minimize phase is bounded by
  $|\mn{var}(q_T)| + 1$. Indeed, the following is easy to prove by induction
  on $i$.

  \smallskip 
  \noindent 
  \textit{Claim}.  Let $i \geq 1$. Then every individual in $\Bmc_i$
  that is reachable in the directed graph
  $G_{\Bmc_{i}} = (\mn{ind}(\Bmc_{i}),\{ (a,b) \mid r(a,b) \in
  \Bmc_{i} \})$ from some individual in
  $\mn{ind}(\Amc_1) \times \mn{ind}(\Amc_2)$ on a path of length at
  most $i-1$ is unraveled.

  \smallskip 
  \noindent 
  Since every individual in $\Bmc_i$ that is reachable from some
  individual in $\mn{ind}(\Amc_1) \times \mn{ind}(\Amc_2)$ is
  reachable on a path of length at most $|\mn{ind}(\Bmc_i)|$ and
  $|\mn{ind}(\Bmc_i)| \leq |\mn{var}(q_T)|$ by Point~2 of
  Lemma~\ref{lem:minimization-works}, it follows that the
  Unravel step is thus no longer applicable to $\Bmc_{m + 2}$ for
  $m = |\mn{var}(q_T)|$.
  
  Next observe that the length of the sequence $\Bmc_n,\Bmc_{n+1}$
  computed in the Expand/Minimize phase is also bounded by
  $|\mn{var}(q_T)|$ as we have for all $i\geq n$, $|\mn{ind}(\Bmc_i)|
  \leq |\mn{var}(q_T)|$, by Lemma~\ref{lem:expansioncorrectthree}
  Point~3, and $|\mn{ind}(\Bmc_i)|<|\mn{ind}(\Bmc_{i+1})|$, by
  Lemma~\ref{lem:expansioncorrectthree} Point~5.

  It remains to show that every step runs in polynomial time. First
  note that, for all $i\geq 1$, we have $|\mn{ind}(\Bmc_i)| \leq
  |\mn{var}(q_T)|$ by Lemma~\ref{lem:minimization-works} Point~2 and
  Lemma~\ref{lem:expansioncorrectthree} Point~3. Moreover, by
  definition of the Unravel/Minimize steps, we have
  $\mn{sig}(\Bmc_i)\subseteq \Omega$, for all $i$, where
  $\Omega=\mn{sig}(q)$ and thus $|\Omega|\leq ||q||$.  Applying the
  Unravel step to $\Bmc_i$ thus takes time polynomial in
  $|\mn{ind}(\Bmc_i)|$ and $|\Omega|$.  The resulting ABox $\Bmc'$ is
  such that $|\mn{ind}(\Bmc')| \leq
  |\mn{ind}(\Bmc_i)|^2\cdot|\Omega|$. The number of membership queries
  needed in the minimization step is thus bounded by
  $|\mn{var}(q_T)|^2\cdot ||q||$.

  For Expand, note that chordless cycles of length $n>n_{\text{max}}$
  can be identified in time polynomial in~$|\mn{ind}(\Bmc_i)|\leq
  |\mn{var}(q_T)|$ and constructing the ABox $\Bmc_{i+1}$ from
  $\Bmc_i$ is clearly also possible in polynomial time.
\end{proof}

\begin{lemma} \label{lem:unravel-refine-progress}
  Let \Omc be an \ELdr-ontology. 
  Let $\Amc_1$ and $\Amc_2$ be ABoxes and $\bar a_i$, $i\in\{1,2\}$, be
  tuples of individuals from $\Amc_i$ of the same length. Moreover, let
  $q(\bar z)$ be $\Cmc^3_{\Amc_1,\Omc}\times \Cmc^3_{\Amc_2,\Omc}$
  viewed as CQ with answer variables $\bar z=\bar a_1\otimes \bar a_2$ and 
  let $p(\bar x)$ be the result of $\mn{refine}(q(\bar z))$ with
  respect to some target query $q_T(\bar y)$.
  Then there is a homomorphism $h_i$ from $p(\bar x)$ to 
  $\Umc_{\Amc_i,\Omc}$ with $h_i(\bar x)=(\bar a_i)$, for $i\in\{1,2\}$.
\end{lemma}
\begin{proof}
  Let $\Bmc_m$ be the result of $\mn{refine}(q(\bar z))$ before it is
  turned into the CQ $p(\bar x)$. Further, let $\Bmc'$ denote the
  restriction of $\Bmc_m$ to all individuals that are reachable from
  an individual in $\mn{ind}(\Amc_1) \times \mn{ind}(\Amc_2)$ in $\Bmc_m$, and
  let $\Bmc''$ be the restriction of $\Bmc_m$ to all individuals that
  are not reachable. Thus $\Bmc_m  = \Bmc' \uplus \Bmc''$.

  It suffices to show that, for $i\in\{1,2\}$, there is a homomorphism
  $h'_i$ from $\Bmc'$ to $\Umc_{\Amc_i,\Omc}$ with
  $h'_i(\bar x)=(\bar a_i)$ and a homomorphism $h''_i$ from $\Bmc''$ to
  $\Umc_{\Amc_i,\Omc}$.

  For the former, note that all individuals in $\Bmc'$ are unraveled.
  Thus
  $\Bmc' \subseteq \Umc_{\Amc_1, \Omc} \times \Umc_{\Amc_2, \Omc}$ and
  the identity is a homomorphism $h'$ from $\Bmc$ to
  $\Umc_{\Amc_1, \Omc} \times \Umc_{\Amc_2, \Omc}$ with
  $h'(\bar b) = \bar a_1 \times \bar a_2$.  Projection to the left and
  right components yields the homomorphisms $h'_i$ as required.

  For the latter, note that none of the individuals in $\Bmc''$ is
  unraveled. In fact, this follows from two obvious properties of
  the Unravel step and 3-compact canonical models:
  \begin{itemize}

  \item if the Unravel step is not applicable to an ABox $\Bmc_i$ and
    $r(a,b) \in \Bmc_i$ with $a$ unraveled, then $b$ is unraveled too;

  \item if Unraveling (and Minimization) is repeatedly applied to
    an ABox $\Cmc^3_{\Amc_1,\Omc}\times \Cmc^3_{\Amc_2,\Omc}$,
    then it never produces any fact $r(a,b)$ with $b$ unraveled, but
    $a$ not unraveled (because there are no $r$-edges from individuals
    of the form $c_{a,i,s,C}$ to individuals from $\mn{ind}(\Amc_i)$
    in $\Cmc^3_{\Amc_i,\Omc}$.

  \end{itemize}
  Thus, the identity is a homomorphism from $\Bmc''$ viewed as a
  Boolean CQ to $\Cmc^3_{\Amc_1,\Omc}\times \Cmc^3_{\Amc_2,\Omc}$.
  Projection to the left and right components yields a homomorphism
  $g''_i$ from $\Bmc''$ to $\Cmc^3_{\Amc_i,\Omc}$ for $i \in \{1,2\}$.
  By definition of the Expansion and Minimize step and its use in
  \mn{refine}, it is clear that $\Bmc''$ is chordal.  From
  Lemma~\ref{lem:noanocycles}, it thus follows that $\Bmc''$ viewed as
  a CQ is an ELIQ. Using the construction of $\Cmc^3_{\Amc_i,\Omc}$
  and $\Umc_{\Amc_i,\Omc}$, it is now straightforward to convert the
  homomorphism $g''_i$ from ELIQ $\Bmc''$ to $\Cmc^3_{\Amc_i,\Omc}$
  into the desired homomorphism $h''_i$ from $\Bmc''$ to
  $\Umc_{\Amc_i,\Omc}$.
\end{proof}

\section{Proof of Theorem~\ref{thm:mainthree}}

Let $q_0(\bar x_0), q_1(\bar x_1), \ldots$ be the sequence of hypotheses generated by the
algorithm. 

\begin{lemma}
  \label{lem:extraction}
    For all $i \geq 0$:     %
    \begin{enumerate}
        \item $q_i \subseteq_\Omc q_T$;
        \item $q_i \subseteq_\Omc q_{i + 1}$;
        \item $q_{i + 1} \not \subseteq_\Omc q_i$.
    \end{enumerate}
\end{lemma}
\begin{proof} 
  Point~1 is a consequence of
  Lemma~\ref{lem:expansioncorrect}~Point~1, for the first
  \mn{refine}-operation, and
  Lemma~\ref{lem:expansioncorrectthree} Point~1, for the second \mn{refine}-operation


  For Point~2, recall that $q_{i + 1} = \mn{refine}(q_H'(\bar x))$
  where $q_H'(\bar x)$ is $\Cmc^3_{\Amc_{q_i},\Omc}\times
  \Cmc^3_{\Amc,\Omc}$, for some positive counterexample $\Amc,\bar a$,
  viewed as CQ with answer variables $\bar x=\bar x_i\otimes \bar a$.
  In case of the first $\mn{refine}$-operation there is a homomorphism $h$ from $q_{i+1}$
  to $\Umc_{\Amc_{q_i},\Omc}$ with $h(\bar x_{i+1})=x_i$, by Lemma~\ref{lem:progress}.
  Lemmas~\ref{lem:refinement-invariant} and~\ref{lem:expansioncorrectthree} Point~4 give us this
  homomorphism for the second \mn{refine}-operation.
  By Lemma~\ref{lem:homlem}, we obtain $q_{i}\subseteq_\Omc q_{i+1}$.


  For Point~3, assume to the contrary that $q_{i+1} \subseteq_\Omc
  q_i$. Then there is a homomorphism $h_1$ from $q_i$ to
  $\Umc_{\Amc_{q_{i+1}},\Omc}$ with $h_1(\bar x_i)=\bar
  x_{i+1}$. Recall once more
that $q_{i + 1} = \mn{refine}(q_H(\bar x))$
  where $q_H(\bar x)$ is $\Cmc^3_{\Amc_{q_i},\Omc}\times
  \Cmc^3_{\Amc,\Omc}$, for some positive counterexample $\Amc,\bar a$,
  viewed as CQ with answer variables $\bar x=\bar x_i\otimes \bar a$.
  In case of the first $\mn{refine}$-operation there is a homomorphism $h_2$ from $q_{i+1}$
  to $\Umc_{\Amc,\Omc}$ with $h_2(\bar x_{i+1})=\bar a$, by Lemma~\ref{lem:progress}.
  Again, Lemmas~\ref{lem:refinement-invariant} and~\ref{lem:expansioncorrectthree} show
  that $h_2$ also exists for the second\mn{refine}-operation.
  By Lemma~\ref{lem:universal-homomorph} Point~1 $h_2$ can be extended to a homomorphism $h_2'$ from 
  $\Umc_{\Amc_{q_{i + 1}},\Omc}$ to $\Umc_{\Amc, \Omc}$ with $h_2'(\bar x_{i+1})=\bar a$.
  Composing $h_1$ and $h_2$ yields a homomorphism $h$ from $q_i$ to
  $\Umc_{\Amc, \Omc}$ with $h(\bar x) = \bar a$. Thus $\Amc, \Omc \models
  q_i(\bar a)$, in contradiction to $\Amc, \bar a$ being a positive counterexample.
\end{proof}

We next observe that the sizes of $q_0, q_1, \ldots$ are
non-decreasing.  
%
\begin{lemma}\label{lem:hyposizes}
  For all $i \geq 0$:
  \begin{enumerate}

    \item $\mn{var}(q_i) \subseteq \mn{img}(h^*)$ for every
      homomorphism $h$ from $q_{i + 1}$ to $\Umc_{\Amc_{q_{i}}, \Omc}$ with
      $h(\bar x_{i+1})=\bar x_i$;

 \item $|\mn{var}(q_i)| \leq |\mn{var}(q_{i + 1})|$;


\end{enumerate}
\end{lemma} \begin{proof} For Point~1, let $h$ be a homomorphism from
  $q_{i + 1}$ to $\Umc_{\Amc_{q_{i}}, \Omc}$ with $h(\bar x_{i+1})=\bar x_i$.
  By Lemma~\ref{lem:universal-homomorph} Point~1, we can extend $h$ to a
  homomorphism from $\Umc_{\Amc_{q_{i+1}},\Omc}$ to
  $\Umc_{\Amc_{q_i},\Omc}$ without
  adding individuals from $\mn{var}(q_i)$ to $\mn{img}(h^*)$.
  By Point~1 of
  Lemma~\ref{lem:extraction}, there is a homomorphism $h'$ from $q_T$
  to $\Umc_{\Amc_{q_{i+1}},\Omc}$ with $h(\bar x)=\bar x_{i+1}$. We
  can compose $h'$ and $h$ into a homomorphism $g$ from $q_T$ to
  $\Umc_{\Amc_{q_i},\Omc}$ with $g(\bar x)=\bar x_i$.  We then obtain
  $\mn{var}(q_i) \subseteq \mn{img}(g^*)$ by
  Lemma~\ref{lem:expansioncorrecttwo} Point~2 for the first
  \mn{refine}-operation or Lemma~\ref{lem:expansioncorrectthree} for the second
  version.
  Since $\mn{img}(g^*) \subseteq
  \mn{img}(h^*)$, it follows that $\mn{var}(q_i)
  \subseteq \mn{img}(h^*)$.

  \smallskip 
  Point~2 is a consequence of Point~1. In fact,
  $q_i\subseteq_\Omc q_{i+1}$ implies via
  Lemma~\ref{lem:homlem} that there is a
  homomorphism $h$ from $q_{i+1}$ to
  $\Umc_{\Amc_{q_i},\Omc}$ with $h(\bar x_{i+1})=\bar x_i$.
  Point~1 yields $\mn{var}(q_i) \subseteq \mn{img}(h^*)$ and thus
  $|\mn{var}(q_i)| \leq |\mn{var}(q_{i+1})|$.
  %
%
\end{proof}

\begin{lemma}\label{lem:iterationbound}
  $q_i \equiv_\Omc q_T$ for some 
  $i \leq p(|\mn{var}(q_T)|+|\Sigma|)$ for some polynomial
  $p$.
\end{lemma}

\begin{proof}
By Lemma~\ref{lem:expansioncorrecttwo} Point~2 in case of the first
\mn{refine}-operation or by Lemma~\ref{lem:expansioncorrectthree} Point~3 in case of
the second,
  we have $|\mn{var}(q_i)|\leq |\mn{var}(q_T)|$ for all $i\geq 0$. 

  Let $q_\ell,\dots,q_{u}$, $\ell\leq u$, be a subsequence of
  $q_1,q_2,\ldots$ such that $|\mn{var}(q_\ell)|= \cdots
  =|\mn{var}(q_{u})|$. By Point~2 of Lemma~\ref{lem:hyposizes}, it
  suffices to show that the length $u-\ell$ of any such sequence is
  bounded by a polynomial in $|\mn{var}(q_T)|$ and $|\Sigma|$.  Let
  $h_i$, for $i\in \{\ell+1,\ldots,u\}$, be the homomorphisms from $q_{i}$
  to $\Umc_{\Amc_{q_{i-1}},\Omc}$ that exist due to
  Lemma~\ref{lem:extraction} Point~2. 
  
  Note that $h^*_i$ is a bijection between $\mn{var}(q_i)$ and
  $\mn{var}(q_{i-1})$.  Denote with $V_i$ the set of all quantified
  variables $x\in \mn{var}(q_{i})$ which do not occur in a role
  atom, and define $U_i=\mn{var}(q_i)\setminus V_i$. Let us further
  denote with $q^x$ the restriction $q|_{\{x\}}$ of a query $q$ to a single
  variable $x\in \mn{var}(q)$. Clearly, $q_i$
  can be written as 
  \[q_i(\bar x_i)\leftarrow q_{i}|_{U_i}\wedge\bigwedge_{x\in V_i} q_i^x\]
  Notice that, by definition, each $q_i^x$, $x\in V_i$ is a
  query without answer variables. 
 

  
  \smallskip\noindent\textit{Claim~1.} $x\in U_i$ implies
  $h_{i}^*(x)\in U_{i-1}$.

  \smallskip\noindent\textit{Proof of Claim~1.} Let $x\in U_{i}$.
  If $x$ is an answer variable then $h_{i}(x)=h^*_{i}(x)$ is an
  answer variable and thus in $U_{i-1}$. Suppose now that there is a role
  atom $R(z,x)$ in $q_i$ and consider the assertion
  $R(h_{i}(z),h_{i}(x))\in \Umc_{\Amc_{q_{i-1}},\Omc}$ which
  exists since $h_{i}$ is a homomorphism.
  Since $h^*_{i}$ is a bijection, $z=x$ or
  $h^*_{i}(z)\neq h^*_{i}(x)$. 
  \begin{itemize}

    \item In the first case, we obtain $R(y,y)\in
      \Umc_{\Amc_{q_{i-1}},\Omc}$ for $y=h_{i}(z)=h_{i}(x)$. The
      definition of the universal models yields
      $y\in\mn{var}(q_{i-1})$, $R(y,y)$ occurs in $q_{i-1}$, and
      $h_{i}^*(z)=h_{i}(z)\in U_{i-1}$.

    \item In the second case, we obtain $h^*_{i}(z)=h_{i}(z)\in
      \mn{var}(q_{i-1})$, $h^*_{i}(x)=h_{i}(x)\in \mn{var}(q_i)$,
      and $R(h_{i}(z),h_{i}(x))$ occurs in $q_{i-1}$. Hence, $x\in
      U_{i-1}$.

  \end{itemize}
  This finishes the proof of Claim~1.

  \medskip Claim~1 implies $|U_i|\leq |U_{i-1}|$ for every
  $i\in \{\ell+1,\ldots,u\}$. We consider now subsequences
  $q_{\ell'},\ldots,q_{u'}$ of $q_\ell,\ldots,q_u$ with
  $|U_{\ell'}|=\dots=|U_{u'}|$ and thus $|V_{\ell'}|=\dots=|V_{u'}|$.
  Since $|U_i|\leq |\mn{var}(q_T)|$, for all $i$, it suffices to show
  that the length of such a sequence is bounded by a polynomial in
  $|\mn{var}(q_T)|$ and $|\Sigma|$.

  \smallskip\noindent\textit{Claim~2.} For every $i\in
  \{\ell'+1,\ldots,u'\}$,
  \begin{enumerate}

    \item $h_{i}^*$ is a bijection between $V_i$ and $V_{i-1}$, and

    \item $h_{i}$ is a bijection between $U_i$ and $U_{i-1}$.

  \end{enumerate}

  \smallskip\noindent\textit{Proof of Claim~2.} The first point is a
  consequence of Claim~1 and the facts that $|V_{i-1}|=|V_{i}|$
  and $h^*_{i}$ is a bijection between $\mn{var}(q_i)$ and
  $\mn{var}(q_{i-1})$. For
  the second point, suppose that some $x\in U_{i-1}$ is not
  $h_{i}(y)$ for some $y\in U_i$. Thus, there is some $y$ such that
  $h^*_{i}(y)=x$ and $h_{i}(y)$ is strictly in the subtree rooted
  at $x$ in $\Umc_{\Amc_{q_{i-1}},\Omc}$. Since $h_{i}^*$ is a
  bijection, $y$ is the unique variable such that $h_{i}(y)$ is in
  the subtree rooted at $x$. We claim that $y\in V_i$. This implies
  $x\in V_{i-1}$, by the first point, and contradicts the assumption
  $x\in U_{i-1}$. To see $y\in V_i$, suppose the contrary, that is,
  $y\in U_i$. If $y$ is an answer variable, then $h_{i}(y)$ is an
  answer variable and thus in $U_i$, contradiction. Otherwise, there
  is some $R(y,z)$ in $q_i$.  Since $h_{i}(y)$ is \emph{strictly}
  below $x$, this leads to a contradiction as follows. If $z\neq y$,
  then $h^*_{i}(z)=h^*_{i}(y)$ contradicts the fact that
  $h_{i}^*$ is a bijection. If, on the other hand, $z=y$, then
  $h_{i}$ is not a homomorphism since there is no self-loop
  $R(h_{i}(y),h_{i}(y))$ in $\Umc_{\Amc_{q_{i-1}},\Omc}$, by
  definition of the universal model. This finishes the proof of
  Claim~2.

  \medskip Sanctioned by the first point in Claim~2, in what follows
  we assume for the sake of readability that $h_{i}^*(x)=x$ for all
  $x \in V_i$ and $i\in \{\ell'+1,\ldots,u'\}$. Hence,
  $V_{\ell'}=\dots=V_{u'}$.  Now, observe that, for all
  $i\in \{\ell'+1,\ldots,u'\}$, one of the following is the case:
  \begin{enumerate}

    \item the inverse of $h_{i}$ is not a homomorphism from 
      $q_{i-1}|_{U_{i-1}}$ to $q_{i}|_{U_{i}}$;

    \item there is some $x\in V_i$ such that
      $q_i^x\not\subseteq_\Omc q_{i-1}^x$.

  \end{enumerate}
  Indeed, if neither Point~1 nor Point~2 is satisfied then
  $q_i\equiv_\Omc q_{i-1}$, in contradiction to Point~3 of
  Lemma~\ref{lem:extraction}. It thus remains to bound the number of
  times each of these points can be satisfied along
  $q_{\ell'},\ldots,q_{u'}$. We start with Point~1.

  \smallskip\noindent\textit{Claim~3.} The number of
  $i\in\{\ell'+1,\ldots,u'\}$ such that the inverse of $h_{i}$ is not
  a homomorphism from $q_{i-1}|_{U_{i-1}}$ to $q_{i}|_{U_{i}}$ is at
  most $(|\mn{var}(q_T)|^2+|\mn{var}(q_T)|) \cdot |\Sigma|$.

  \smallskip\noindent\textit{Proof of Claim~3}. Let $i$ be as in the
  claim. By Point~2 of Claim~2, $h_{i}$ is a bijective
  homomorphism from $q_{i}|_{U_i}$ to $q_{i-1}|_{U_{i-1}}$.  Hence,
  the number $n_i$ of atoms in $q_{i}|_{U_i}$ is at most the number
  $n_{i-1}$ of atoms in $q_{i-1}|_{U_{i-1}}$.  As the inverse of
  $h_{i}$ is not a homomorphism, we have $n_i<n_{i-1}$. Since the
  maximal number of atoms in $q_i$ is bounded by $|\Sigma|\cdot
  |\mn{var}(q_T)|^2 + |\Sigma|\cdot |\mn{var}(q_T)|$, the claim
  follows. 

  \smallskip Now for Point~2.

  \smallskip\noindent\textit{Claim~4.} Let $x\in V_{\ell'}$. The
  number of $i\in\{\ell'+1,\ldots,u'\}$ such that
  $q_i^x\not\subseteq_\Omc q_{i-1}^x$ is at most $|\Sigma|^2$.

  \smallskip\noindent\textit{Proof of Claim~4.} Let $I$ be the set of
  all $i$ as in the claim. Let $i\in I$.  We distinguish two cases.
  \begin{enumerate}[label=(\Alph*)]

    \item $h_{i}(x)=x$.

      Since $h_{i}$ is a homomorphism, the number $n_i$ of atoms in
      $q_i^x$ is at most the number $n_{i-1}$ of atoms in
      $q_{i-1}^x$. From
      $q_{i-1}^x\not\subseteq_\Omc q_i^x$, it follows that
      $n_i<n_{i-1}$.

    \item $h_{i}(x)\neq x$, that is, $h_{i}(x)$ is strictly in the
      subtree below $x$.

      By definition of the universal model and since $\Omc$ is in
      normal form, there is an atom $A(x)$ in $q_{i-1}^x$ such that
      there is a homomorphism from $q_{i}^x$ to
      $\Umc_{\{A(a)\},\Omc}$. We claim that $A(x)$ is not an atom in
      any query $q_j^x$ with $j>i$ and $j\in I$. Indeed, if $A(x)$
      occurs in $q_j^x$ for such $j$, then $q_{j}^x\subseteq_\Omc
      q_i^x$. The homomorphisms $h_{i+1},\ldots,h_{j}$ witness that $q_i^x
      \subseteq_\Omc q_{i+1}^x \subseteq_\Omc \cdots \subseteq_\Omc
      q_j^x$, and thus all these queries are actually equivalent, in
      contradiction to the definition of $I$.

  \end{enumerate}
  Next observe that, by what was said in Point~(A), Point~(A) can
  happen only $|\Sigma|$ times without Point~(B) happening in
  between. Moreover, Point~(B) can happen only $|\Sigma|$ times
  overall. We thus have that the size of $I$ is bounded by
  $|\Sigma|^2$. This finishes the proof of Claim~4.

  \medskip Since $|V_{\ell'}|\leq |\mn{var}(q_T)|$, we obtain that the
  length of the sequence $q_{\ell'},\ldots,q_{u'}$ is bounded by 
  \begin{align*}
    \big((|\mn{var}(q_T)|^2+|\mn{var}(q_T)|) \cdot |\Sigma|\big) +
    |\mn{var}(q_T)|\cdot |\Sigma|^2.
  \end{align*}
  where the first summand accomodates the number of $i$ where Point~1
  is satisfied and the second summand accomodates the number of $i$
  where Point~2 is satisfied. 
\end{proof}

\section{Proofs for Section~\ref{sect:lower}}

\thmellower*
 \begin{proof}\ Assume to the contrary of what is to be shown that $\EL$-concepts
  are polynomial query learnable under \ELI-ontologies when
  unrestricted CQs
  can be used in equivalence queries. Then there exists a
  learning algorithm and a polynomial $p$ such that at any time, the
  sum of the sizes of the inputs to membership and equivalence queries
  made so far is bounded by $p(n_1,n_2,n_3)$, where $n_1$ is the size
  of $C_T$, $n_2$ is the size of \Omc, and $n_3$ is the size of the
  largest counterexample seen so far.

  We choose $n$ such that $2^n>p(q_1(n),q_2(n),q_3(n))+1$ where $q_1, q_2,q_3$ are
  polynomials such that for every $n\geq 1$, $q_1(n)\geq ||H||$ for all
  $H \in \Hmc_n$, $q_2(n)\geq
  ||\Omc_n||$, and $q_3(n)$ bounds from above the size of all
  counterexamples returned by the oracle that we craft below.

  Consider now $\Omc_n$ and $\Hmc_n$ as defined in
  Section~\ref{sect:lower}.  We let the oracle maintain a set of
  hypotheses \Hmc, starting with $\Hmc=\Hmc_n$ and then proceeding to
  subsets thereof, in such a way that at any point in time the learner
  cannot distinguish between any of the candidate targets in \Hmc.

  More precisely, consider a membership query with ABox \Amc and
  individual~$a_0$. 
  The oracle responds as
  follows:
  \begin{enumerate}

  \item if $\Amc,\Omc_n \models L_0(a_0)$, then answer yes;

  \item if $\Amc,\Omc_n\models K_0(a_0)$ and there are
    $\sigma_1,\dots,\sigma_n \in \{r,s\}$ with $\Amc,\Omc_n \models
    W^{\sigma_i}_i(a_0)$ for $1 \leq i \leq n$, then answer yes and remove
    $\exists \sigma_1 \cdots \exists \sigma_n . \exists r^n . A$ from $\Hmc$;
    
  \item otherwise, answer no and remove all $H$ with $\Amc,\Omc_n\models H(a_0)$ from
    $\Hmc$.

  \end{enumerate}
  Higher up rules have higher priority, e.g., Case~2 is applied only
  if Case~1 does not apply. It is not hard to verify that the answers
  are correct regarding the hypothesis set $\Hmc$ that remains after
  the answer is given.

  Now consider an equivalence query with CQ $q_H(x_0)$. The oracle
  responds as follows:
\begin{enumerate}


\item if $\{L_0(a_0)\},\Omc_n \not\models q_H(a_0)$, then return
  $\{ L_0(a_0) \}$ as positive counterexample;

\item if $\{ \top (a_0), L_0(a_1)\} \models q_H(a_0)$, then return
  $\{ \top (a_0), L_0(a_1)\}$ as negative counterexample;
  
\item if there are $\sigma_1,\dots,\sigma_n \in \{r,s\}$
  such that
  $$
  \{  K_0(a_0),W^{\sigma_1}_1(a_0),\dots,W^{\sigma_n}_n(a_0) \}, \Omc_n \not\models q_H(a_0),$$
  then choose such $\sigma_1,\dots,\sigma_n$, return
  $$\{
  K_0(a_0),W^{\sigma_1}_1(a_0),\dots,W^{\sigma_n}_n(a_0) \}
  $$
  as positive
  counterexample, and remove $\exists \sigma_1 \cdots \exists \sigma_n
  . \exists r^n . A$ from $\Hmc$ (if present).

  

  
  

  
\end{enumerate}
%
Again, higher up rules have higher priority and the answers are always
correct with respect to the updated set~$\Hmc$.  For Case~3, we 
remark that the counterexample
$\Amc=\{ K_0(a_0),W^{\sigma_1}_1(a_0),\dots,W^{\sigma_n}_n(a_0) \}$ is such
that
$\Amc,\Omc_n \sqsubseteq \exists \sigma'_1 \cdots \exists \sigma'_n
. \exists r^n . A(a_0)$ for all $\sigma'_1 \cdots \sigma'_n \in \{r,s\}^n$
except $\sigma'_1 \cdots \sigma'_n=\sigma_1 \cdots \sigma_n$.

We argue that the
cases are exhaustive. Assume that Cases~1 and~2 do not apply.
Let 
$q'_H$ be the restriction of $q_H$ to variables that are reachable 
from $x_0$ in the undirected graph
$$G_{q_H} = (\mn{var}(q_H),\{ \{ x,y \} \mid t(x,y) \in q_H \})$$
and let $q''_H$ be the restriction of $q_H$ to 
the variables that are not reachable.

Non-applicability of Case~1 implies that there is a homomorphism $h$
from $q_H$ to $\Umc_{\{L_0(a_0)\},\Omc_n}$ with
$h(x_0)=a_0$. Consequently, $q_H$ can only contain the symbols $r$,
$s$, $A$ as well as the concept names $U_i^\sigma$. Moreover, if there
is an atom $A(x)$ in $q'_H$, then $x$ is reachable from $x_0$ in
$G_{q_H}$ on an $r/s$-path of length exactly~$2n$ whose last $n$
components are all $r$.

We next show that $q'_H$ contains an atom $A(x)$. If it does not, in
fact, then $h$ is also a homomorphism from $q'_H$ to
$\Umc_{\{\top(a_0)\},\Omc_n}$ due to the first CI in $\Omc_n$, and  we
find a homomorphism $h'$ from $q''_H$ to
$\Umc_{\{L_0(a_1)\},\Omc_n}$. Combining $h$ and $h'$ yields a
homomorphism $g$ from $q_H$ to $\Umc_{\{\top(a_0),L_0(a_1)\},\Omc_n}$
  with $g(x_0)=a_0$, in contradiction to Case~2 not being applicable.

  So $q'_H$ contains an atom $A(x)$.  As argued above, $x$ is
  then reachable from $x_0$ along an $r/s$-path of length exactly
  $2n$
  whose last $n$ components are all $r$. Let the first $n$ components
  be $\sigma_1,\dots,\sigma_n$. Then
  $$
  \{  K_0(a_0),W^{\sigma_1}_1(a_0),\dots,W^{\sigma_n}_n(a_0) \},
  \Omc_n \not\models q_H(a_0)$$
  and thus Case~3 is applicable.

\smallskip

We next observe the following.

\smallskip \noindent\textit{Claim.} 
If \Amc and $a_0$ were given as a membership query and Cases~1 and~2
of membership queries do not apply, then $||\Amc|| \geq |\{ H \in
\Hmc_n \mid \Amc,\Omc_n \models H(a_0) \}|$.

\smallskip \noindent\textit{Proof of the claim.}
We may assume w.l.o.g.\ that \Amc is connected,
that is, the undirected graph
$G_\Amc = (\mn{ind}(\Amc),\{ \{ a,b \} \mid t(a,b) \in \Amc \})$ is
connected. 
Since Cases~1 and~2 of membership queries
do not apply,
\begin{itemize}
\item[(a)]
  $\Amc,\Omc_n \not\models L_0(a_0)$ and

\item[(b)] there are no
$\sigma_1,\dots,\sigma_n \in \{r,s\}$ such that $\Amc,\Omc_n \models
K_0(a_0)$ and $\Amc,\Omc_n \models W^{\sigma_i}_i(a_0)$ for $1 \leq i
\leq n$.

\end{itemize}
By the construction of $\Omc_n$, these properties imply the 
following for $0 \leq i \leq 2n$:
\begin{itemize}


\item[(c)] \Amc contains no $r/s$-path of length $i$ from $a_0$ to some $a$
  with $\Amc,\Omc_n \models L_i(a)$.

  In fact, the existence of such a path implies
  $\Amc,\Omc_n \models L_0(a_0)$.
  
\item[(d)] \Amc contains no $r/s$-path of length $i$ from $a_0$ to
  some $a$ with $\Amc,\Omc_n \models K_i(a)$ and assertions
  $W^{\sigma_1}_1(a_1),\dots,W^{\sigma_n}_n (a_n)$ where
  $\sigma_1,\dots,\sigma_n \in \{r,s\}$.

  In fact, the existence of such a path and such assertions implies
  $\Amc,\Omc_n \models K_0(a_0)$ and
  $\Amc,\Omc_n \models W^{\sigma_i}_i(a_0)$ for $1 \leq i \leq n$.

\item[(e)] \Amc contains no $r/s$-paths $p_1,p_2$ of length $i$ that
  end at the same individual and such that $p_1$ starts with an
  $r$-edge while $p_2$ starts with an $s$-edge.

  In fact, the existence of such paths and the connectedness of \Amc
  implies $\Amc,\Omc_n \models D(a_0)$, thus
  $\Amc,\Omc_n \models L_0(a_0)$.

\end{itemize}
We now sketch the construction of a model \Imc of \Amc and $\Omc_n$.
Let $\Wmc \subseteq \{0,\dots,n\}$ contain $i$ iff \Amc contains an
assertion $W^\sigma_i(a)$ for some $\sigma \in \{r,s\}$ and
$a \in \mn{ind}(\Amc)$. The following interpretations are used as
building blocks for \Imc:
\begin{itemize}

\item an \emph{$L_i$-path}, $n \leq i \leq 2n$, is an $r$-path of
  length $2n-i$ that makes makes $L_{i+j}$ true at the node at
  distance $j \in \{0,\dots,n-i\}$ from the start of the path and that
  makes true $A$ at the end of the path;

\item a \emph{$K_i$-path}, $n \leq i \leq 2n$, is defined as follows;
  let $\ell$ be maximal such that $\{i-n,\dots,\ell \} \subseteq \Wmc$;
  then a $K_i$-path is an $r$-path of length $\ell-(i-n)$ that makes
  $K_{i+j}$ true at the node at distance $j \in \{0,\dots,\ell\}$ from
  the start of the path and that makes true $A$ at the node at
  distance $2n-i$ (if it exists); in addition, the start of the path
  might make true any of the concept names $V^\sigma_j$,
  $\sigma \in \{r,s\}$ and $1 \leq j \leq n$, which
  are then all also made true by all other nodes on the path;
  
\item an \emph{$L_i$-tree}, $1 \leq i < n$, is a binary $r/s$-tree
  of depth $n-i$ that makes $L_{i+j}$ true at every node on level
  $j \in \{0,\dots,n-i\}$; in addition, every node on level $n-i$ is
  the start of an $L_n$-path;

\item a \emph{$K_i$-tree}, $1 \leq i < n$, is 
  a binary $r/s$-tree of depth~$n-i$ that makes true
  $K_{i+j}$ at every node on level $j \in \{0,\dots,n-1\}$ and
  $V^\sigma_{i+j}$ at every node on level at least $j$ that is a
  $\sigma$-successor of its parent; in addition, every node on level
  $n-i$ is the start of a $K_0$-path; moreover, the root
  might make true any of the concept names $V^\sigma_j$,
  $\sigma \in \{r,s\}$ and $1 \leq j \leq i$, which
  are then all also made true by all other nodes in the tree.
  
\end{itemize}
In all of the above, any node that has an incoming $r/s$-path of
length~$i \in \{1,\dots,2n\}$ that starts with $\sigma \in \{r,s\}$ is
additionally labeled with concept name~$U^\sigma_i$. Moreover, the
beginning of the path/root of the tree may be labeled with concept
names of the form $U^\sigma_i$. Then any node on depth $i+j$, with
$i+j \leq 2n$, is labeled with $U^\sigma_{i+j}$.

\medskip

Now, the announced model \Imc is constructed  by starting with \Amc
and doing the following:
%
%
\begin{enumerate}
 
\item exhaustively apply all concept inclusions in $\Omc_n$
  that have a concept name on the right-hand side; 
  
\item at every  $a \in \mn{ind}(\Amc)$, attach an infinite
  tree in which every node has two successors, one for each role
  name $r,s$, and in which no concept names are made true;

\item if $a \in L_i^\Imc$, $0 \leq i < n$, then attach at $a$ an
  $L_{i}$-tree;
  
\item if $a \in K_i^\Imc$, $0 \leq i < n$, then attach at $a$ a
  $K_{i}$-tree;

\item if $a \in L_i^\Imc$, $n \leq i \leq 2n$, then attach at $a$ an
  $L_i$-path;

\item if $a \in K_i^\Imc$, $n \leq i \leq 2n$, then attach at $a$ a
  $K_i$-path;

\item if $W^\sigma_i(a) \in \Amc$ for some $a$, then make 
  $W^\sigma_i$ true everywhere in \Imc.

  

\end{enumerate}
By going over the concept inclusions in $\Omc_n$ and using
Properties~(a) and~(b), it can be verified that \Imc is indeed a model
of $\Omc_n$. In particular, the inclusions $W^r_i \sqcap W^s_i
\sqsubseteq L_0$ are satisfied since there is no $d \in \Delta^\Imc$
with $d \in (W^r_i \sqcap W^s_i)$; if there was such a $d$, then by
construction of \Imc there are assertions $W^r_i(a)$ and $W^s_i(b)$ in
\Amc, in contradiction to the connectedness of \Amc and $\Amc,\Omc_n
\not\models L_0(a_0)$. For the CIs $U^r_i \sqcap U^s_i\sqsubseteq
L_0$, we argue that there is no $d \in \Delta^\Imc$ with $d \in (U^r_i
\sqcap U^s_i)^\Imc$. To see this, note that there are no
$U^r_i(a),U^s_i(a) \in \Amc$ for any $a$ as otherwise $\Amc,\Omc_n
\models L_0(a_0)$. Now consider Step~1 of the construction of \Imc and
assume that it adds some $a \in \mn{ind}(\Amc)$ to $(U^r_i \sqcap
U^s_i)^\Imc$.  But this means that $\Amc,\Omc_n \models U^r_i(a)$ and
$\Amc,\Omc_n \models U^s_i(a)$, in contradiction to $\Amc,\Omc_n
\not\models L_0(a_0)$, due to connectedness of \Amc.  Given that there
is no $d \in (U^r_i \sqcap U^s_i)^\Imc$ for any $i$ after Step~1, it
is readily checked that the elements $d$ added in Steps~2-6 also
satisfy $d \notin (U^r_i \sqcap U^s_i)^\Imc$.

\medskip

We now use \Imc to prove the claim. Let $\Hmc'$ be the set of all
$H \in \Hmc_n$ with $\Amc,\Omc_n \models H(a_0)$. With each
$H \in \Hmc'$, we associate an $a_H \in \mn{ind}(\Amc)$ as follows.
Let $H=\exists \sigma_1 \cdots \sigma_{2n} . A$. Then $\Imc$
contains a path from $a_0$ to some element $d_H \in A^{\Imc}$ that is
labeled $\sigma_1 \cdots \sigma_{2n}$. If $d_H \in \mn{ind}(\Amc)$, then
$a_H=d_H$. Otherwise, $d_H$ is in a path or tree attached to some
$a \in \mn{ind}(\Amc)$. Set $a_H=a$. To show that
$||\Amc|| \geq |\{ H \in \Hmc_n \mid \Amc,\Omc_n \models H(a_0) \}|$
as required, it suffices to prove that $a_H \neq a_{H'}$ whenever
$H \neq H'$.  Thus let $H,H' \in \Hmc'$ with $H \neq H'$,
$H=\exists \sigma_1 \cdots \sigma_{2n} . A$, and
$H'=\exists \sigma'_1 \cdots \sigma'_{2n} . A$.  Assume to the contrary
of what is to be shown that $a_H=a_{H'}$. We distinguish four
cases:
\begin{itemize}

\item $d_H=a_H$, $d_{H'}=a_{H'}$.

  Then there is a path from $a_0$ to $a_H$ in \Imc labeled
  $\sigma_1 \cdots \sigma_{2n}$ and a path from $a_0$ to $a'_H$ labeled
  $\sigma'_1 \cdots \sigma'_{2n}$. By construction of \Imc, these paths
  must already exist in \Amc.   From $H \neq H'$, we thus obtain a
  contradiction to (e).

\item $d_H=a_H$, $d_{H'} \neq a_{H'}$.

  By construction of \Imc, $d_{H'} \neq a_{H'}$ implies that
  $\Amc,\Omc_n \models L_i(a_H)$ or   $\Amc,\Omc_n \models K_i(a_H)$
  for some $i$ with $0 \leq i < 2n$. Moreover, $d_H=a_H$ implies
  $\Amc,\Omc_n \models A(a_H)$. Thus $\Amc,\Omc_n \models
  D(a_H)$. By the connectedness of \Amc, we obtain $\Amc,\Omc_n \models
  D(a_0)$, thus $\Amc,\Omc_n \models
  L_0(a_0)$ in contradiction to (c).

\item $d_H \neq a_H$, $d_{H'} = a_{H'}$.

  Symmetric to previous case.

\item $d_H \neq a_H$, $d_{H'} \neq a_{H'}$.
  
  We first show that $d_H$ and $d_{H'}$ are not in an $L_i$-tree,
  $0 \leq i < n$.  Assume to the contrary that $d_H$ is (the case of
  $d_{H'}$ is symmetric). Then it occurs on level $2n-i$ in the tree,
  since $d_H \in A^\Imc$. Since an $L_i$-tree was attached to $a_H$,
  we must have $\Amc,\Omc_n \models L_i(a_H)$.  Moreover, there is an
  $r/s$-path in \Imc from $a_0$ to $a_H$ of length $i$, the prefix of
  $\sigma_1 \cdots \sigma_{2n}$ of this length. By construction of
  \Imc, this path must already be in \Amc. This is in contradiction to (c).

  We next show that $d_H$ and $d_{H'}$ are not in a $K_i$-tree,
  $0 \leq i < n$.  Assume to the contrary that $d_H$ is (the case of
  $d_{H'}$ is symmetric). Then it occurs on level $2n-i$ in the tree,
  since $d_H \in A^\Imc$.  By definition of such trees (and the
  attached paths), this implies that $\Wmc = \{ 1,\dots, n \}$ and
  thus \Amc contains assertions
  $W^{\sigma''_1}_1(a_1),\dots, W^{\sigma''_n}_n(a_n)$.  We must
  further have $\Amc,\Omc_n \models K_i(a_H)$ and there is an
  $r/s$-path in \Imc, thus in \Amc from $a_0$ to $a_H$ of length $i$.
  This is in contradiction to (d).

  Thus, $d_H$ and $d_H'$ are both in an $L_i$-path or in a $K_i$-path,
  $n \leq i \leq 2n$. If they are in different such paths, then
  $\Amc,\Omc_n \models D(a_H)$, which is in contradiction to (c) as
  \Amc is connected. Thus, they must be in the same $L_i$-path or in
  the same $K_i$-path. Since each such path contains a single element
  $d$ with $d \in A^\Imc$, we obtain $d_H=d'_H$. From $H \neq H'$, it
  thus follows that there are two different paths of length $i$ in \Imc
  from $a_0$ to $a_H$, the prefixes of this length of
  $\sigma_1 \cdots \sigma_{2n}$ and of
  $\sigma'_1 \cdots \sigma'_{2n}$.  This is in contradiction to (e).
    
\end{itemize}
This finishes the proof of the claim.

\medskip

We can use the claim to show the following invariant:
\begin{itemize}

\item[($*$)] at any point in time, the sum $m$ of the sizes of the
  inputs to membership and equivalence queries made so far is not
  smaller than the number of candidates that were removed from \Hmc.

\end{itemize}
Note in fact that only Cases~2 and~3 of membership queries and Cases~3
and~6 of equivalence queries may remove candidates from $\Hmc$ and
that they all remove only one candidate for each query posed with the
exception of Case~3 of membership queries which can remove multiple
candidate. However, the claim implies that the number of removed
candidates in Case~3 of membership queries is bounded from above by
the size of the query posed.

It is clear that there is a polynomial $q_3$ such that the size of all
counterexamples returned by the oracle is at most $q_3(n)$.  The
overall sum of the sizes of posed membership and equivalence queries
is bounded by $p(q_1(n),q_2(n),q_3(n))$. It thus follows from ($*$)
that at most $p(q_1(n),q_2(n),q_3(n))$ candidate concepts have been
removed from $\Hmc_n$.  By the choice of $n$, at least two candidate
concepts remain in $\Hmc$ after the algorithm finished. Thus, the
learner cannot distinguish between them, and we have derived a
contradiction.
\end{proof}

\end{document}